\documentclass{article}

\PassOptionsToPackage{numbers, compress, sort}{natbib}

\usepackage[preprint]{neurips_2022}

\usepackage[utf8]{inputenc} 
\usepackage[T1]{fontenc}    
\usepackage{nicefrac}       
\usepackage{microtype}      

\usepackage{paralist}
\usepackage{enumitem}

\usepackage{cancel}

\usepackage[final]{changes}
\definechangesauthor[name={Bastian}, color=aerospace]{BR}
\definechangesauthor[name={Florian}, color=darkaqua]{FG}
\definechangesauthor[name={Roland}, color=blue]{RK}
\definechangesauthor[name={done}, color=green]{done}
\definechangesauthor[name={done}, color=bured]{SZ}

\setdeletedmarkup{
  \ifmmode
  \,\cancel{\textcolor{wine}{#1}}
  \else
  \textcolor{wine}{\sout{#1}}  
  \fi
}

\usepackage{titletoc}
\usepackage{hyperref}

\usepackage{amsmath}
\usepackage{amsfonts}       
\usepackage{amssymb}
\usepackage{amsthm}
\usepackage{bbm}
\usepackage[geometry]{ifsym}

\usepackage{mathtools}

\usepackage{float}
\usepackage{multirow}
\usepackage{makecell}
\usepackage{subcaption}

\usepackage{url}            
\usepackage{booktabs}       
\usepackage{bm}
\usepackage{mathrsfs}  

\usepackage{graphicx}
\usepackage{tikz}
\usetikzlibrary{arrows.meta}
\usetikzlibrary{math}
\usetikzlibrary{decorations.pathreplacing,calligraphy}
\usetikzlibrary{plotmarks}

\usepackage{pgfplots}
\pgfplotsset{compat=1.16}
\usepackage[labelfont={bf,small},textfont=small]{caption}
\usepackage{cases}
\usepackage{wrapfig}
\usepackage[capitalise]{cleveref}
\usepackage{rotating}

\usepackage{tablefootnote}
\usepackage{pdflscape}
\usepackage{afterpage} 
\usepackage{capt-of}

\usepackage[toc]{appendix}

\usepackage{todonotes}

\makeatother

\theoremstyle{plain}
\newtheorem{theorem}{Theorem}[section]
\newtheorem{proposition}[theorem]{Proposition}
\newtheorem{lemma}[theorem]{Lemma}
\newtheorem{corollary}[theorem]{Corollary}
\theoremstyle{definition}
\newtheorem{definition}[theorem]{Definition}

\newtheoremstyle{solution}%
  {3pt}
  {3pt}
  {}
  {}
  {\bfseries\color{black}}
  {\textcolor{black}{.}}
  {.5em}
  {}
\theoremstyle{solution}
\newtheorem{remark}[theorem]{Remark}
\newtheorem{example}{\bf \textcolor{black}{Example}}[section]

\makeatletter
\usepackage{ifthen}

\definecolor{unired}{HTML}{A6192E}  
\definecolor{uniblue}{HTML}{002D72} 
\definecolor{unigreen}{HTML}{115740}
\definecolor{aerospace}{HTML}{FF4F00}
\definecolor{blue0}{HTML}{207AB6}
\definecolor{orange1}{HTML}{FF7F0B}

\definecolor{darkaqua}{HTML}{0A8293}
\definecolor{wine}{HTML}{920923}
\definecolor{bured}{rgb}{0.8, 0.0, 0.0}

\definecolor{bargreena}{RGB}{45, 106, 96}
\definecolor{bargreenb}{RGB}{179, 201, 195}
\definecolor{bargraya}{RGB}{115, 115, 115}
\definecolor{bargrayb}{RGB}{230, 230, 230}
\definecolor{bargrayc}{RGB}{181, 181, 181}

\hypersetup{
    linktoc=all,
    colorlinks=true,
    linkcolor=uniblue,
    filecolor=unired,      
    urlcolor=uniblue,
    citecolor=unired,
    pdfpagemode=FullScreen,
    }

\newcommand\DoToC{%
  \startcontents
  \printcontents{}{1}{{\large \textbf{Appendix (Supplementary Material) -- Overview}}\vskip3pt\hrule\vskip5pt}
  \vskip3pt\hrule\vskip5pt
}

\newcommand{\ie}{i.e.}
\newcommand{\eg}{e.g.}
\newcommand{\wrt}{w.r.t.~}

\DeclareMathOperator{\Lip}{Lip}
\DeclareMathOperator{\Id}{Id}
\DeclareMathOperator{\Comp}{Comp}
\DeclareMathOperator{\Sum}{Sum}
\DeclareMathOperator{\Cat}{Cat}
\DeclareMathOperator{\card}{card}

\DeclareMathOperator*{\argmax}{arg\,max}
\DeclareMathOperator*{\argmin}{arg\,min}

\newcommand{\layer}{\operatorname{layer}}

\newcommand{\compto}[1]{\prescript{\to}{}{#1}}

\newcommand{\rade}[2]{\hat {\mathfrak{R}}_{#1} ({#2})}

\newcommand{\set}[1]{\left\{ #1\right\}} 
\newcommand{\tuple}[1]{\left( #1\right)} 

\newcommand{\norm}[1]{\left\lVert#1\right\rVert}
\newcommand{\ceil}[1]{\left\lceil#1\right\rceil}
\newcommand{\floor}[1]{\left\lfloor#1\right\rfloor}

\newsavebox\Tbox

\title{On Measuring Excess Capacity in Neural Networks}

\author{%
  Florian Graf \\
  University of Salzburg \\
  \texttt{\href{mailto:florian.graf@plus.ac.at}{florian.graf@plus.ac.at}} \\
  \And
  Sebastian Zeng \\
  University of Salzburg \\
  \texttt{\href{mailto:sebastian.zeng@plus.ac.at}{sebastian.zeng@plus.ac.at}} \\
  \And
  Bastian Rieck \\
  Institute for AI and Health\\
  Helmholtz Munich\\
  \texttt{\href{bastian@rieck.met}{bastian@rieck.me}} \\
  \And
  Marc Niethammer \\
  UNC Chapel Hill\\
  \texttt{\href{mailto:mn@cs.unc.edu}{mn@cs.unc.edu}} \\
  \And
  Roland Kwitt \\
  University of Salzburg \\
  \texttt{\href{mailto:roland.kwitt@plus.ac.at}{roland.kwitt@plus.ac.at}} \\
}

\begin{document}

\maketitle

\begin{abstract}
We study the \emph{excess capacity} of deep networks in the context
of supervised classification. That is, given a capacity measure of
the underlying hypothesis class -- in our case, empirical Rademacher
complexity -- 
to what extent can we (a priori) constrain this class
while retaining an empirical error on a par with the unconstrained
regime? To assess excess capacity in modern architectures (such as
\emph{residual networks}), we extend and unify prior Rademacher
complexity bounds to accommodate function \emph{composition} and
\emph{addition}, as well as the structure of convolutions. The
capacity-driving terms in our bounds are the Lipschitz constants of
the layers and an ($2,1$) group norm distance to the initializations of
the convolution weights. Experiments on benchmark datasets of
varying task difficulty indicate that (1) there is a substantial
amount of excess capacity per task, and (2) capacity can be kept at
a surprisingly similar level across tasks. Overall, this suggests
a notion of \emph{compressibility} with respect to weight norms,
complementary to classic compression via weight pruning.  
Source code is available at \url{https://github.com/rkwitt/excess_capacity}.
\end{abstract}

\vspace{0.15cm}
\section{Introduction}
\label{section:introduction}
Understanding the generalization behavior of deep networks in supervised
classification is still a largely open problem, despite a long history
of theoretical advances. The observation that (overparametrized) models
can easily fit---\ie, reach zero training error---to randomly permuted training labels~\cite{Zhang16a, Zhang21a} but, when trained on
unpermuted labels, yield good generalization performance, has fueled
much of the progress in this area. Recent works range from relating
generalization to weight norms~\cite{Neyshabur15a,Neyshabur17a,Bartlett17a,Golowich18a,Ledent21a},
measures of the distance to initializations~\cite{Nagarajan17a}, implicit
regularization induced by the optimization algorithm \cite{Soudry18a,
Cao19a}, or model compression \cite{Arora18a,Baykal19a,Suzuki20a}. Other
works study connections to optimal transport \cite{Chuang21a}, or
generalization in the neural tangent kernel setting
\cite{Jacot18a,Arora19a}.

When seeking to establish generalization guarantees within the classic
uniform convergence regime, bounding a capacity measure, such as the
Rademacher complexity~\cite{Bartlett02a}, of the hypothesis class is the
crucial step. While the resultant generalization bounds are typically
vacuous and can exhibit concerning behavior \cite{Nagarajan19a}, the
\emph{capacity bounds} themselves offer  invaluable insights through the
behavior of the bound-driving quantities, such as various types of
weight norms or Lipschitz constants. 

Particularly relevant to our work is the observation that the
bound-driving quantities tend to increase with \emph{task difficulty}.
Fig.~\ref{fig:intro} illustrates this behavior 
in terms of Lipschitz constants per layer and the distance
of each layer's weight to its initialization (measured via a group norm
we develop in \cref{subsection:preliminaries}).
\begin{wrapfigure}{r}{0.45\textwidth}
  \includegraphics[width=0.45\textwidth]{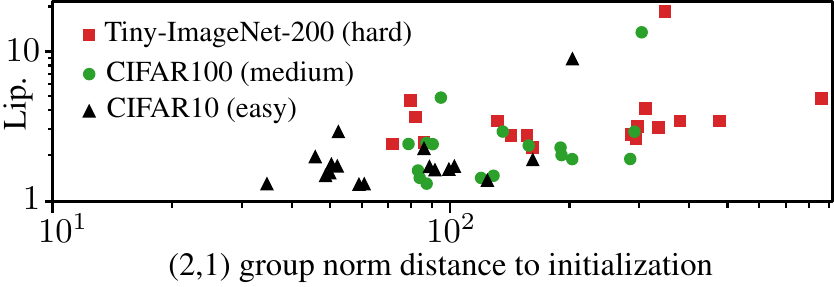}
  \vspace{-0.4cm}
  \caption{Layerwise Lipschitz constants and distance to initialization of a ResNet18 model (see \cref{section:empirical_evaluation}), trained on different datasets.\label{fig:intro}
 }
  \vskip-0.35cm
\end{wrapfigure}
This raises two immediate questions: First (\hypertarget{Q1}{\textbf{Q1}}), can a network maintain empirical testing performance at a substantially lower capacity?
Second (\hypertarget{Q2}{\textbf{Q2}}), 
is the level of this lowered capacity inevitably tied to task difficulty,
i.e., would a reduced-capacity model for an easy task fail on a difficult task? Both of these questions aim at the amount
of ``unneeded'' capacity, which we refer to as \emph{excess capacity}
in the remainder of this work.

We will address questions (\textbf{Q1}) and 
(\textbf{Q2}) by means of controlling the empirical Rademacher
complexity of a neural network. To this end, we consolidate and extend
prior results from the literature on Rademacher complexity bounds to
accommodate a broad range of network components in a unified way,
including convolutions and skip connections, two ubiquitous elements in
state-of-the-art models.

Our \textbf{contributions} can be summarized as follows:
\vspace{-.15cm}
\begin{enumerate}[leftmargin=0.5cm]
\itemsep0.2em 
\item We establish two bounds (in
  \cref{paper:theorem:rademacher_complexity}) for the empirical
  Rademacher complexity of neural networks that use \emph{convolutions} and
  implement functions built from \emph{composition} and \emph{addition}.
  Specifically, we introduce two novel, convolution-specific,
  single-layer covering number bounds in
  Section~\ref{subsection:single_layer} and contrast them to prior art,
  then modularize the single-layer to multi-layer covering approach of
  Bartlett et al.~\cite{Bartlett17a} in
  Section~\ref{subsection:whole_network}, and eventually present one
  incarnation of our framework for convolutional residual networks in Sections
  \ref{subsection:resnets_cover} and \ref{subsection:resnets_rade}. 
\item We present an extensive set of experiments (in
  \cref{section:empirical_evaluation}) with a ResNet18 model across
  benchmark datasets of varying task difficulty, demonstrating that
  model capacity, when measured via our weight norm based bound, (1) can
  be kept surprisingly small per task, and (2) can be kept at
  \emph{roughly the same level} regardless of task difficulty. Both
  observations suggest \emph{compressibility} of neural networks with
  respect to weight norms, complementary to the well-known compressibility property of neural networks
  with respect to the number of parameters~\cite{Arora18a,Suzuki20a}.
\end{enumerate}

\section{Related Work}
\label{section:related_work}
Many prior works establish uniform-convergence type generalization
bounds for neural networks through Rademacher complexity analysis. We
review such approaches, highlighting challenges that arise with
modern network architectures and the peculiarities of convolutional
layers.

One direct approach to bound the empirical Rademacher complexity is via
a \emph{layer-peeling} strategy~\cite{Neyshabur15a, Golowich18a,Yun19a}
where the Rademacher complexity of $L$-layer networks is expressed by
a factor times the Rademacher complexity of $(L-1)$-layer networks; in
other words, the last layer is peeled off. This factor is typically
a matrix ($p,q$) group norm, and thus the bounds usually scale with the product of the latter. Notably, the nonlinearities need to be
\emph{element-wise} operations, and some approaches only work for
specific nonlinearities, such as ReLUs or asymmetric
activations. 
A second strategy is to bound the empirical Rademacher complexity via
a \emph{covering numbers} approach~\cite{Zhang02a,
Bartlett17a,Lin19a,Ledent21a}, typically achieved via Dudley's entropy
integral \cite{Dudley67a}. This strategy is particularly flexible as it
allows for arbitrary (but fixed) nonlinearities and various paths to
bounding covering numbers of network parts, \eg, via Maurey's
sparsification lemma or via parameter counting. The corresponding
whole-network bounds typically scale with the product of each layer's
Lipschitz constant or local empirical estimates thereof \cite{Wei19a}.
Irrespective of the particular proof strategy, most formal arguments only hold for neural networks constructed from function composition, \ie, maps of the form
\begin{equation}
    x \mapsto \sigma_L(A_L\sigma_{L-1}(A_{L-1} \dots \sigma_1(A_1 x)\dots))
    \enspace,
    \label{eq:simple_architecture}
\end{equation}
where $\sigma_i: \mathbb R^{d_{i-1}} \to \mathbb R^{d_i}$ are
nonlinearities and $A_i$ are weight matrices specifying the $i$-th linear map. However, modern architectures often rely on operations specifically
tailored to the data, such as convolutions, and typically incorporate
skip connections as in residual networks \cite{He16a}, rendering
many results inapplicable or suboptimal for such models. In this work, we  handle convolutions and skip connections, thus increasing the applicability and utility of such bounds.

For example, while residual networks have been studied extensively, theory mostly focuses on expressivity or optimization aspects \cite{Hardt17a,Lin18a,Bartlett17a,Yun19a}. Yun et al.\ \cite{Yun19a} provide a Rademacher complexity bound via layer-peeling for fully-connected layers and element-wise activations. He 
et al.\ \cite{HeF20a} establish a generalization bound for residual
networks via covering number arguments, resting upon earlier work by
Bartlett et al.\ \cite{Bartlett17a} for linear maps. However, when
directly applied to convolutions, both bounds scale unfavorably \wrt the spatial input size (see Section~\ref{subsection:single_layer}). Other works provide generalization guarantees \emph{specifically} tailored to convolutional networks, cf. \cite{Lin19a,Long20a,Ledent21a,Gouk21a}, and, although such bounds scale benignly with input size, they only apply for models as in \cref{eq:simple_architecture}. 

\vskip -.25cm
\paragraph*{An initial numerical comparison.}
Bounds on the empirical Rademacher complexity differ in their dependence on various quantities, such
as matrix ($p,q$) group norms, Lipschitz constants, or the number of parameters. 
Thus, a precise formal comparison is challenging and, depending on the setting, different bounds may be preferable. To provide some intuition about magnitude differences, we evaluated several existing bounds (including ours from \cref{subsection:resnets_rade}) on two convolutional (ReLU) networks with 6 and 11 layers, see Fig.~\ref{fig:simple_comparison} and \cref{appendix:subsection:comparison_table} for details.

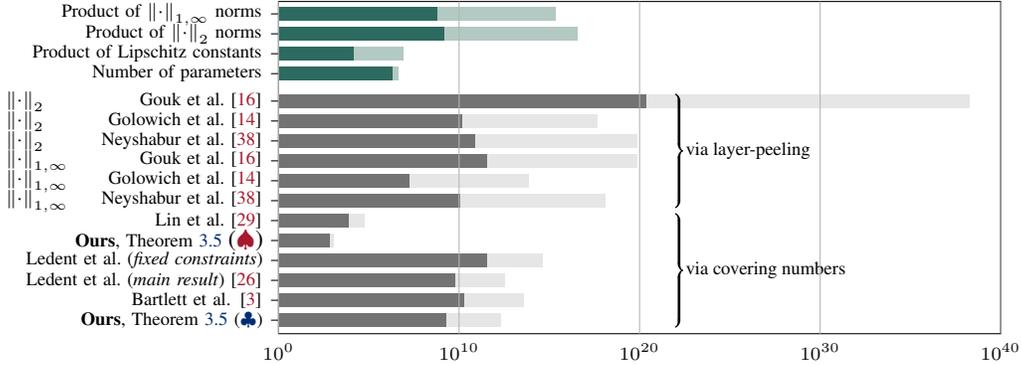
\begin{figure}
    \begin{tikzpicture}
    \tikzset{
    font={\scriptsize}
    }

    \tikzmath{
        \width = 0.09; 
        \sep = .265; 
        \xscale = .24; 
        \nbars = 16; 
        \ticksize = 0.1; 
        \bracepos = 22; 
    } 

    \newcommand\barplot[4]{%
    \foreach \coun/\lab/\val in #1{%
        \path[fill, #2]%
        (0,{\coun*\sep+\width})
        -- ({\val*\xscale},{\coun*\sep+\width})
        -- ({\val*\xscale}, {\coun*\sep-\width})
        -- (0,{(\coun*\sep)-\width})
        ;
        \ifthenelse{#3=1}{
            \draw (0, {\coun*\sep}) -- ({-1*\ticksize}, {\coun*\sep}) node[text width=#4, align=right, left] {\lab};
        }{}
    }
}

    \def\normsconva{
        6/{$\norm{\cdot}_{1,\infty}$ \hfill Neyshabur et al.\ {\cite{Neyshabur15a}}}/10.07,
        7/{$\norm{\cdot}_{1,\infty}$\hfill Golowich et al.\ \cite{Golowich18a}}/7.23,
        8/{$\norm{\cdot}_{1,\infty}$ \hfill Gouk et al.\ \cite{Gouk21a}}/11.54,
        9/{$\norm{\cdot}_{2}$ \hfill Neyshabur et al.\ \cite{Neyshabur15a}}/10.91,
        10/{ $\norm{\cdot}_{2}$\hfill Golowich et al.\ \cite{Golowich18a}}/10.15,
        11/{ $\norm{\cdot}_{2}$ \hfill Gouk et al.\ \cite{Gouk21a}}/20.34
        }
    \def\quantsconva{ 
        12/{Number of parameters}/6.31,
        13/{Product of Lipschitz constants}/4.17,
        14/{Product of $\norm{\cdot}_{2}$ norms}/9.16,
        15/{Product of $\norm{\cdot}_{1,\infty}$ norms}/8.81
    }
    \def\restconva{ 
        0/{\textbf{Ours}, \cref{paper:theorem:rademacher_complexity} {\eqref{theorem:main_residual_network_covering_rademacher:eq_norms}}
        }/9.31,
        1/{Bartlett et al. \cite{Bartlett17a}}/10.29,
        2/{Ledent et al.\! (\emph{main result}) \cite{Ledent21a}}/9.77,
        3/{Ledent et~al.\! (\emph{fixed constraints})}/11.57,
        4/{\textbf{Ours}, \cref{paper:theorem:rademacher_complexity} \small{\eqref{theorem:main_residual_network_covering_rademacher:eq_params}
        }}/2.86,
        5/{Lin et al.\ \cite{Lin19a}}/3.91
    }

    \def\normsconvb{
        6/{$\norm{\cdot}_{1,\infty}$ \hfill Neyshabur et al.\ \cite{Neyshabur15a}}/18.13,
        7/{$\norm{\cdot}_{1,\infty}$\hfill Golowich et al.\ \cite{Golowich18a}}/13.84,
        8/{$\norm{\cdot}_{1,\infty}$ \hfill Gouk et al.\ \cite{Gouk21a}}/19.86,
        9/{$\norm{\cdot}_{2}$ \hfill Neyshabur et al.\ \cite{Neyshabur15a}}/19.85,
        10/{ $\norm{\cdot}_{2}$\hfill Golowich et al.\ \cite{Golowich18a}}/17.68,
        11/{ $\norm{\cdot}_{2}$ \hfill Gouk et al.\ \cite{Gouk21a}}/38.28
    }
    \def\quantsconvb{ 
        12/{Number of parameters}/{6.62},
        13/{Product of Lipschitz constants}/6.91,
        14/{Product of $\norm{\cdot}_{2}$ norms}/16.59,
        15/{Product of $\norm{\cdot}_{1,\infty}$ norms}/15.37
    }
    \def\restconvb{ 
        0/{\textbf{Ours-norms}
        }/12.33,
        1/{Bartlett et al. \cite{Bartlett17a}}/13.58,
        2/{Ledent et al.\! (\emph{main result}) \cite{Ledent21a}}/12.55,
        3/{Ledent et~al.\! (\emph{fixed constraints})}/14.64,
        4/{\textbf{Ours-params}
        }/3.07,
        5/{Lin et al.\ \cite{Lin19a}}/4.78
    }   


    \begin{scope}[yshift=0.3 cm]    
        \barplot{\normsconvb}{bargrayb}{0}{3.3cm}
        \barplot{\restconvb}{bargrayb}{0}{3.5cm}
             
        \barplot{\normsconva}{bargraya}{1}{3.4cm}
        \barplot{\restconva}{bargraya}{1}{3.5cm}

        \begin{scope}[yshift=0.1cm]
            \barplot{\quantsconvb}{bargreenb}{0}{3.5cm} 
            \barplot{\quantsconva}{bargreena}{1}{3.5cm}
        \end{scope}    

    \draw [line width = 1.pt, decorate,
    decoration = {calligraphic brace, mirror}] (\bracepos*\xscale,0*\sep-\width) --  (\bracepos*\xscale,5*\sep+\width)
    node[pos=.5, right] {via covering numbers};
    \draw [line width = 1.pt, decorate,
    decoration = {calligraphic brace, mirror}] (\bracepos*\xscale,6*\sep-\width) --  (\bracepos*\xscale,11*\sep+\width)
    node[pos=.5, right] {via layer-peeling};        
    \end{scope}

    \draw (0,0.12) -- ({40*\xscale},0.12) -- ({40*\xscale},{(\nbars*\sep)+0.3}) -- (0,{(\nbars*\sep)+0.3}) -- cycle;
    \foreach \i in {0, 10, 20, 30, 40}{
        \draw[color=bargrayc] ({\i*\xscale},0.12) node[below] {\textcolor{black}{$10^{\i}$}} -- ({\i*\xscale},{(\nbars*\sep)+0.3});
    }
\end{tikzpicture}
    \vspace{-0.1cm}
    \caption{Empirical Rademacher complexity bounds (grouped by proof strategy; lower is better), for a 6- (\textcolor{black!50}{$\blacksquare$}\textcolor{unigreen!75}{$\blacksquare$}) and an 11-layer (\textcolor{black!10}{$\blacksquare$}\textcolor{unigreen!50}{$\blacksquare$}) convolutional network, trained on CIFAR10. Bounds are listed in \cref{appendix:subsection:comparison_table} and quantities that typically appear in these bounds are highlighted in \textcolor{unigreen}{\textbf{green}} (top part of figure) for reference. \label{fig:simple_comparison}}
    \vspace{-0.25cm}
\end{figure}

\section{Rademacher Complexity Analysis}
\label{section:generalization_bounds}
\vskip -0.15cm
To derive bounds on the empirical Rademacher complexity, we follow the
margin-based multiclass learning formalism and the flexible proof
strategy of Bartlett et al.~\cite{Bartlett17a}.
Section~\ref{subsection:single_layer} introduces novel single-layer
covering number bounds for convolutions.
Section~\ref{subsection:whole_network} modularizes and extends the
single- to multi-layer covering step to account for architectures such
as residual networks (\cref{subsection:resnets_cover}). Last, 
\cref{subsection:resnets_rade} presents and discusses our Rademacher complexity bounds.
\vspace{-0.1cm}
\subsection{Preliminaries}
\label{subsection:preliminaries}
\vspace{-0.1cm}
In a $\kappa$-class classification task, we are given $n$ instance/label
pairs $S=((x_1,y_1),\ldots,(x_n,y_n))$, drawn iid from a probability
measure $\mathcal{D}$ on $\mathbb{R}^{d}{\times}\{1,\ldots,\kappa\}$. For
a neural network $f$ in a hypothesis class $\mathcal{F}\ {\subset}\
\{f\colon \mathbb{R}^d\ {\to}\ \mathbb{R}^\kappa\}$, a class label for input $x$ is obtained by taking the argmax over the components of $f(x)\ {\in}\ \mathbb{R}^\kappa$. The \emph{margin operator}
$\mathcal{M}\colon \mathbb{R}^\kappa{\times}\{1,\ldots,\kappa\}\ {\to}\ \mathbb{R},
(v,y) \mapsto v_y - \max_{i\neq y} v_i$ leads, with margin $\gamma > 0$, to the \emph{ramp loss}
$\ell_\gamma$ and the \emph{empirical ramp
risk} $\hat{R}_{\gamma}$, defined as
\begin{align}
    \label{eqn:def:ramploss}
    \ell_{\gamma}: \mathbb{R}\to\mathbb{R}^+, r \mapsto (1+\nicefrac{r}{\gamma})\mathbbm{1}_{r \in [-\gamma,0]} + \mathbbm{1}_{r>0} \quad \text{and}\quad \hat{R}_{\gamma}(f) = \frac{1}{n}\sum_{i=1}^n \ell_{\gamma}(-\mathcal{M}(f(x_i),y_i) ) \enspace.
\end{align}
To derive generalization bounds via classical Rademacher complexity
analysis \cite{Mohri18a}, without having to resort to vector-contraction
inequalities \cite{Maurer16a}, we consider the hypothesis class 
\begin{equation}
    \label{eqn:def:Fgamma}
    \mathcal F_\gamma = \set{(x,y) \mapsto \ell_\gamma(-\mathcal M (f(x),y)) : f\in \mathcal F}\enspace.
\end{equation}
Then, defining the \emph{empirical Rademacher complexity} of any class  $\mathcal{H}$ of real-valued functions as
\begin{equation}
    \label{eqn:def:empRade}
    \rade{S} {\mathcal H}
    = \underset{\sigma}{\mathbb E} 
    \left[ 
      \sup\limits_{h \in \mathcal H} \frac{1}{n}\sum_{i=1}^n\sigma_i h(x_i,y_i)\right]\enspace,
\end{equation}
with iid Rademacher variables $\sigma = (\sigma_1,\ldots,\sigma_n)$ from
a uniform distribution on $\{\pm 1\}$, facilitates to study $\mathcal{F}_\gamma$ via \cref{eqn:def:empRade}.
The following lemma \cite[Lemma 3.1]{Bartlett17a} establishes the link 
to a margin-based multiclass generalization bound for any $f \in \mathcal{F}$.
\vskip1ex
\begin{lemma}
    Given a hypothesis class $\mathcal{F}$ of functions $f\colon \mathbb R^d \to \mathbb R^\kappa$ and a margin parameter $\gamma>0$, then, with probability of at least $1-\delta$ over the choice of $S \sim \mathcal{D}^n$, for any $f \in \mathcal F$, it holds that
      \begin{equation}
           \mathbb{P}
           [\argmax_{
            i\in \set{1,\dots,\kappa}} f(x)_i \neq y]\,
          {\le}\,
           \hat{R}_\gamma(f)
          \,{+}\, 2 \rade{S}{\mathcal{F}_\gamma}\,{+}\,3 \sqrt{\frac{\log\left(\tfrac{2}{\delta}\right)}{2n}}
      \enspace.
      \end{equation}
      \label{thm:Bartlett17a_Lemma31}
      \vspace{-0.25cm}
\end{lemma}
To obtain a computable expression for the right-hand side of
\cref{thm:Bartlett17a_Lemma31}, we seek a bound on
$\rade{S}{\mathcal{F}_\gamma}$ tied to some measurable quantities of the
network realizing  $\mathcal{F}$. 
For our purposes, the relationship of $\rade{S}{\mathcal{F}_\gamma}$ and the 
covering number of $\mathcal{F}_\gamma$ turns out to be a flexible approach.
In general, given a normed space $(\mathcal G, \norm{\cdot}_{\mathcal G})$,
the \emph{covering number} $\mathcal N(\mathcal G, \epsilon,
\norm{\cdot}_{\mathcal G})$ is the cardinality of the smallest $\epsilon$-cover of
$\mathcal G$, i.e., of the smallest subset $\mathcal U \subset \mathcal G$
such that, for any $g\in \mathcal G$, there exists $u \in \mathcal U$
with $\norm{g-u}_{\mathcal G} \le \epsilon$.
In our setting, $\mathcal G$ is a class of functions $g: (\mathcal X, \norm{\cdot}_{\mathcal X}) \to (\mathcal{Y}, \norm{\cdot}_{\mathcal Y})$ between normed spaces $\mathcal{X}$ and $\mathcal{Y}$.
Given data $X = (x_1,\ldots,x_n) \in \mathcal{X}^n$, we define a \emph{data-dependent norm} on $\mathcal{G}|_{X}$ as 
\begin{equation}
    \norm{g}_X = \sqrt{\sum_{i} \norm{g(x_i)}_{\mathcal Y}^2}\enspace.
    \label{eq:data_l2_norm}
\end{equation}
In other words, \cref{eq:data_l2_norm} is the $l_2$ norm on the restriction of $\mathcal{G}$ to $X$.
Specifically, we seek to bound $\log \mathcal N(\mathcal{F}_\gamma, \epsilon, \norm{\cdot}_S)$, 
as this facilitates to control the empirical 
Rademacher complexity of $\mathcal{F}_\gamma$ by means of Dudley's entropy integral.
Typically, such covering number bounds depend on the norm of the data itself, i.e., $\norm{X} = \sqrt{\|x_1\|_{\mathcal X}^2 + \cdots \|x_n\|_{\mathcal X}^2}$.

\subsection{Covering number bounds for convolutions}
\label{subsection:single_layer}
We consider 2D convolutions, acting on images with $c_\textit{in}$ channels of
width $w$ and height $h$, \ie, $x \in \mathbb R^{c_\textit{in}\times h\times
w}$. For readability only, we discuss convolutions of stride 1 and
input-size preserving padding; this is not an assumption required for \cref{paper:theorem:single_layer}.
Formally, a convolutional layer is a linear map
$\phi_K\colon \mathbb R^{c_\textit{in} \times h\times w} \to \mathbb R^{c_\textit{out} \times
h\times w}$ (as we omit bias parameters), where $c_\textit{in}$ and $c_\textit{out}$ denote the number of input and output
channels. The map is parametrized by a tensor $K \in \mathbb R^{c_\textit{out}\times c_\textit{in} \times k_h \times k_w}$ of spatial extension/kernel size $(k_h,k_w)$.
Since convolutions are linear maps, they can be specified by matrices which act on the (reshaped) inputs and one could invoke existing covering
number bounds.
However, this is suboptimal, as any
structure specific to convolutions is ignored. In particular, norm-based
generalization bounds agnostic to this structure incur unfavorable
scaling behavior \wrt the dimensionality of the input. To be more
specific, the weight tensor $K$ of a convolutional layer does not
directly specify the corresponding matrix; instead, it parametrizes
$c_\textit{out}$ filters, \ie, local linear maps, which are applied to the $(c_\textit{in}\!\times\!k_h\!\times\!k_w)$-sized pixel neighborhoods of the input. Hence, the matrix $M_K$
corresponding to the \emph{global} linear map
consists of many copies of the elements of this tensor, one for each of
the $hw$ patches the filters are
applied to. Thus, the $l_p$ norm of the matrix $M_K$ is $\norm{M_K}_p = (hw)^{1/p}\norm{K}_p$ (see
\cref{subsection:norm_of_M_k}). We mitigate this scaling issue by tying
the covering number of a convolutional layer to a variant of the (2,1)
group norm on the \emph{tensor} $K$ itself. We define this norm
as the sum over the $l_2$ norms taken along the \emph{input channels} of
$K$, i.e., 
\begin{equation}
    \norm{K}_{2,1} 
    = \sum_{ikl} \norm{K_{i\cdot kl}}_2
    = \sum_{ikl}
    \sqrt{\sum_{j} K_{ijkl}^2}
    \enspace.
    \label{eqn:def_21_norm}
\end{equation}
For the special case of inputs of size $(h,w)=(1,1)$ and kernel size $(k_h,k_w)= (1,1)$, convolution is just matrix multiplication along the channels. In this case, $M_K = K_{\cdot \cdot 1 1}$ and our norm from \cref{eqn:def_21_norm} agrees with the standard ($2,1$) group norm on $M_K^\top$, i.e., $\norm{K}_{2,1} = \norm{M_K^\top}_{2,1}$. \cref{paper:theorem:single_layer} establishes two covering number bounds for convolutions. 
\vspace{0.1cm}
\begin{theorem}
    \label{paper:theorem:single_layer}
    Let $b > 0$ and $\mathcal{F} = \{\phi_K|~ K \in \mathbb R^{c_\textit{out}\times c_\textit{in} \times k_h \times k_w}, \norm{K}_{2,1}\le b\}$ denote the class of 2D convolutions with $c_\textit{in}$ input channels, $c_\textit{out}$ output channels and kernel size $k_h{\times}k_w$, parametrized by tensors $K$ with $W = c_{\textit{out}}c_{\textit{in}}k_h k_w$ parameters.
    Then, for any $X \in \mathbb R^{n \times c_\textit{in}\times
    h \times w}$ and covering radius $\epsilon>0$,
    \begin{subnumcases}
        {\log\mathcal N(\mathcal F, \epsilon, \norm{\cdot}_X ) \leq}
        \ceil{\frac{\norm{X}^2 b^2}{\epsilon^2}} \log(2W) 
        \label{paper:theorem:single_layer_simple:eq_norms}
        \\
        2 W \log\left(1+
        \ceil{
            \frac{\norm{X}^2 b^2}{\epsilon^2}
        }
        \right)\label{paper:theorem:single_layer_simple:eq_params}\enspace.
     \end{subnumcases}
\end{theorem}

\cref{paper:theorem:single_layer_simple:eq_norms} is analogous to the single-layer bound of Bartlett et al.\ \cite[Lemma 3.2]{Bartlett17a} for fully-connected layers, 
but replaces the ($2,1$) group norm constraint on matrices $M^\top$ with a constraint on tensors $K$. This is tighter than invoking \cite[Lemma 3.2]{Bartlett17a} directly on $M_K^\top\in \mathbb R^{c_\textit{in} h w\,\times\, c_\textit{out} h w}$, as $K$ has only 
$c_\textit{out}c_\textit{in} k_h k_w$ parameters and $\norm{M_K^\top}_{2,1} \ge \nicefrac{hw}{\sqrt{k_h k_w}} \norm{K}_{2,1}$, see \cref{subsection:norm_of_M_k}.
A thorough comparison between the two bounds in \cref{paper:theorem:single_layer} is nuanced, though, 
as preferring one over the other depends on the ratio between the number of parameters $W$ and $\norm{X}^2 b^2/\epsilon^2$. 
The latter, in turn, requires to consider all covering radii $\epsilon$. Hence, we defer this discussion to \cref{subsection:resnets_rade}, 
where differences manifest more clearly in the overall empirical Rademacher complexity bounds.

\paragraph*{Proof sketch.} The statement of \cref{paper:theorem:single_layer} follows from an application of Maurey's sparsification lemma, which guarantees the existence of an $\epsilon$-cover of $\mathcal F$ (of known cardinality) if there is a finite subset $\set{V_1,\dots, V_d} \subset \mathcal F$ s.t.\ every $f\in \mathcal F$ is a convex combination of the $V_i$. We show that one can find such a finite subset of cardinality $d=2 c_\textit{in} c_\textit{out} k_h k_w = 2W$. The cardinality of the cover is then determined by a \emph{combinatorial quantity} which additionally depends on $\norm{X}$ and the norm constraint $b$. Bounding this quantity, i.e., a binomial coefficient, in two different ways, establishes the bounds. 

\paragraph*{Relation to prior work.} Closely related is recent work by Ledent
et al.\ \cite{Ledent21a} who derive $l_\infty$ covering number bounds
for convolutional layers based on a classic result by Zhang
\cite{Zhang02a}. Similar to
Eq.~\eqref{paper:theorem:single_layer_simple:eq_norms}, their bound depends
on the square of a weight norm directly on the tensor $K$, the square of
a data norm, as well as a logarithmic term. The data norm is the maximal
$l_2$ norm of a single patch. Compared to our result, this implicitly removes a factor of the
spatial dimension $hw$. However, when transitioning to multi-layer bounds, this factor reenters in the form of the spatial dimension of the output (after subsequent pooling) via the Lipschitz constant.
Overall, the quadratic terms
across both results scale similarly (with our data norm being less sensitive to
outliers), but we improve on Ledent et al.~\cite{Ledent21a} in the
logarithmic term. By contrast, the use of $l_\infty$ covers in \cite{Ledent21a}
yields whole-network bounds with improved dependency on the number of classes; see \cref{appendix:comparison_ledent} for an in-depth comparison.
In other related work, Lin et al.\ \cite{Lin19a} derive an
$l_2$ covering number bound for convolutional layers similar to
Eq.~\eqref{paper:theorem:single_layer_simple:eq_params}, which depends
linearly on the number of parameters and logarithmically on norms.
In their proof, Lin et al. \cite{Lin19a} show that every cover of a convolutional layer's weight space (a subset of a Euclidean space) induces a cover of the corresponding function space \wrt the data dependent norm defined in \cref{eq:data_l2_norm}. However, their approach incurs an additional factor inside the logarithm that corresponds to the number of how often each filter is applied, \ie, the spatial dimension of the output.
Importantly, non-convolution specific approaches can equally mitigate undesirable scaling issues, e.g., by utilizing ($1,\infty$) group norms on the matrices representing the linear map \cite{Neyshabur15a,Golowich18a,Gouk21a}; as differences primarily manifest in the resulting bounds on the Rademacher complexity, we refer to our discussion in Section~\ref{subsection:resnets_rade}.

\subsection{Covering number bounds for composition \& addition}
\label{subsection:whole_network}
As many neural networks are built from composition and summation of layers, we study covering numbers under these operations. The key
building blocks are the following, easy to verify, inequalities.
\vskip1ex
\begin{lemma}
    \label{theorem:sum_and_comp}
    Let $\mathcal F_1, \mathcal F_2$ be classes of functions on normed
    spaces $(\mathcal X, \norm{\cdot}_{\mathcal X})\ {\to}\ (\mathcal Y,
    \norm{\cdot}_{\mathcal Y})$ and let $\mathcal G$ be a class of
    $c$-Lipschitz functions $(\mathcal Y, \norm{\cdot}_{\mathcal Y}) {\to}
    (\mathcal Z, \norm{\cdot}_{\mathcal Z})$. Then, for any $X \in
    \mathcal X^n$ and $\epsilon_{ \mathcal F_1}$, $\epsilon_{
    \mathcal F_2}, \epsilon_{ \mathcal G} >0$, it holds that
    \begin{equation}
        \label{theorem:sum_and_comp:eq_sum}
        \mathcal N (
            \set{\textcolor{black}{f_1 + f_2}~|~ f_1 \in \mathcal F_1, f_2 \in \mathcal F_2},
            \epsilon_{ \mathcal F_1} + \epsilon_{ \mathcal F_2},
            \norm{\cdot}_X
        )
        \le
        \mathcal N (
            \mathcal F_1,
            \epsilon_{ \mathcal F_1},
            \norm{\cdot}_X
        )
        \mathcal N (
            \mathcal F_2,
            \epsilon_{ \mathcal F_2},
            \norm{\cdot}_X
        )
        \end{equation}
        \vspace{-0.05cm}
        and
        \begin{equation}
        \label{theorem:sum_and_comp:eq_comp}
        \mathcal N (
            \set{\textcolor{black}{g\circ f}|~ g \in \mathcal G, f \in \mathcal F_2},
            \epsilon_{ \mathcal G} + c \epsilon_{ \mathcal F_2},
            \norm{\cdot}_X
        )
        \le
        \mathcal N (
            \mathcal F_2,
            \epsilon_{ \mathcal F_2},
            \norm{\cdot}_X
        )
        \sup_{f \in \mathcal F_2}
        \mathcal N (
            \mathcal G,
            \epsilon_{ \mathcal G},
            \norm{\cdot}_{f(X)}
        )
        \!\enspace.
        \end{equation}
\vspace{-0.3cm}
\end{lemma}

To establish these inequalities, one chooses minimal covers of the
original function spaces and links their elements via the considered
operation, \ie, addition or composition. The resulting functions
correspond to tuples of elements of the original covers.
Hence, the right-hand side of the inequalities is a product of covering numbers.
The crucial step is to determine a preferably small radius $\epsilon$
such that these functions form an $\epsilon$-cover. In
\cref{theorem:sum_and_comp}, this is achieved via standard properties of
norms. Notably, iterative application of \cref{theorem:sum_and_comp}
allows bounding the covering numbers of any function class built from
compositions and additions of simpler classes.

In \cref{appendix:subsection:whole_network_general}, we apply
\cref{theorem:sum_and_comp} on two examples, \ie, (1) $f \in \mathcal
F = \set{f_L \circ \dots \circ f_1}$ and (2) $h  \in \mathcal H =\set{
g + h_L \circ \dots \circ h_1}$. Instantiating the first example for $f_i
= \sigma_i \circ \phi_i$, with $\sigma_i$ fixed and $\phi_i$ from
a family of linear maps, yields covering number bounds for networks as
in \cref{eq:simple_architecture}. As the second example corresponds to
residual blocks (with $g$ possibly the identity map), the combination of
(1) and (2) yields covering number bounds for residual networks; see
\cref{theorem:residual_networks}.

Overall, this strategy not only allows to derive covering number bounds
for a broad range of architectures, but also facilitates integrating
linkings between function spaces in a modular way. For
instance, \cref{appendix:lemma:concatenation} provides a variant of
\cref{theorem:sum_and_comp} for \emph{concatenation}, used in DenseNets
\citep{Huang17a}.

\paragraph*{Relation to prior work.}
He et al.\ \cite{HeF20a} investigate
covering number bounds for function spaces as considered above. They
present covering number bounds for residual networks and show that the
covering number $\mathcal N(\mathcal F, \epsilon, \norm{\cdot}_X)$ of
such models with layers $\mathcal F_\alpha$ is bounded by the product
$\prod_{\alpha} \sup_{\phi \in \mathcal G_\alpha} \mathcal N(\mathcal
F_\alpha, \epsilon_{\mathcal F_{\alpha}}, \norm{\cdot}_{\phi(X)})$ for
appropriately defined function spaces $\mathcal G_\alpha$. Yet, the
dependency of the whole-network covering radius $\epsilon$ on the
single-layer covering radii $\epsilon_{\mathcal F_\alpha}$ is
\emph{only} derived for a very specific residual network. Our  addition
to the theory is a more modular and structured way of approaching the
problem, which we believe to be valuable on its own.

\subsection{Covering number bounds for residual networks}
\label{subsection:resnets_cover}
We next state our whole-network covering number bounds for residual
networks and then present the corresponding bounds on the empirical Rademacher
complexity in \cref{subsection:resnets_rade}. Accompanying generalization guarantees (obtained via
\cref{thm:Bartlett17a_Lemma31}) are given in
\cref{appendix:subsection:generalization_bounds}. 
The results of this section hold for a hypothesis class~$\mathcal{F}$ of
networks implementing functions of the form 
\begin{equation}f =  \sigma_L \circ f_L
\circ \dots \circ\, \sigma_1 \circ f_1\quad \text{with}\quad  
f_i(x) = g_i(x) + (\sigma_{iL_i}\circ h_{iL_i}\circ\dots\circ \sigma_{i1} \circ h_{i1})(x)\enspace,
\label{eq:setting}
\end{equation}
\ie, a composition of $L$ residual blocks.
Here, the nonlinearities $\sigma_i$ and $\sigma_{ij}$ are fixed and $\rho_i$-, resp., $\rho_{ij}$-Lipschitz continuous with $\sigma_i(0)=0$ and $\sigma_{ij}(0)=0$. We further fix the shortcuts to maps with $g_i(0)=0$.
The map $h_{ij}$ identifies the $j$-th layer in the $i$-th residual block with Lipschitz constraints $s_{ij}$ and distance constraints $b_{ij}$ (\wrt reference weights $M_{ij}$). Specifically, if $h_{ij}$ is \emph{fully-connected}, then
\begin{equation}
h_{ij} \in \set{\phi:x \mapsto A_{ij}x~{\Bigl\vert}~ \Lip({\phi})\,{\le}\,s_{ij},\, \norm{A_{ij}^\top - M_{ij}^\top}_{2,1}\,{\le}\,b_{ij}}\enspace,
\end{equation}
and, in case $h_{ij}$ is \emph{convolutional}, then
\begin{equation}
    h_{ij} \in \set{\phi_{K_{ij}}~\Bigl\vert~ \Lip(\phi_{K_{ij}})\,{\le}\,s_{ij},~ \norm{K_{ij} - M_{ij}}_{2,1}\le b_{ij}}\,.
\end{equation}
In terms of notation, $s_i = \Lip(g_i) + \prod_{j=1}^{L_i} \rho_{ij}
s_{ij}$ further denotes the upper bound on the Lipschitz constant of the
$i$-th residual block $f_i$. The Lipschitz constants are \wrt
Euclidean norms; for a fully-connected layer this coincides with
the spectral norm of the weight matrix.

The covering number bounds in Theorem~\labelcref{paper:theorem:residual_network_covering} below depend on
three types of quantities: (1) the total number of layers $\bar L = \sum_{i} L_i$, (2)
the numbers $W_{ij}$ of parameters of the $j$-th layer in the $i$-th
residual block, their maximum $W = \max_{ij} W_{ij}$, and (3) terms $C_{ij}$
that quantify the part of a layer's capacity attributed to weight and
data norms. With respect to the latter, we define
\begin{equation}
    C_{ij}(X)=
    2\, \frac{\norm{X}}{\sqrt{n}} 
    \left(
        \prod_{\substack{l=1}}^{L}
            s_l \rho_l
    \right)
    {\frac
    {\prod_{\substack{k=1}}^{L_i} s_{ik}\rho_{ik}}     
    {s_i} 
    \frac{b_{ij}}{s_{ij}}}
\end{equation}
and write $C_{ij} = C_{ij}(X)$ for brevity. Importantly, $\norm{X} \le
\sqrt{n} \max_i \norm{x_i}$ and so the $C_{ij}$ can be bounded
independently of the sample size.
Overall, this yields the following covering number bounds for residual networks.
\vskip0.5ex
\begin{theorem}
    \label{paper:theorem:residual_network_covering}
    The covering number of the class of residual networks $\mathcal F$ as specified above, satisfies
    \begin{subnumcases}{\log \mathcal N(
        \mathcal F,
        \epsilon,
        \norm{\cdot}_{X}
    ) \leq}
    \log(2W)
    \Biggl(
        \sum_{i=1}^{{L}}
        \sum_{j=1}^{{L_i}}
        \ceil{C_{ij}^{{2}/{3}}}
    \Biggr)^{\!3}
        \ceil{\frac{n}{\epsilon^2}}
        \\
        \sum_{i=1}^{{L}}
        \sum_{j=1}^{{L_i}}
            2 W_{ij}
            \log\left(
                1 +
                \ceil{\bar L^2
                C_{ij}^2}
                \ceil{\frac{n}{\epsilon^2}}
            \right)\enspace.
     \end{subnumcases}
\end{theorem}

\subsection{Rademacher complexity bounds}
\label{subsection:resnets_rade}
In combination with Dudley's entropy integral, \cref{paper:theorem:residual_network_covering}
implies the empirical Rademacher complexity bounds in \cref{paper:theorem:rademacher_complexity}. These bounds equally hold for
non-residual networks as in \cref{eq:simple_architecture}, \ie, the special case of setting the
shortcuts $g_i$ to the zero map (with $L=1$ block). 

\begin{theorem}
    \label{paper:theorem:rademacher_complexity}
    Let $\gamma>0$ and define ${\tilde C}_{ij}= 2C_{ij}/\gamma$. Further, let $H_{n-1}= \sum_{m=1}^{n-1} \nicefrac{1}{m} = \mathcal O(\log(n))$ denote the $(n-1)$-th harmonic number. Then, the empirical Rademacher complexity of $\mathcal F_\gamma$ satisfies 
    \begin{equation}
        \label{theorem:main_residual_network_covering_rademacher:eq_norms}
        \textcolor{uniblue}
        {
            \rade{S}{\mathcal F_\gamma}
            \le
            \frac{4}{n}
            +
            \frac{12 H_{n-1}}{\sqrt{n}}
            \sqrt{\log(2W)}
            \left(
                \sum_{i=1}^L
                \sum_{i=j}^{L_i}
                \ceil{\tilde C_{ij}^{\nicefrac{2}{3}}}
            \right)^{\!\nicefrac{3}{2}}
        }
        \tag{\textcolor{uniblue}{$\clubsuit$}}
    \end{equation}
    \vspace{-0.1cm}
    and
    \begin{equation}
        \label{theorem:main_residual_network_covering_rademacher:eq_params}
        \textcolor{unired}
        {
            \rade{S}{\mathcal F_\gamma}
            \le
            \frac{12}{\sqrt n}
            \sqrt{
                \sum_{i=1}^L
                \sum_{i=j}^{L_i}
                2 W_{ij}
                \bigg(
                \log\left(1+\ceil{\bar L^2 \tilde C_{ij}^2}\right) + 
                \psi\left(\ceil{\bar L^2 \tilde C_{ij}^2}\right)
                \bigg)
            }
        }
        \enspace, 
        \tag{\textcolor{unired}{$\spadesuit$}}
    \end{equation}
    where $\psi$ is a monotonically increasing function, satisfying $\psi(0)=0$ and $\forall x: \psi(x)<2.7$.
\end{theorem}

The theorem considers the function class $\mathcal F_\gamma$ as defined in \cref{eqn:def:Fgamma}. As a consequence, the bounds depend on the quotients \smash{$\tilde C_{ij} = 2 C_{ij} /\gamma$}, which measure a layer's capacity (with respect to weight and data norms) relative to a classification margin parameter $\gamma$.
As we will see in the experiments, constraining the layers' Lipschitz constants and weight norms, allows to substantially reduce the quantities $C_{ij}$ while the margin parameter $\gamma$ decreases only moderately.

\cref{paper:theorem:rademacher_complexity} also immediately implies generalization bounds for $\mathcal F_\gamma$ via \cref{thm:Bartlett17a_Lemma31}. In a subsequent step one can gradually decrease the constraint strengths and invoke a union bound argument over the corresponding generalization bounds, as for example done in \cite[Lemma A.9]{Bartlett17a}. This yields a generalization bound which does not depend on a priori defined constraint strengths, but on the actual Lipschitz constants and group norms computed from a neural network's weights.

\vspace{-.25em}
\paragraph*{Interpretation.}
    To facilitate a clean comparison between the bounds in \cref{paper:theorem:rademacher_complexity}, we disregard the ceiling function and apply Jensen's inequality 
    to the first bound \eqref{theorem:main_residual_network_covering_rademacher:eq_norms}, yielding
    \begin{equation}
        \rade{S}{\mathcal F_\gamma} 
        \le \frac 4 n + 
        \frac{12 H_{n-1}}{\sqrt{n}}
        \sqrt{
            \log(2W) 
            \sum\nolimits_{ij}
            \bar L^2
            \tilde{C}_{ij}^2
        }
        \enspace.
        \label{eqn:boundinterpretation}
    \end{equation}
    Denoting $\tilde C = \max_{ij}\tilde C_{ij}$, \cref{eqn:boundinterpretation} reveals that the bounds essentially differ only in that \eqref{theorem:main_residual_network_covering_rademacher:eq_norms} depends on $(\log(2W)\bar L^2 \tilde C^2)^{\nicefrac{1}{2}}$ and \eqref{theorem:main_residual_network_covering_rademacher:eq_params} depends on $(2W \log(1+\bar L^2 \tilde C^2))^{\nicefrac 1 2}$. Thus, the question of \emph{which one is tighter}, hinges on the ratio of $2W$ and $\bar L^2 \tilde C^2$, i.e., a tradeoff between the number of parameters and the weight norms.
    As we see in \cref{fig:simple_comparison}, for simple, unconstrained networks, our second bound \eqref{theorem:main_residual_network_covering_rademacher:eq_params} is much tighter.
        However, due to the logarithmic dependency on \smash{$\tilde{C}$}, it is less affected by constraining the distances $b$ to initialization  and the Lipschitz constants $s$. In \cref{section:empirical_evaluation}, we show that this effect causes \eqref{theorem:main_residual_network_covering_rademacher:eq_norms} to be a more faithful measure of excess capacity.
        As \smash{$\tilde C$} depends exponentially on the network depth via the product of Lipschitz constants, another perspective on the bounds is that \cref{theorem:main_residual_network_covering_rademacher:eq_norms} favors shallow architectures whereas \cref{theorem:main_residual_network_covering_rademacher:eq_params} favors narrow architectures.
        Notably, replacing the function class $\mathcal F_\gamma$ with a class of networks composed with a Lipschitz augmented loss function \cite{Wei19a} facilitates deriving Rademacher complexity- and generalization bounds, which do not suffer from the exponential depth dependency via the product of Lipschitz constants.
        Instead, such bounds depend on data dependent empirical estimates thereof, which are typically much smaller.

    \paragraph*{Relation to prior work.}
    Prior works \cite{Neyshabur15a, Golowich18a,Gouk21a} that tie generalization to
    ($1,\infty$) group norms of matrices of fully-connected layers are
    equally applicable to convolutional networks without unfavorable
    scaling \wrt input size. In particular, for ($1,\infty$) group norms
    of $M_K^\top$, we have $\norm{M_K^\top}_{1,\infty} = \max_{o}
    \norm{K_{o\cdot\cdot\cdot}}_1$, i.e., the maximum $l_1$ norm over each
    (input channel, width, height) slice of $K$. Yet, due to the
    layer-peeling strategy common to these works, the bounds scale with
    the product of matrix group norms \emph{vs.}\ the product of Lipschitz constants (as in the \smash{$\tilde{C}_{ij}$} in \cref{paper:theorem:rademacher_complexity}) for covering number based strategies. While one can construct settings where the product of ($1,\infty$) group norms is smaller than the product of Lipschitz constants, this is typically not observed empirically, cf. Fig.~\ref{fig:simple_comparison}. Alternatively, Long \& Sedghi\ \cite{Long20a} derive a generalization bound which does not depend on $l_p$ norms or group norms, but only on the distance to initialization with respect to the spectral norm.
    Notably, an intermediate result in this reference yields a generalization bound of similar form as \eqref{theorem:main_residual_network_covering_rademacher:eq_params}, scaling with the logarithm of the product of Lipschitz constants and with the square root of the number of parameters, see \cref{appendix:comparison_long}.
    The distance to initialization then enters the main result \cite[Theorem 3.1]{Long20a} at the cost of a Lipschitz constraint on the initialization. We argue that \eqref{theorem:main_residual_network_covering_rademacher:eq_params} incorporates the distance to initialization more naturally, as it comes without constraints on the initialization itself. Further, it holds for any sample size $n$ and the numerical constants are explicit.
    Last, in the special case of fully-connected layers and no skip connections, 
    \eqref{theorem:main_residual_network_covering_rademacher:eq_norms} reduces to the Rademacher complexity bound from \cite{Bartlett17a}.
    Yet, there are three differences to this result: (i) a different numerical constant, (ii) the logarithm is replaced with a harmonic number, and (iii) there are no ceiling functions. 
    From our understanding, these modifications are equally necessary when proving the special case directly. 
    Nevertheless, these differences are only of minor importance, as they do not affect the asymptotic behavior of the bound. For more details, see \cref{appendix:comparison_bartlett}.

\section{Empirical Evaluation}
\label{section:empirical_evaluation}
\vspace{-0.2cm}
To assess the excess capacity of a neural network trained via a standard protocol on some dataset, we seek a hypothesis class that contains a network of the same architecture with comparable testing error but \emph{smaller} capacity.
Controlling capacity via the bounds in \cref{paper:theorem:rademacher_complexity} requires \emph{simultaneously} constraining the Lipschitz constants per layer and the ($2,1$) group norm distance of each layer's weight to its initialization.
We first discuss how to enforce the constraints. Then, we fix a residual network architecture and train on datasets of increasing difficulty while varying the constraint strengths.

\paragraph*{Capacity reduction.} 
Controlling hypothesis class capacity necessitates ensuring that optimization yields a network
parametrization that satisfies the desired constraints. To this end, we implement a variant of
\emph{projected} stochastic gradient descent (SGD) where, after a certain number of update steps, we
project onto the intersection of the corresponding constraint sets $\mathcal{C}_1$ and
$\mathcal{C}_2$. For convolutional layers, parametrized by tensors $K$, these are the convex sets
$\mathcal{C}_1\ {=}\ \{K: \lVert K-K^0 \rVert_{2,1} \leq b\}$ and $\mathcal{C}_2 = \{K: \Lip(\phi_K)
\leq s\}$. Hence, jointly satisfying the constraints is a convex feasibility
problem of finding a point in $\mathcal{C} = \mathcal{C}_1 \cap \mathcal C_2$. To ensure
$\mathcal{C} \neq \emptyset$, we initially (prior to optimization) scale each layer's weight
\smash{$K^0$} so that $\Lip(\phi_{K^0}) = s$. This starting point (per layer) resides in
$\mathcal{C}$ by construction.

To project onto $\mathcal C$, we rely on alternating \emph{orthogonal projections} which map $K$ to a tensor in $\mathcal C_1$, resp. $\mathcal C_2$ with minimal distance to $K$. Repeated application of these projections  converges to a point in the intersection $\mathcal C$ \cite{Bauschke96a}.
To implement the orthogonal projections onto $\mathcal C_1$ and $\mathcal C_2$, we rely on work by Liu et al.\ \cite{Liu09a} and Sedghi et al.\ \cite{Sedghi19a}, respectively. The latter requires certain architectural prerequisites, and in consequence, we need to use convolutions of  stride 1 (though our bounds equally hold for strides $>\!\!1$) and to reduce spatial dimensionality only via max-pooling.
Further, we use circular padding and kernel sizes not larger than the input dimensionalities.
For details on the projection algorithm and a comparison to alternating \emph{radial} projections, see \cref{appendix:subsection:projection_method}.

\paragraph*{Architecture.}
We use a slightly modified (pre-activation) ResNet18  \cite{He16b}. Modifications include: (1) the \emph{removal of batch normalization} and biases; (2) skip connections for residual blocks where the number of channels doubles and spatial dimensionality is halved are implemented via a fixed map. Each half of the resulting channels is obtained via $2 \times 2$ spatial max-pooling (shifted by one pixel). This map has Lipschitz constant \smash{$\sqrt{2}$} and is similar to the shortcut variant (A) of \cite{He16a}; finally, (3) 
we \emph{fix} the weight vectors of the last (classification) layer at the vertices of
a $\kappa-1$ unit simplex. Fixing the classifier is motivated by \cite{Hardt17a} and the simplex
configuration is inspired by \cite{Graf21a,Zhu21a} who show that this configuration corresponds to
the geometric weight arrangement one would obtain at minimal cross-entropy loss. By construction,
this classifier has Lipschitz constant \smash{$\sqrt{\nicefrac{\kappa}{\kappa-1}}.$}
Notably, modifications (2) and (3) \emph{do not} harm performance, with empirical testing errors on
a par with a standard ResNet18 without batch normalization. Modification (1), \ie, the omission of
normalization layers, was done to ensure that the experiments are in the setting of
\cref{eq:setting} and therefore that
(\ref{theorem:main_residual_network_covering_rademacher:eq_norms},
\ref{theorem:main_residual_network_covering_rademacher:eq_params}) are faithful capacity measures.
However, it is accompanied by a noticeable increase in testing error. In principle, our theory could
handle batch normalization, as, during evaluation, the latter is just an affine map parametrized by the running mean and variance learned during training. However, including normalization in our empirical evaluation is problematic, as normalizing batches of small variance
requires the normalization layers to have a large Lipschitz constant.
Consequently, considering normalization layers as affine maps and enforcing Lipschitz constraints on them could prevent proper normalization of the data. Another strategy would be to consider normalization layers as fixed nonlinearities which normalize each batch to zero mean and unit variance. However, this map is not Lipschitz continuous, and again, modifications could hinder normalization (which defeats the very purpose of these layers). Hence, we decided to remove normalization layers in our empirical evaluation. Presumably, however, there is a middle ground where capacity is reduced, and normalization is still possible. If so, excess capacity could be assessed for very deep architectures, which are difficult to train without normalization layers.

\paragraph*{Datasets \& Training.}
We test on three benchmark datasets: CIFAR10/100 \cite{Krizhevsky09a}, and Tiny-ImageNet-200 \cite{TinyImagnet}, listed in order of increasing task difficulty. We minimize the cross-entropy loss using SGD with momentum (0.9) and small weight decay (1e-4) for 200
epochs with batch size 256 and follow a CIFAR-typical stepwise learning
rate schedule, decaying the initial learning rate (of 3e-3) by a factor
of $5\times$ at epochs 60, 120 \& 160~\cite{DeVries17a}. \emph{No
data augmentation is used}. When projecting onto the constraint sets, we
found one alternating projection step every 15th SGD update to be
sufficient to remain close to $\mathcal{C}$. To ensure that a trained model is within the capacity-constrained class, we perform 15 \emph{additional} alternating projection steps after the final SGD update. For consistency, all experiments are run with the same hyperparameters. Consequently, hyperparameters are chosen so that training converges for the strongest constraints we assess. In particular, we train for 200 epochs even though unconstrained and weakly constrained models can be trained much faster. Importantly, this affects the assessment of excess capacity only marginally, as we observe that the testing error does not deteriorate in case of more update steps. Similarly, the Lipschitz constant and the distance to initialization stay almost constant once the close-to-zero training error regime is reached, which may happen way before 200 epochs are completed.

\subsection{Results}
\label{subsection:results}

First, we assess the capacity-driving quantities in our bounds for
models trained \emph{without} constraints\footnote{At evaluation time, Lipschitz constants are computed via a power iteration for convolutional layers \mbox{\cite{Gouk21b, Li19a}}.}.
\cref{table:numericalcapacity}~(top) lists a comparison across datasets,
along with the capacity measures
(\ref{theorem:main_residual_network_covering_rademacher:eq_norms},
\ref{theorem:main_residual_network_covering_rademacher:eq_params}), the
training/testing error, and the empirical generalization gap (\ie, the difference between testing and training error).  In
accordance with our motivating figure from \cref{section:introduction} (\cref{fig:intro}), we
observe an overall \emph{increase} in both capacity-driving quantities
as a function of task difficulty.

To assess excess capacity in the context of questions
(\hyperlink{Q1}{\textcolor{black}{\textbf{Q1}}}) and
(\hyperlink{Q2}{\textcolor{black}{\textbf{Q2}}}), we first identify, per
dataset, the \emph{most restrictive} constraint combination where the
testing error\footnote{
    We are primarily interested in \emph{what is feasible} in terms of tolerable capacity reduction. Hence, leveraging the testing split of each dataset for this purpose is legitimate from this exploratory perspective.
} is \emph{as close as possible} to the unconstrained regime. We refer
to this setting as the \emph{operating point} for the constraints, characterizing the function class $\mathcal F$ that serves as a reference to measure excess capacity. The operating points per dataset, as well as the corresponding results are listed in Table~\ref{table:numericalcapacity}~(bottom).

\begin{figure*}[t!]
    \begin{center}
    \includegraphics[width=\textwidth]{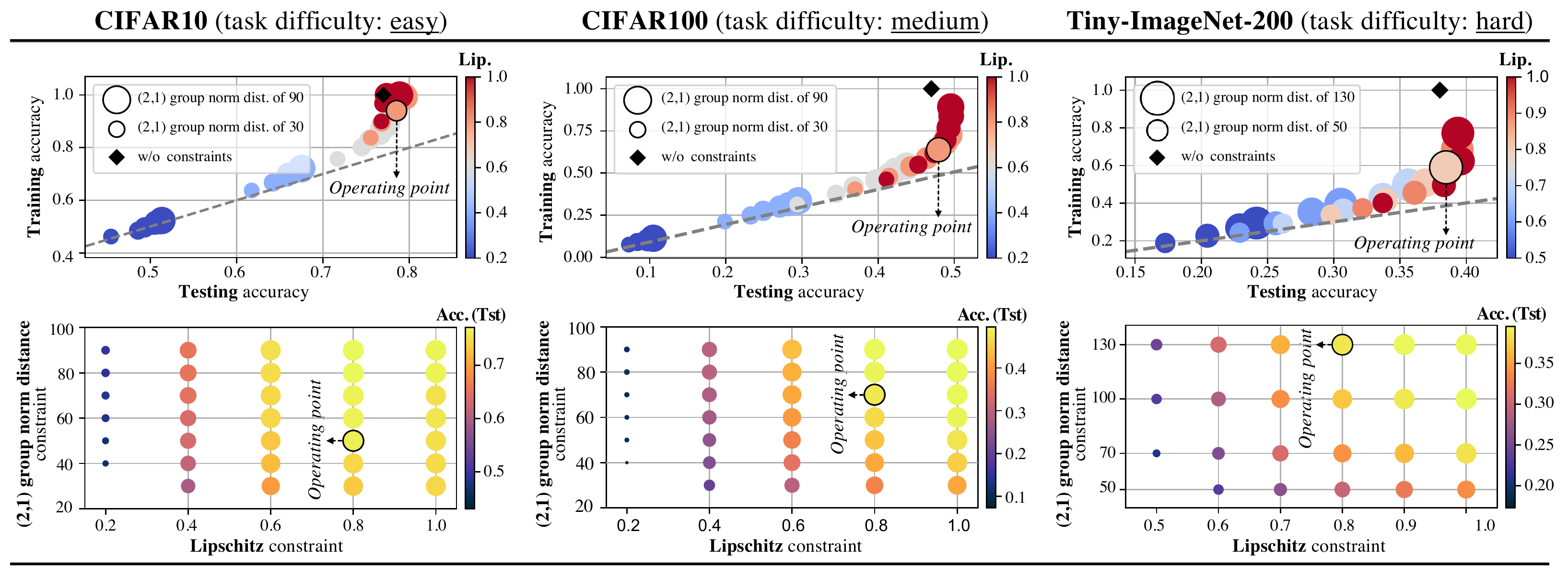}
    \end{center}
    \vspace{-0.1cm}
    \caption{Fine-grained analysis of training/testing accuracy in relation to the Lipschitz
    constraint and the ($2,1$) group norm distance to initialization constraint. We see that testing
    accuracy can be retained (relative to \raisebox{-.48ex}{\FilledDiamondshape}) for a range of
    fairly restrictive constraints (\emph{top row}), compared to the unconstrained regime (cf.
    \textbf{Lip./Dist.} columns in the top part of \cref{table:numericalcapacity}). However, this
    range noticeably narrows with increasing task difficulty (\emph{bottom row})
    \label{fig:detailedresults}.}
    \vspace{-0.2cm}
\end{figure*}
With respect to (\hyperlink{Q1}{\textcolor{black}{\textbf{Q1}}}) we find
that networks can indeed maintain, or even improve, testing error at
substantially lowered capacity (see performance comparison in \cref{fig:detailedresults} relative to \raisebox{-.48ex}{\FilledDiamondshape}\!).
Furthermore, the observation that the capacity of the constrained
models (surprisingly) remains in the \emph{same order} of magnitude across tasks of varying difficulty, suggests a negative answer to question (\hyperlink{Q2}{\textcolor{black}{\textbf{Q2}}}). A reduced-capacity model from an easy task can perform well on a difficult task. In consequence, when comparing the top \emph{vs.} bottom part of \cref{table:numericalcapacity} with respect to column (\ref{theorem:main_residual_network_covering_rademacher:eq_norms}), we do see that task difficulty primarily manifests as \emph{excess} capacity. Another manifestation of task difficulty is evident from the more
detailed analysis in \cref{fig:detailedresults} (bottom), where we see
that tightening both constraints beyond the identified operating point
leads to a more rapid deterioration of the testing error as the task
difficulty increases.
Interestingly, at the operating point, the constrained models do not only share similar capacity across datasets, but also similar empirical \emph{generalization gaps}, primarily due to leaving the ubiquitous zero-training-error regime. The latter is particularly noteworthy, as strong regularization (\eg, via weight decay) can equally enforce this behavior, but typically at the cost of a large
increase in testing error (which we do not observe). Finally, the parameter-counting variant of our bound (see \cref{table:numericalcapacity}, column \ref{theorem:main_residual_network_covering_rademacher:eq_params}) is, by construction, much less affected by the constraints and apparently fails to capture the observations above. This highlights the relevance of tying capacity to weight norms and underscores their utility in our context. 

\vspace{-0.2cm}
\begin{table}[t!]
    \caption{Assessment of the capacity-driving quantities. We list the \emph{median} over the Lipschitz constants (\textbf{Lip.}) and the ($2,1$) group norm distances (\textbf{Dist.}) across all layers. \textbf{Err.} denotes the training/testing error, \textbf{Capacity} denotes the measures (\ref{theorem:main_residual_network_covering_rademacher:eq_norms}, \ref{theorem:main_residual_network_covering_rademacher:eq_params})
    from \cref{paper:theorem:rademacher_complexity} and \textbf{Gap} the empirical generalization gap. The top part lists results in the unconstrained regime (see \raisebox{-.48ex}{\FilledDiamondshape}  in \cref{fig:detailedresults}), the bottom part lists results at the 
    operating point of the \underline{most restrictive} constraint
  combination where the testing error is not worse than in the unconstrained case. \textbf{Mar.} denotes the margin parameter $\gamma$ used for computing the capacity measures. As the constrained models do not fit the training data anymore, they do not have a positive classification margin. Thus, we choose $\gamma$ such that the unconstrained and constrained models have the same ramp loss value.
  \label{table:numericalcapacity}}
    \begin{small}
    \begin{center}
        \setlength{\tabcolsep}{8pt}
    \begin{tabular}{rccccccc}
        \toprule
                & \textbf{Lip.} 
                & \textbf{Dist.} 
                & \textbf{Mar.}
                & \textbf{Err. (Tst)} 
                & \textbf{Err. (Trn)}
                & \textbf{Capacity} (\ref{theorem:main_residual_network_covering_rademacher:eq_norms}, \ref{theorem:main_residual_network_covering_rademacher:eq_params})
                & \textbf{Gap} \\
        \midrule
        CIFAR10         
            & $1.63$ 
            & $60.3$  
            & $11.2$
            & $0.24$
            & $0.00$
            & \textcolor{uniblue}{1.0$\cdot$10$^{10}$} / \textcolor{unired}{8.8$\cdot$10$^2$}  & $0.24$ 
            \\ 
        CIFAR100 
            & $2.17$ 
            & $129.1$ 
            & $23.4$
            & $0.54$
            & $0.00$
            & \textcolor{uniblue}{1.7$\cdot$10$^{11}$} / \textcolor{unired}{9.3$\cdot$10$^2$} & $0.54$
            \\ 
        Tiny ImageNet
            & $3.05$ 
            & $287.3$ 
            & $24.7$
            & $0.62$ 
            & $0.00$
            & \textcolor{uniblue}{4.5$\cdot$10$^{13}$} / \textcolor{unired}{1.1$\cdot$10$^3$} & $0.62$
            \\ 
        \midrule
        \midrule
        CIFAR10 
            & $\underline{0.80}$ 
            & $\underline{50.0}$  
            & $1.00$
            & $0.21$ 
            & $0.06$
            & \textcolor{uniblue}{1.8$\cdot$10$^8$} / \textcolor{unired}{7.8$\cdot$10$^2$} & $0.15$
            \\
        CIFAR100 
            & $\underline{0.80}$ 
            & $\underline{70.0}$  
            & $1.00$
            & $0.52$
            & $0.36$
            & \textcolor{uniblue}{2.6$\cdot$10$^8$} / \textcolor{unired}{7.9$\cdot$10$^2$}  
            & $0.16$ 
            \\
        Tiny ImageNet  
            & $\underline{0.80}$ 
            & $\underline{130.0}$ 
            & $1.00$
            & $0.62$ 
            & $0.41$ 
            &  \textcolor{uniblue}{\,8.9$\cdot$10$^8$} / \textcolor{unired}{8.9$\cdot$10$^2$} & $0.21$
            \\ 
        \bottomrule
    \end{tabular}
    \end{center}
\end{small}
\vspace{-0.4cm}
\end{table}

\vspace{-0.05cm}
\section{Discussion}
\label{section:discussion}
\vspace{-0.1cm}Studying the capacity of neural networks hinges crucially on the measure
that is used to quantify it. In our case, capacity rests upon two bounds on the
empirical Rademacher complexity, both depending on weight norms and the
number of parameters, but to different extents. Hence, exerting control
over the \emph{weight norms} manifests in different ways: in case of the
more weight norm dependent capacity measure, our results show 
substantial task-dependent \emph{excess capacity}, while, when relying
more on parameter counting, this effect is less pronounced. Although the
latter measure yields tighter bounds, its utility in terms of explaining
the observed empirical behavior is limited: in fact, capacity tied to
weight norms not only better correlates with observed generalization
gaps~(both with and without constraints), but the amount of tolerable
capacity reduction also reflects the smaller generalization gaps in the
constrained regime. Note that our results rest upon carefully implementing
constraint enforcement during optimization. Hence, numerical
schemes to better account for this setting might potentially reveal an even
more pronounced excess capacity effect. 

In summary, our experiments, guided by the theoretical bounds, strongly suggest a notion of \emph{compressibility} of networks with respect to weight norms. This compressibility only moderately reduces with task difficulty. We believe these observations to be particularly relevant and we foresee them sparking future work along this
direction.
\subsubsection*{Acknowledgments}
This work was supported by the Austrian Science Fund (FWF) under project FWF P31799-N38 and the Land Salzburg under projects 0102-F1901166-KZP and 
20204-WISS/225/197-2019.

\clearpage
\listofchanges

\clearpage
\bibliography{references.bib, self.bib}
\bibliographystyle{plain}
\clearpage

\section*{Checklist}


\begin{enumerate}

\item For all authors...
\begin{enumerate}
  \item Do the main claims made in the abstract and introduction accurately reflect the paper's contributions and scope? \answerYes{}
  \item Did you describe the limitations of your work?
    \answerYes{\emph{\textcolor{black!50}{For all presented theoretical results, we discuss connections, as well as limitations or advantages/disadvantages with respect to prior work. Furthermore, our supplementary material (Appendix) contains substantial extensions of these discussions. Regarding our experimental results, we discuss potential limitations in \cref{section:empirical_evaluation} and \cref{section:discussion}.}}}
  \item Did you discuss any potential negative societal impacts of your work?
    \answerNA{\emph{\textcolor{black!50}{All experiments throughout this work are conducted on well-established, publicly available benchmark datasets for image classification. We primarily address theoretical aspects of existing neural network models for which we do not see a potential negative societal impact.}}}
  \item Have you read the ethics review guidelines and ensured that your paper conforms to them?
    \answerYes{}
\end{enumerate}

\item If you are including theoretical results...
\begin{enumerate}
  \item Did you state the full set of assumptions of all theoretical results?
    \answerYes{\emph{\textcolor{black!50}{All key theoretical results (theorems, lemmas, etc.), including assumptions and definitions are listed in the main part of the paper.}}}
        \item Did you include complete proofs of all theoretical results?
    \answerYes{\emph{\textcolor{black!50}{Proofs for all theoretical results listed in the main part of the paper (theorems, lemmas, etc.) can be found in the supplementary material (Appendix), including any omitted technical details.}}}
\end{enumerate}

\item If you ran experiments...
\begin{enumerate}
  \item Did you include the code, data, and instructions needed to reproduce the main experimental results (either in the supplemental material or as a URL)?
    \answerYes{\emph{\textcolor{black!50}{Code to re-run experiments and reproduce results is included in the supplementary as part of this submission.}}}
  \item Did you specify all the training details (e.g., data splits, hyperparameters, how they were chosen)?
    \answerYes{\emph{\textcolor{black!50}{We include training/architectural details in \cref{section:empirical_evaluation} of the main part of the paper. Details for our initial bound comparison in \cref{fig:simple_comparison} can be found in the supplementary material (Appendix). Throughout all experiments, we adhere to the common training/testing splits of the three datasets we used, i.e., CIFAR10/100 and Tiny-ImageNet-200.}}}
        \item Did you report error bars (e.g., with respect to the random seed after running experiments multiple times)?
    \answerNA{\emph{\textcolor{black!50}{We did run experiments multiple times (using the same network, dataset and training regime), but under varying constraint strengths (rather than different seeds). \cref{fig:detailedresults} visualizes \underline{all} these runs.}}}
        \item Did you include the total amount of compute and the type of resources used (e.g., type of GPUs, internal cluster, or cloud provider)?
    \answerYes{\emph{\textcolor{black!50}{\cref{appendix:subsection:hardwareresources} lists all hardware resources used in our experiments.}}}
\end{enumerate}

\item If you are using existing assets (e.g., code, data, models) or curating/releasing new assets...
\begin{enumerate}
  \item If your work uses existing assets, did you cite the creators?
    \answerYes{}
  \item Did you mention the license of the assets?
    \answerNo{\emph{\textcolor{black!50}{The datasets we used are publicly available and well-established benchmark datasets commonly used in computer vision research.}}}
  \item Did you include any new assets either in the supplemental material or as a URL?
    \answerNo{}
  \item Did you discuss whether and how consent was obtained from people whose data you're using/curating?
    \answerNA{\emph{\textcolor{black!50}{All datasets in our experiments are used as is.}}}
  \item Did you discuss whether the data you are using/curating contains personally identifiable information or offensive content?
    \answerNA{}
\end{enumerate}

\item If you used crowdsourcing or conducted research with human subjects...
\begin{enumerate}
  \item Did you include the full text of instructions given to participants and screenshots, if applicable?
    \answerNA{}
  \item Did you describe any potential participant risks, with links to Institutional Review Board (IRB) approvals, if applicable?
    \answerNA{}
  \item Did you include the estimated hourly wage paid to participants and the total amount spent on participant compensation?
    \answerNA{}
\end{enumerate}

\end{enumerate}

\newpage
\begin{appendices}
\onecolumn

\DoToC
\vskip2ex
In the appendix, we present (1) a more detailed comparison to prior work
(\cref{appendix:section:detailed_comparison}), (2) additional experiments
(\cref{appendix:section:additionalexperiments}) and (3) list all proofs which are left-out in the
main manuscript (\cref{appendix:section:proofs}). In particular, in
\cref{appendix:subsection:singlelayerbounds}, we derive our \emph{single-layer covering number
bounds} from \cref{paper:theorem:single_layer}; \cref{appendix:subsection:whole_network_general}
presents the modularized strategy from \cref{subsection:whole_network} to obtain
\emph{whole-network} covering number bounds. This includes several examples (e.g., residual
networks) and an extension to accommodate \emph{concatenation}.
\cref{appendix:subsection:multi_layer_specific} then tailors our empirical Rademacher complexity
bounds to networks with fully-connected and convolutional layers, and
\cref{appendix:subsection:generalization_bounds} finally lists the accompanying generalization
bounds.
\clearpage

\section{Comparison with prior work on Rademacher complexity bounds}
\label{appendix:section:detailed_comparison}

\subsection{Analysis of matrices corresponding to convolutions}
\label{subsection:norm_of_M_k}

We compare the norms of the matrices $M_K$ corresponding to the linear map realized by a convolutional layer with the norm of the corresponding weight tensor $K$. This facilitates studying the Rademacher complexity of convolutional layers via norm-based bounds for fully-connected layers.

In accordance with the definition of the ($2,1$) group norm in \cref{eqn:def_21_norm}, we define the ($p,q$) group norm of a weight tensor $K$ as
\begin{equation}
    \norm{K}_{p,q} = \left(
        \sum_{a,b=1}^k \sum_{o=1}^{c_\textit{out}}
        \norm{K_{o\cdot a b}}_p^q
    \right)^{1/q}
    \enspace.
\end{equation}
For simplicity, we
\begin{inparaenum}[(i)]
  \item consider only circular, input-size preserving paddings,
  \item assume that the spatial input dimension $h=w=d$ is a multiple of the convolution stride $t$, and
  \item assume that the kernel size $k_h=k_w= k\le d$.
\end{inparaenum}
In this setting, convolution corresponds to the application of a local map $\mathbb R^{c_\textit{in} \times k\times k} \to \mathbb R^{c_\textit{out}}$, specified by the weight tensor $K\in \mathbb R^{c_\textit{out}\times c_\textit{in}\times k \times k}$, to all $(d/t)^2$ patches. 
Each row $[M_K]_{a\cdot}$ of $M_K\in \mathbb R^{c_\textit{out} (d/t)^2 \times c_\textit{in} d^2}$ has as non-zero elements only the entries of the tensor $K_{o_a\cdot\cdot\cdot}$ for some $o_a=1,\dots, c_\textit{out}$. Thus, the ($p,q$) group norm of $M_K^\top$ is
\begin{equation}
    \norm{M_K^\top}_{p,q} 
    \!= 
    \left(
        \sum_{a=1}^{c_\textit{out} (d/t)^2}
        \!\!\!\norm{[M_K]_{a\cdot}}_p^q
    \right)^{\!\!1/q}
    \!\!\!\!{=} 
    \left(
        \sum_{a=1}^{c_\textit{out} (d/t)^2}\!\!\!
        \norm{K_{o_a\cdot\cdot\cdot}}_p^q
    \right)^{\!\!1/q}
    \!\!\!\!{=} 
    \left(\frac{d}{t}\right)^{\!\!2/q}\!
    \left(
        \sum_{o=1}^{c_\textit{out}}
        \norm{K_{o\cdot\cdot\cdot}}_p^q
    \right)^{\!\!1/q}
    \,.
\end{equation}
In particular,
\begin{align}
    \norm{M_K^\top}_{2,1} 
    &=
    \left(\frac{d}{t}\right)^{2}
        \sum_{o=1}^{c_\textit{out}}
        \norm{K_{o\cdot\cdot\cdot}}_2
    \enspace,
    \\
    \norm{M_K^\top}_{2} 
    &=
    \frac{d}{t} \norm{K}_2
    \enspace,
    \\
    \norm{M_K^\top}_{1,\infty} 
    &= 
    \max_{o} \norm{K_{o\cdot\cdot\cdot}}_1
    \enspace.
\end{align}
Note the benefit of the ($1,\infty$) group norm, which does not scale with the input dimension $d$.

We point out that for $p> q$ (Hölder inequality for $p/q$)
\begin{equation}
    \norm{K_{o\cdot\cdot\cdot}}_p^q
    =
    \left(
        \sum_{a,b=1}^{k}
        \norm{K_{o\cdot a b}}_p^p
    \right)^{q/p} 
    \ge
    \frac{1}{k^{2(1-q/p)}} 
    \sum_{a,b=1}^{k}
    \norm{K_{o\cdot a b}}_p^q
\end{equation}
and so 
\begin{equation}
    \begin{split}
    \norm{M_K^\top}_{p,q} 
    & = 
    \left(
        \left(\frac{d}{t}\right)^{2}
        \sum_{o=1}^{c_\textit{out}}
        \norm{K_{o\cdot\cdot\cdot}}_p^q
    \right)^{1/q} \\
    & \ge 
    \left(
        \frac{d^2}{t^2 k^{2(1-q/p)}} 
        \sum_{o=1}^{c_\textit{out}}
        \sum_{a,b=1}^{k}
        \norm{K_{o\cdot a b}}_p^q
    \right)^{1/q} \\
    & =
    \left(\frac{d}{tk}\right)^{2/q} {k^{2/p}}
    \norm{K}_{p,q}
    \enspace.
    \end{split}
\end{equation}
\emph{This inequality quantifies the disadvantage of applying generalization bounds for fully-connected layers directly on the matrices that parametrize the linear maps corresponding to convolutions.}

\subsection{Comparison of bounds for convolutional networks}
\label{appendix:subsection:comparison_table}
\cref{table:bound_comparison} lists various upper bounds on the empirical Rademacher complexity of convolutional networks as specified in the paragraph below, in a common notation (see Table~\ref{table:notation}).
As in \cref{subsection:norm_of_M_k}, we consider input-size preserving circular padding, convolutions with stride of $t$, and assume that (i) the spatial dimensions $h=w=d$ are a multiple of $t$ and (ii) the kernel size is $k_h=k_w=k\le d$.
\begin{table}[t!]
    \begin{small}
    \caption{\label{table:notation}Notation}
    \begin{center}
    \begin{tabular}{llll}
        \toprule
        $n$   & Number of samples & $\kappa$ & Number of classes \\
        $L$   & Number of layers  & $c_i$ & Number of input channels to layer $i$\\
        $k_i$ & Kernel size at layer $i$ &  $t_i$ & Stride at layer $i$\\
    $K_i$ & Weight of layer $i$, $K_i \in \mathbb R^{c_{i+1} \times c_i \times k_i \times k_i}$ & $K_i^{(0)}$ & Initialization of layer $i$\\
        $d_i$ & Input spatial width at layer $i$ & $s_i$ & Lipschitz constant of layer $i$\\
        $\gamma$ & Margin at output & $W_i$ & Number of parameters of layer $i$ \\
        \bottomrule
    \end{tabular}
    \end{center}
    \end{small}
\end{table}
\vskip2ex

\begin{table}[t!]
\begin{footnotesize}
\caption{\label{table:bound_comparison} Comparison of empirical Rademacher complexity bounds (in our notation). When referring to sections or theorems 
from references, we \underline{underline} the corresponding results.
}
\begin{center}
\resizebox{\textwidth}{!}{
    \begin{tabular}{ll}
    \toprule
    \makecell[l]{\textbf{Ours} \\
        \textcolor{black!50}{(see Thm. \ref{paper:theorem:rademacher_complexity} \eqref{theorem:main_residual_network_covering_rademacher:eq_norms})}}
    &
    $\frac{4}{n}
    +
    \frac{12 H_{n-1}}{\sqrt{n}}
    \sqrt{\log(2W)}
    \left(
        \sum_{i=1}^L
        \ceil{
            \left(                
                \frac 4 \gamma \, \frac{\norm{X}}{\sqrt{n}} 
                \left(
                    \prod_{\substack{l=1}}^{L}
                        s_l
                \right)
                \frac{
                    \norm{K_i - K_i^{(0)}}_{2,1}
                    }{s_{i}}
            \right)^{\nicefrac{2}{3}}
            }
    \right)^{\!\nicefrac{3}{2}}$   \\
    \midrule
    \makecell[l]{Bartlett et al. \tablefootnote{The numerical constant (48) differs from the one (36) in \cite{Bartlett17a} as discussed in \cref{appendix:comparison_bartlett}.}\ \cite{Bartlett17a}\\
    \textcolor{black!50}{(\underline{see Lem. A.8})}} & 
    $\frac{4}{n}
    +
    \frac{48}{\gamma} \frac{\norm{X}}{\sqrt n}
    \left( \prod_{i=1}^L s_i\right)
    \left(
        \sum_{i=1}^L
        \left(
            \log\left(\frac{2 W_i d_i^2}{t_i^2 k_i^2}\right) 
            \frac{d_i^4}{t_i^4}
            \frac{
                \left( 
                    \sum_{o=1}^{c_{i+1}}\norm{[K_i-K^{(0)}_i]_{o\cdot\cdot\cdot}}_2
                \right)^2
                }{s_i^2}
        \right)^{1/3}
    \right)^{3/2}
    {\frac{\log(n)}{\sqrt{n}}}$  \\
    \midrule    
    \multirow{5}{*}{\makecell[l]{Ledent et al.\ \cite{Ledent21a} \\ 
        \textcolor{black!50}{(main result, see \underline{Thm. 11})}}} & 
    $
    {768} R \sqrt{\log_2(32 \Gamma n^2 + 7 \bar W n)} \, \frac{\log(n)}{\sqrt{n}}
    $
    \\
    &
    $R = \left( \sum_{i=1}^L r_i^{2/3}\right)^{3/2}\enspace,$
    \qquad
    $\Gamma = \max_i \left( r_i d_{i+1}^2 c_{i+1}\right)\enspace,$
    \qquad $\bar W = \max_i 
    {d_i^2} c_{i}
    $
    \\
    & 
    $r_i = a_i B_{i-1}(X) \rho_{i+}$
    \\
    &$a_{i} = \sum_o \norm{[K_i-K_i^{(0)}]_{o\cdot\cdot\cdot}}_2\enspace,$
    \qquad $a_L = \norm{K_L - K_L^{(0)}}_2$\\
    & $\rho_{i+}
    = d_{i+1} \max_{U\le L}  \frac{\prod_{u=l+1}^U s_u}{B_U(X)}\enspace,$
    \qquad $\rho_{L+} = \frac{1}{\gamma}$\\
    &
    $B_{i-1}(X) = $ maximal $l_2$ norm of a convolutional patch of the inputs to the $i$-th layer
    \\    
    \midrule    
    \multirow{5}{*}{\makecell[l]{Ledent et al.\ \cite{Ledent21a} \\
        \textcolor{black!50}{(simpler result, see \underline{Sec. E})}}} & 
    $
        \frac{4}{n}
        +
        {768} R \sqrt{\log_2(32 \Gamma n^2 + 7 \bar W n)} \, \frac{\log(n)}{\sqrt{n}}$
        \\
        & $R = \left( \sum_{i=1}^L r_i^{2/3}\right)^{3/2}\enspace,$
        \qquad
        $\Gamma = \max_i \left( r_i d_{i+1}^2 c_{i+1}\right)\enspace,$
        \qquad
        $\bar W = \max_i d_{i}^2 c_{i}$
        \\
        & 
        $r_i = \frac{|X|_0}{\gamma}
            \max_o \norm{[K_L]_{o\cdot\cdot\cdot}}_2
            \left(\prod_{j=1}^{L-1} s_j\right)
            d_{i+1} 
            \frac{\sum_o \norm{
                    [K_i - K_i^{(0)}]_{o\cdot\cdot\cdot}
                    }_2
                }{s_i}
        $
        \\
        & $|X|_0 = B_0(X) $ maximal $l_2$ norm of a convolutional patch on $X$
    \\
    \midrule  
    \multirow{3}{*}{\makecell[l]{\textbf{Ours} \\\textcolor{black!50}{(see Thm. \ref{paper:theorem:rademacher_complexity} \eqref{theorem:main_residual_network_covering_rademacher:eq_params})}}} &
    $12
            \sqrt{
                \sum_{i=1}^L
                2 W_{i}
                \bigg(
                \log\left(1+\ceil{L^2 \tilde C_{i}^2}\right) + 
                \psi\left(\ceil{L^2 \tilde C_{i}^2}\right)
                \bigg)
            }
        \frac{1}{\sqrt{n}}
    $
    \\
    &
    $\tilde C_i 
    =
    \frac{4}{\gamma}
    \frac{\norm{X}}{\sqrt{n}} 
    \left(
        \prod_{\substack{i=l}}^{L}
            s_l
    \right)
    \frac{\norm{K_i - K_i^{(0)}}_{2,1}}{s_{i}}
    $\enspace, \qquad
    $\psi(x) = \zeta\left(\frac 3 2,1\right)^{1/3}
        \zeta\left(
            \frac 3 2,
            1+ {1}/{x}
        \right)^{2/3} < 2.7
$\\
    \midrule
    \makecell[l]{Lin et al.\ \cite{Lin19a}\\ \textcolor{black!50}{(see \underline{Lem. 18})}}
    &
    $ 16 
    \left(
        \frac 2 \gamma \frac{\norm{X}_2}{\sqrt{n}} L^2
        \left( \prod_{i=1}^L s_i \right) 
        \left(
            \sum_{i=1}^L  W_i^2 \frac{d_i}{t_i} \frac{\norm{K_i}_2}{s_i} 
        \right)
    \right)^{1/4}    
    \frac{1}{\sqrt{n}}
    $\\
    \midrule
    \makecell[l]{Neyshabur et al.\ \cite{Neyshabur15a}\\ 
    \textcolor{black!50}{(see \underline{Cor. 2}, with $l_{1,\infty}$)}}
    &
    $
    2^L \kappa
    \left(\prod_{i=1}^L \max_{o} \norm{[K_i]_{o\cdot\cdot\cdot}}_1\right)
    \log(2c_1 d_1^2) \max_{k} \norm{x_k}_{\infty}
    \frac{1}{\sqrt{n}}
    $
    \\
    \midrule
    \makecell[l]{Golowich et al.\ \cite{Golowich18a} \\ 
        \textcolor{black!50}{(see \underline{Thm. 2}, $l_{1,\infty}$)}} 
    &
    $2 \kappa \sqrt{L +1 + \log(c_1 d_1^2)} 
    \left(\prod_{i=1}^L \max_{o} \norm{[K_i]_{o\cdot\cdot\cdot}}_1\right)
    \sqrt{
        \frac{\max_{abc} \sum_{k=1}^n [x_k]_{abc}^2}{n}
    }
    \frac{1}{\sqrt n}$ \\
    \midrule
    \makecell[l]{Gouk et al.\ \cite{Gouk21a}\\ 
        \textcolor{black!50}{(see \underline{Thm. 1}, with $l_{1,\infty}$)}}
    &
    $
    \splitfrac{2^{L+1} \kappa
    \sqrt{\log(2c_1 d_1^2)}
    \left(\prod_{i=1}^L \max_{o} \norm{[K_i]_{o\cdot\cdot\cdot}}_1\right)}{
    \left(
        \sum_{i=1}^L 
        \frac{
            \max_o
            \norm{
                [K_i - K_i^{(0)}]_{o\cdot\cdot\cdot}
            }_1
        }
        {\max_o  \norm{[K_i]_{o\cdot\cdot\cdot}}_1}
    \right)
    \max_k \norm{x_k}_{\infty}
    \frac{1}{\sqrt{n}}}
    $ \\
    \midrule
    \makecell[l]{Neyshabur et al.\ \cite{Neyshabur15a} \\
        \textcolor{black!50}{(see \underline{Cor. 2}, with $l_{2}$)}}
    &
    $
    2^{L-1}{\kappa}
    \frac{\norm{X}_2}{\sqrt{n}}
    \left(\prod_{i=1}^L \frac{d_i}{t_i} \norm{K_i}_2 \right)
    \frac{1}{\sqrt{n}}
    $
    \\
    \midrule
    \makecell[l]{Golowich et al.\ \cite{Golowich18a} \\
        \textcolor{black!50}{(see \underline{Thm. 1}, with $l_{2}$)}}
    &
    $
    \kappa
    \frac{\norm{X}_2}{\sqrt{n}}
    \left(\prod_{i=1}^L \frac{d_i}{t_i} \norm{K_i}_2 \right)
    (\sqrt{2 \log(2) L} +1 ) \frac{1}{\sqrt n}
    $
    \\
    \midrule
    \makecell[l]{Gouk et al.\ \cite{Gouk21a} \\
        \textcolor{black!50}{(see \underline{Thm. 2}, with $l_{2}$)}}
    &
    $
    2^{L} \sqrt{2}
    \kappa   
    \frac{\norm{X}_2}{\sqrt{n}}
    \left(
        \prod_{i=1}^L 
        \frac{d_i^2}{t_i} \sqrt{c_{i}}\norm{K_i}_2 
    \right)
    \left(
        \sum_{i=1}^L
        \frac{
            \norm{K_i-K_i^{(0)}}_2
        }
        {
            \norm{K_i}_2
            \prod_{j=1}^i
            d_j \sqrt{c_{j}}
        }
    \right)
    \frac{1}{\sqrt n}
    $\\
    \bottomrule
    \end{tabular}
}
\end{center}
\end{footnotesize}
\end{table}

The empirical Rademacher complexity bounds in \cref{table:bound_comparison} are formulated in dependence of norms of the weights $K_i$. 
This is not entirely accurate, as the bounds typically refer to neural networks whose weights satisfy a priori defined norm \emph{constraints}. We choose this abuse of notation so that we do not need to introduce additional variables for each norm constraint. 

The listed bounds are used in the numerical comparison for unconstrained networks in Fig.~\ref{fig:simple_comparison}. More specifically, we consider a hypothesis class $\mathcal F$ represented by a neural network of the form $f = \sigma_L \circ f_L \circ \dots \circ \sigma_1 \circ f_1$, where $\sigma_i: x\mapsto \max(x,0)$ denotes the ReLU activation function and $f_i$ identifies a convolutional layer. Note that fully-connected layers, e.g., a linear classifier at the last layer, can be handled by setting the spatial input dimension $d_i$, the kernel size $k_i$ and the stride $t_i$ all equal to $1$.

Bounds designed for fully-connected networks are applied to the matrices $M_K$ that correspond to the weight tensor $K$ which parametrizes the convolution. To handle the multiclass regime, the covering number based bounds will be applied to $\mathcal F_\gamma$, see \cref{eqn:def:Fgamma}. Layer-peeling based bounds, originally presented for binary classification, are multiplied by the number of classes $\kappa$ (according to \cite{Maurer16a}) as done in \cite{Gouk21a}.

\subsection{Details for the numerical comparison in \texorpdfstring{\protect\cref{fig:simple_comparison}}{Fig. \ref{fig:simple_comparison}}}
\label{appendix:subsection:numerical_comparison_details}
In Fig.~\ref{fig:simple_comparison}, we evaluated several existing upper bounds (see~\cref{table:bound_comparison}) on the empirical Rademacher complexity of convolutional networks  for two specific architectures. 

The \emph{first architecture}~(a 6-layer network) consists of 5 convolutional layers with stride 2, kernel size 3, padding of 1 and 256 filters / output channels, so that a $(3\times 32 \times 32)$-dimensional input image is mapped to a $(256\times1\times1)$-dimensional representation. The subsequent linear classifier is a convolutional layer with kernel size 1 and no padding. Its number of filters equals the number of classes of the classification problem. This classification layer is equivalent to applying a fully-connected layer to the flattened representations.
The \emph{second architecture}~(an 11-layer network) only differs in that each convolutional layer with stride 2 is followed by an additional convolutional layer with stride 1 (kernel size 3, padding of 1 and 256 filters). All activation functions are ReLUs.

We trained both networks on the CIFAR10 dataset, minimizing the cross-entropy loss using stochastic gradient descent (SGD) with batch size 256, weight decay (1e-4), and momentum (0.9). During the 100 training epochs, the learning rate is gradually reduced following a cosine annealing schedule, starting with an initial learning rate of 1e-4. \emph{Notably, we do not use any data augmentation}.
Both networks fit the training data, achieving an accuracy of 72.5\% (6-layer), resp. 77.9\% (11-layer), on the test data.

To assess the different empirical Rademacher complexity bounds, we measured weight norms of the networks' layers and inserted them into the bounds from \cref{table:bound_comparison}. \cref{fig:appendix:simple_comparison} illustrates the results. Note that, following our discussion in \cref{appendix:subsection:comparison_table}, Rademacher complexity bounds are typically formulated for networks with a priori specified weight norm constraints, whereas here, we train \emph{unconstrained} networks and merely measure the weight norms at the end of training.

\cref{fig:appendix:simple_comparison} highlights several aspects of the studied bounds. First of
all, \textbf{all bounds are vacuous}, as they are larger than $1$. Somewhat surprisingly, the bounds
mainly driven by the number of parameters are clearly the smallest, \ie, the ones from Lin et al. \cite{Lin19a} and
\cref{paper:theorem:rademacher_complexity}
\eqref{theorem:main_residual_network_covering_rademacher:eq_params}.
As expected, our bound from \cref{paper:theorem:rademacher_complexity}
\eqref{theorem:main_residual_network_covering_rademacher:eq_norms} is smaller than 
\cite{Bartlett17a}, as it accounts for the structure of convolutions. Furthermore, we see the benign
scaling of the product of Lipschitz constants with the network depth compared to the product of
$l_2$ norms, resp., $l_{1,\infty}$ norms. This is mirrored in the benign scaling of the covering
number based bounds compared to the layer-peeling based ones.

\begin{figure}[t!]
    \begin{center}
        \begin{tikzpicture}
    \tikzset{
    font={\scriptsize}
    }

    \tikzmath{
        \width = 0.09; 
        \sep = .265; 
        \xscale = .24; 
        \nbars = 16; 
        \ticksize = 0.1; 
        \bracepos = 22; 
    } 

    \newcommand\barplot[4]{%
    \foreach \coun/\lab/\val in #1{%
        \path[fill, #2]%
        (0,{\coun*\sep+\width})
        -- ({\val*\xscale},{\coun*\sep+\width})
        -- ({\val*\xscale}, {\coun*\sep-\width})
        -- (0,{(\coun*\sep)-\width})
        ;
        \ifthenelse{#3=1}{
            \draw (0, {\coun*\sep}) -- ({-1*\ticksize}, {\coun*\sep}) node[text width=#4, align=right, left] {\lab};
        }{}
    }
}

    \def\normsconva{
        6/{$\norm{\cdot}_{1,\infty}$ \hfill Neyshabur et al.\ {\cite{Neyshabur15a}}}/10.07,
        7/{$\norm{\cdot}_{1,\infty}$\hfill Golowich et al.\ \cite{Golowich18a}}/7.23,
        8/{$\norm{\cdot}_{1,\infty}$ \hfill Gouk et al.\ \cite{Gouk21a}}/11.54,
        9/{$\norm{\cdot}_{2}$ \hfill Neyshabur et al.\ \cite{Neyshabur15a}}/10.91,
        10/{ $\norm{\cdot}_{2}$\hfill Golowich et al.\ \cite{Golowich18a}}/10.15,
        11/{ $\norm{\cdot}_{2}$ \hfill Gouk et al.\ \cite{Gouk21a}}/20.34
        }
    \def\quantsconva{ 
        12/{Number of parameters}/6.31,
        13/{Product of Lipschitz constants}/4.17,
        14/{Product of $\norm{\cdot}_{2}$ norms}/9.16,
        15/{Product of $\norm{\cdot}_{1,\infty}$ norms}/8.81
    }
    \def\restconva{ 
        0/{\textbf{Ours}, \cref{paper:theorem:rademacher_complexity} {\eqref{theorem:main_residual_network_covering_rademacher:eq_norms}}
        }/9.31,
        1/{Bartlett et al. \cite{Bartlett17a}}/10.29,
        2/{Ledent et al.\! (\emph{main result}) \cite{Ledent21a}}/9.77,
        3/{Ledent et~al.\! (\emph{fixed constraints})}/11.57,
        4/{\textbf{Ours}, \cref{paper:theorem:rademacher_complexity} \small{\eqref{theorem:main_residual_network_covering_rademacher:eq_params}
        }}/2.86,
        5/{Lin et al.\ \cite{Lin19a}}/3.91
    }

    \def\normsconvb{
        6/{$\norm{\cdot}_{1,\infty}$ \hfill Neyshabur et al.\ \cite{Neyshabur15a}}/18.13,
        7/{$\norm{\cdot}_{1,\infty}$\hfill Golowich et al.\ \cite{Golowich18a}}/13.84,
        8/{$\norm{\cdot}_{1,\infty}$ \hfill Gouk et al.\ \cite{Gouk21a}}/19.86,
        9/{$\norm{\cdot}_{2}$ \hfill Neyshabur et al.\ \cite{Neyshabur15a}}/19.85,
        10/{ $\norm{\cdot}_{2}$\hfill Golowich et al.\ \cite{Golowich18a}}/17.68,
        11/{ $\norm{\cdot}_{2}$ \hfill Gouk et al.\ \cite{Gouk21a}}/38.28
    }
    \def\quantsconvb{ 
        12/{Number of parameters}/{6.62},
        13/{Product of Lipschitz constants}/6.91,
        14/{Product of $\norm{\cdot}_{2}$ norms}/16.59,
        15/{Product of $\norm{\cdot}_{1,\infty}$ norms}/15.37
    }
    \def\restconvb{ 
        0/{\textbf{Ours-norms}
        }/12.33,
        1/{Bartlett et al. \cite{Bartlett17a}}/13.58,
        2/{Ledent et al.\! (\emph{main result}) \cite{Ledent21a}}/12.55,
        3/{Ledent et~al.\! (\emph{fixed constraints})}/14.64,
        4/{\textbf{Ours-params}
        }/3.07,
        5/{Lin et al.\ \cite{Lin19a}}/4.78
    }   


    \begin{scope}[yshift=0.3 cm]    
        \barplot{\normsconvb}{bargrayb}{0}{3.3cm}
        \barplot{\restconvb}{bargrayb}{0}{3.5cm}
             
        \barplot{\normsconva}{bargraya}{1}{3.4cm}
        \barplot{\restconva}{bargraya}{1}{3.5cm}

        \begin{scope}[yshift=0.1cm]
            \barplot{\quantsconvb}{bargreenb}{0}{3.5cm} 
            \barplot{\quantsconva}{bargreena}{1}{3.5cm}
        \end{scope}    

    \draw [line width = 1.pt, decorate,
    decoration = {calligraphic brace, mirror}] (\bracepos*\xscale,0*\sep-\width) --  (\bracepos*\xscale,5*\sep+\width)
    node[pos=.5, right] {via covering numbers};
    \draw [line width = 1.pt, decorate,
    decoration = {calligraphic brace, mirror}] (\bracepos*\xscale,6*\sep-\width) --  (\bracepos*\xscale,11*\sep+\width)
    node[pos=.5, right] {via layer-peeling};        
    \end{scope}

    \draw (0,0.12) -- ({40*\xscale},0.12) -- ({40*\xscale},{(\nbars*\sep)+0.3}) -- (0,{(\nbars*\sep)+0.3}) -- cycle;
    \foreach \i in {0, 10, 20, 30, 40}{
        \draw[color=bargrayc] ({\i*\xscale},0.12) node[below] {\textcolor{black}{$10^{\i}$}} -- ({\i*\xscale},{(\nbars*\sep)+0.3});
    }
\end{tikzpicture}
    \end{center}
    \caption{Empirical Rademacher complexity bounds (grouped by proof strategy), for a 6- (\textcolor{black!45}{$\blacksquare$}\textcolor{unigreen!90}{$\blacksquare$}) and an 11-layer (\textcolor{black!30}{$\blacksquare$}\textcolor{unigreen!30}{$\blacksquare$}) convolutional network, trained on CIFAR10. Quantities that typically appear in these bounds are shown in \textcolor{unigreen}{\textbf{green}} (top part of figure) for reference. \label{fig:appendix:simple_comparison}}
\end{figure}
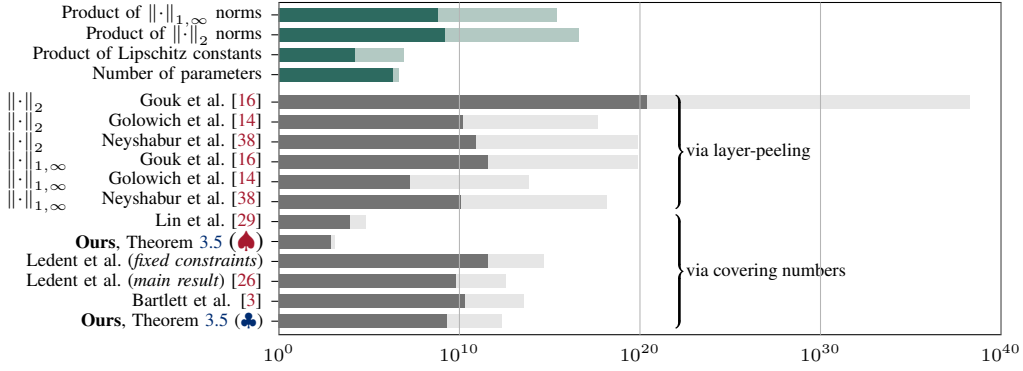

\subsection{Comparison with Bartlett et al. \texorpdfstring{\cite{Bartlett17a}}{(NeurIPS, 2017)}}
\label{appendix:comparison_bartlett}

Our Rademacher complexity bounds are based on the proof strategy of Bartlett et al. \cite{Bartlett17a}. That is, we first derive single layer covering number bounds for \emph{convolutional} layers. In a second step, we derive covering numbers for entire \emph{residual} networks. Last, a combination of Dudley's entropy integral and \cite[Lemma 3.1]{Bartlett17a} implies the Rademacher complexity bounds \eqref{theorem:main_residual_network_covering_rademacher:eq_norms} and \eqref{theorem:main_residual_network_covering_rademacher:eq_params}. 
As already discussed in \cref{subsection:single_layer}, our single-layer covering number bound for convolutional layers includes the single-layer covering number bound for fully-connected layers from \cite[Lemma 3.2]{Bartlett17a}. Consequently, in the special case of fully-connected layers and no skip connections, our main result \eqref{theorem:main_residual_network_covering_rademacher:eq_norms} reduces to the Rademacher complexity bound from \cite{Bartlett17a}.
To be more specific, we show that 
\begin{equation}
    \label{eq:bartlett_comparison_ours}
    \rade{S}{\mathcal F_\gamma}
    \le
    \frac{4}{n}
    +
    \frac{12 H_{n-1}}{\sqrt{n}}
    \sqrt{\log(2W)}
    \left(
        \sum_{i=1}^L
        \ceil{\tilde C_{i}^{\nicefrac{2}{3}}}
    \right)^{\!\nicefrac{3}{2}}
\end{equation}
and Bartlett et al. \cite{Bartlett17a} prove
\begin{equation}
    \label{eq:bartlett_comparison_theirs}
    {
        \rade{S}{\mathcal F_\gamma}
        \le
        \frac{4}{n}
        +
        \frac{9 \log(n)}{\sqrt{n}}
        \sqrt{\log(2W)}
        \left(
            \sum_{i=1}^L
            \tilde C_{i}^{\nicefrac{2}{3}}
        \right)^{\!\nicefrac{3}{2}}
    }
    \enspace.
\end{equation}
Here, $L$ denotes the depth of the network and $\tilde C_i$ is the part of the capacity of the $i$-th layer due to weight and data norms, \ie,
\begin{equation}
    \tilde C_{i}(X)=
    \frac 4 \gamma \, \frac{\norm{X}}{\sqrt{n}} 
    \left(
        \prod_{\substack{l=1}}^{L}
            s_l \rho_l
    \right) 
    \frac{b_{i}}{s_{i}}
\enspace,
\end{equation}
with $s_i$ the layers' Lipschitz constraints, $b_i$ the layers' (2,1) norm constraints and $\rho_i$ the Lipschitz constants of the nonlinearities.

A closer look reveals that there are three differences between the results:
\begin{inparaenum}[(i)]
\item a different numerical constant,
\item the logarithm is replaced with a
harmonic number, and 
\item \cref{eq:bartlett_comparison_theirs} contains no ceiling functions. 
\end{inparaenum}
Of course, the differences do not affect the asymptotic behavior of the bound and are thus only of
minor importance.
From our understanding, the differences are rooted in a lapse in the inequality chains of \cite[Eq. (A.3)]{Bartlett17a}.

\begin{itemize}
    \item The difference in the numerical constant appears, because proving the entire network covering number bound requires transitioning to \emph{external} covering numbers, which manifests as an additional factor of 2 in the final result.
    This is because the single layer covering number bounds from \cref{paper:theorem:single_layer}, resp. \cite[Lemma 3.2]{Bartlett17a}, hold for layers with only a ($2,1$) group norm constraint, which form a \emph{superset} of the layers with a ($2,1$) group norm constraint \emph{and} a Lipschitz constraint as considered in \cref{paper:theorem:rademacher_complexity}, resp \cite[Theorem 3.3]{Bartlett17a}.
    On the other hand, the parameter $\alpha$ in the proof of \cite[Lemma A.8]{Bartlett17a} can be chosen as $\alpha=1/\sqrt{n}$, which improves the bound by a factor of $\nicefrac 3 2$ (see proof of \cref{theorem:residual_network_covering_rademacher}).    
    Overall, both effects lead to a factor of $\nicefrac 4 3$, which is precisely the quotient of the numerical constants in \cref{eq:bartlett_comparison_ours} and \cref{eq:bartlett_comparison_theirs}.
    \item Our result in \cref{eq:bartlett_comparison_ours} contains ceiling functions and a harmonic number, which is a direct consequence of the ceiling function appearing in the single layer covering number bound of \cref{paper:theorem:single_layer}, resp. \cite[Lemma 3.2]{Bartlett17a}.
    In the chain of inequalities \cite[Eq. (A.3)]{Bartlett17a} in the proof of \cite[Theorem 3.3]{Bartlett17a}, the single layer bound is inserted \emph{without} the ceiling function.
\end{itemize}

\subsection{Comparison with Long \& Sedghi \texorpdfstring{\protect\cite{Long20a}}{(ICLR, 2020)}}
\label{appendix:comparison_long}
Long \& Sedghi \cite{Long20a} study generalization bounds for the class $\mathcal{F}$ of convolutional networks that realize functions of the form 
$$f = \sigma_L \circ \phi_{K_L} \circ \dots \circ \sigma_1 \circ  \phi_{K_1}: \mathcal X \to \mathbb R$$ 
with Lipschitz/spectral-norm constraints, \ie, they assume that the initializations $\phi_{K_i^{(0)}}$ per layer are $(1+\nu)$-Lipschitz and that the distances 
$$\beta_i=\Lip(\phi_{K_i} - \phi_{K_i^{(0)}})$$ 
to initialization satisfy $\sum_i {\beta_i}\le \beta$ for some given constant $\beta>0$.
They show \cite[Theorem 3.1]{Long20a} that for $\lambda$-Lipschitz loss functions $\ell$, the generalization gap is (with probability $1-\delta$) uniformly bounded over the class $\mathcal F_\ell = \set{(x,y)\mapsto \ell(f(x),y)~|~ f\in \mathcal F}$ by
\begin{equation}
    C M \sqrt{
        \frac{
            \bar W(\beta + \nu L + \log\left(\lambda \beta \chi\right)) + \log\left(\frac 1 \delta\right)
        }{n}
    }
    \enspace,
    \label{eq:sedghi_original}
\end{equation}
assuming that $\lambda \beta \chi(1+\nu +\beta/L)^L\ge 5$ and $n$ large enough. Here, $C$ denotes an unspecified constant and $M$ the maximum of the loss function $\ell$. Further, $\bar W = \sum_i W_i$ is the total number of network parameters and $\norm{x}_2\le \chi$ is an upper bound on the Euclidean norm of the data.
As can be seen, this bound depends on the square root of parameters and the distance $\beta$ to initialization. 
In contrast to other results (e.g., \cite[\eqref{theorem:main_residual_network_covering_rademacher:eq_norms}, \eqref{theorem:main_residual_network_covering_rademacher:eq_params}]{Bartlett17a,Ledent21a}), it also depends on a Lipschitz constraint $(1+\nu)$ \emph{directly} on the initialization.

\cref{eq:sedghi_original} is based on \cite[Lemma 2.3]{Long20a}, which requires the class $\mathcal F_\ell$ to be $(B,d)$-\emph{Lipschitz parametrized}, \ie, that there exists $d\in \mathbb N$ and a norm $\norm{\cdot}$ on $\mathbb R^d$, together with a $B$-Lipschitz continuous and surjective map $\phi: \mathcal B(1,\norm{\cdot}) \to \mathcal F_\ell$ from the $\norm{\cdot} $-unit ball in $\mathbb R^d$ onto $\mathcal F_\ell$, which is $B$-Lipschitz. The latter means that for every $\theta, \theta' \in \mathcal B(1,\norm{\cdot})$, it holds that $\|\phi(\theta) - \phi(\theta')\|_\infty \le B \norm{\theta-\theta'}$. 
In this situation, the generalization gap is bounded by 
\begin{equation}
    C M \sqrt{\frac{d \log B + \log\left(\frac 1 \delta\right)}{n} } \enspace.
    \label{eq:sedghi_lemma2.3}
\end{equation}

In a series of lemmas \cite[Lemma 3.2--3.4]{Long20a}, the authors show that $\mathcal F_\ell$ is indeed $(B,d)$-Lipschitz parametrized with $d= \bar W$ and $B = \lambda \chi \beta (1+\nu+\beta/L)^L $. We will repeat the argument and show that it implies an intermediate result which scales similarly to our result \eqref{theorem:main_residual_network_covering_rademacher:eq_params} from \cref{paper:theorem:rademacher_complexity}, \ie, with the square root of 
$${\bar W \log\left(\prod_j s_j\right)}\enspace,$$
where $s_i$ denote Lipschitz constraints on the layers $\phi_{K_i}$.

Let $\mathbf{K} = (K_1, \dots, K_L)$ and $\mathbf{\tilde K} = (\tilde K_1, \dots, \tilde K_L)$ be tuples of weight tensors and denote the corresponding networks by $f_\mathbf{K}$, resp. $f_\mathbf{\tilde K}$. If $\mathbf{K}$ and $\mathbf{\tilde K}$ differ in only one layer, say $K_j \neq \tilde K_j$, then for all $x \in \mathcal X$ (see proof of \cite[Lemma 3.2]{Long20a}), 
\begin{equation}
    |f_{\mathbf K}(x) - f_{\mathbf{\tilde K}} (x)| 
    \le \norm{x}_2 \left( \prod_{i\neq j} \Lip(\phi_{K_i})\right) \Lip(\phi_{K_j}{-}\phi_{\tilde K_j})
    \le 
    \chi  \left( \prod_{i=1}^L s_i \right) \Lip(\phi_{K_j}{-}\phi_{\tilde K_j})
    \enspace.
    \label{eq:sedgi_intermediate}
\end{equation}
Consequently, if $\mathbf{K}$ and $\mathbf{\tilde K}$ differ in \emph{all} layers, it holds that
\begin{equation}
    |f_{\mathbf K}(x) - f_{\mathbf{\tilde K}} (x)|
    \le 
    \chi \left( \prod_{i=1}^L s_i \right) \sum_{j=1}^L \Lip(\phi_{K_j} - \phi_{\tilde K_j})
    \enspace.
\end{equation}
As $\sum_{j=1}^L \Lip(\phi_{K_j})$ defines a norm $\norm{\cdot}$ on $\mathbb R^{\bar W}$ (in \cite{Long20a} this norm is denoted $\norm{\cdot}_\sigma$), the inequality above implies that the surjective map
\begin{align*}
    \mathcal B(1,\norm{\cdot})\to \mathcal F
    \enspace, 
    \quad
    \frac{\mathbf{K}}{\sum_i s_i} \mapsto f_{\mathbf K}
\end{align*}
is $\left(\chi  \left( \prod_{i=1}^L s_i \right)\sum_i s_i\right)$-Lipschitz, \ie, the class $\mathcal F_\ell = \set{(x,y)\mapsto \ell(f(x),y)|~ f\in \mathcal F}$ is 
$\left(\lambda \chi  \left( \prod_{i=1}^L s_i \right) \sum_i s_i, \bar W\right)$-Lipschitz parametrized. Thus, \cref{eq:sedghi_lemma2.3} implies a generalization bound of the form 
\begin{equation}
    C M \sqrt{\frac{\bar W \log\left(\lambda
        \chi \left( \prod_{i=1}^L s_i \right) \sum_i s_i
    \right)
    + \log\left(\frac 1 \delta \right)}{n}}
    \enspace.
    \label{eq:sedghi_asours}
\end{equation}
Similarly to our result \eqref{theorem:main_residual_network_covering_rademacher:eq_params} from \cref{paper:theorem:rademacher_complexity}, this bound scales with the square root of the number of parameters and with the logarithm of the product of Lipschitz constants.
However, as \cref{eq:sedghi_asours} and \cref{eq:sedghi_original} are proven via an asymptotic bound from Gin\'{e} and Guillou \cite{Gine01}, the constant $C$ and the minimal sample size $n$ required for \cref{eq:sedghi_asours} and \cref{eq:sedghi_original} to hold are not readily available. This makes further comparisons difficult.

\cref{eq:sedghi_asours} differs from the main result in \cite{Long20a}, \ie, \cref{eq:sedghi_original}, as, instead of constraints on the layers' Lipschitz constants $\Lip(\phi_{K_i})\le s_i$, Long \& Sedghi consider constraints on the Lipschitz constants of the initialization $\Lip(\phi_{K_i^{(0)}})\le 1+ \nu$ and on the distance to initialization $\Lip(\phi_{K_i} - \phi_{K_i^{(0)}})\le \beta_i$ with $\sum_i \beta_i = \beta$. Starting from \cref{eq:sedgi_intermediate}, these constraints enter via the triangle inequality, \ie,  
\begin{equation*}\Lip(\phi_{K_i}) \le \Lip(\phi_{K_i^{(0)}}) +  \Lip(\phi_{K_i} - \phi_{K_i^{(0)}}) \le 1 +\nu +\beta_i \enspace.\end{equation*}
Maximizing $\prod_{i} (1+ \nu + \beta_i)$ subject to $\sum_i \beta_i = \beta$, yields 
\begin{align}
    |f_{\mathbf K}(x) - f_{\mathbf{\tilde K}} (x)|
    &\le 
    \chi\left(1 + \nu +\beta/L\right)^L 
    \Lip(\phi_{K_j} - \phi_{\tilde K_j})
    \\ &\le
    \chi \exp\left(\nu L + \beta\right)
    \Lip(\phi_{K_j} - \phi_{\tilde K_j})
    \enspace.
\end{align}
\cite[Lemma 3.3 \& 3.4]{Long20a} then imply that $\mathcal F_\ell$ is $(B,d)$-Lipschitz parametrized with $d = \bar W$ and 
$B = \lambda \chi \beta \exp\left(\nu L + \beta\right)$, which in turn implies \cref{eq:sedghi_original}.

Obviously, every bound that depends on weight norms can be transferred to a bound that depends on the norm of the initialization and the distance to it, simply by application of the triangle inequality. We argue, that utilizing the translation invariance of covering numbers, as done in, e.g., \eqref{theorem:main_residual_network_covering_rademacher:eq_norms}, \eqref{theorem:main_residual_network_covering_rademacher:eq_params}, as well as in \cite{Bartlett17a, Ledent21a}, is a more natural way of incorporating the distance to initialization, as it allows for bounds which do not depend on norm constraints on the initialization.

\subsection{Comparison with Ledent et al.\ \texorpdfstring{\protect\cite{Ledent21a}}{(AAAI, 2021)}}
\label{appendix:comparison_ledent}
In \cite{Ledent21a}, Ledent et al. derive generalization/Rademacher complexity bounds via $l_\infty$ coverings of convolutional networks.
These bounds incorporate weight sharing and thus directly depend on the norms of the weight tensors, instead of depending on the norms of the matrix that parametrizes the linear (convolutional) map. This results in an improved scaling with respect to the spatial input width.

In general, the bounds in \cite{Ledent21a} scale similarly to our bound (\ref{theorem:main_residual_network_covering_rademacher:eq_norms}) from \cref{paper:theorem:rademacher_complexity} in that they depend on the square root of the product of Lipschitz constants (or empirical estimates thereof). In particular, just as our result (\ref{theorem:main_residual_network_covering_rademacher:eq_norms}), \cite[Theorem 16]{Ledent21a} is based on Rademacher complexity bounds for function classes $\mathcal F_\gamma$, \ie, the composition of Lipschitz- and distance-constrained convolutional networks with the ramp loss at margin $\gamma>0$.
The main result \cite[Theorem 3]{Ledent21a}, as well as \cite[Theorem 20]{Ledent21a}, adapts techniques from \cite{Wei19a} and \cite{Nagarajan19b} to replace the product of Lipschitz constants with empirical equivalents, which  are typically much smaller. To this end, they study the composition of convolutional networks with an \emph{augmented} loss function, see for example \cite[Eq. (26)]{Ledent21a}.

In this part of the appendix, we compare our norm-driven bound (\ref{theorem:main_residual_network_covering_rademacher:eq_norms}) with the main results in \cite{Ledent21a}. As mentioned in \cref{subsection:single_layer}, we find that both results exhibit similar scaling behavior, but \emph{we improve in the logarithmic term and in that our dependency on data norms is less sensitive to outliers}. On the other hand, the main bounds in \cite{Ledent21a} exhibit an improved dependency on the number of classes. The latter pays off, e.g., for shallow networks or in extreme multiclass problems with a large number of classes. All three effects are due to the use of $l_2$ \emph{vs.} $l_\infty$ covering numbers.

Central to all Rademacher complexity bounds \cite{Ledent21a} is the single-layer $l_\infty$ covering number bound restated in the proposition below.

\begin{proposition}[{\cite[Proposition 6]{Ledent21a}}]
    Let positive reals $(a,b,\epsilon)$ and positive integer $m$ be given. Let the tensor $X \in \mathbb R^{n \times U \times d'}$ be given with 
    $\forall i \in \set{1,\dots,n}$, $\forall u \in \set{1,\dots,U}$, $\norm{X_{iu\cdot}}_2 \le b$. For any fixed $M$:
    \begin{equation}
        \log \mathcal N\left(\set{\!XA\colon A \in \mathbb R^{d'\times m}, \norm{A-M}_{2,1}\le a}\!, \epsilon, \norm{\cdot}_*\!\right)
        \le 
        \frac{64 a^2 b^2}{\epsilon^2} 
        \log_2\!\left[ 
            \left(\frac{8ab}{\epsilon}+7 \right) mnU
        \right]
        \label{eq:ledent_singlelayer}
    \end{equation}
    with the norm $\norm{\cdot}_*$ over the space $\mathbb R^{n\times U \times m}$ defined by $\norm{Y}_* = \max_{i\le n} \max_{j\le U} (\sum_{k=1}^m Y_{ijk}^2)^{1/2}$.
\end{proposition}
Some remarks regarding the notation. Here, $X$ does \emph{not} denote the input data $(x_1,\dots, x_n)$, but the $nU$-tuple of all $d'$-sized convolutional patches of the input data. Thus, $d'=k^2 c_{in}$ is the square of the kernel size times the number of input channels and $U = \ceil{d/t}^2$ is the number of patches per image, which is computed as the square of the spatial width divided by the stride. The matrix $A\in \mathbb R^{k^2 c_\mathit{in} \times c_\mathit{out}}$ then is the local linear map acting on the convolutional patches, \ie, $A$ is a reshaping of the weight tensor $K$ and $XA$ is the output of the convolutional layer, \ie, $n$ images with $c_\mathit{out}$ channels with $U$ pixels each. Further, $\norm{A}_{2,1}$ is the \emph{standard} matrix ($2,1$) group norm which differs from $\norm{K}_{2,1}$ defined in \cref{eqn:def_21_norm}. 

As the single-layer bound in \cref{eq:ledent_singlelayer} and our single-layer bound in \cref{theorem:single_layer:eq_norms} are the fundamental building blocks of all inferred results (and we did not study augmented loss functions), we will focus on them for the comparison. For ease of reference, we restate the relevant part of \cref{paper:theorem:single_layer}.

\vskip2ex
\begin{theorem}
    \label{led_comp:theorem:single_layer}
    Let $b > 0$ and $\mathcal{F} = \{\phi_K|~ K \in \mathbb R^{c_\textit{out}\times c_\textit{in} \times k \times k}, \norm{K}_{2,1}\le b\}$ denote the class of 2D convolutions with $c_\textit{in}$ input channels, $c_\textit{out}$ output channels and kernel size $k{\times}k$, parametrized by tensors $K$ with $W$ parameters. Then, 
    \begin{equation}
        \log\mathcal N(\mathcal F, \epsilon, \norm{\cdot}_X ) \le
        \ceil{\frac{\norm{X}^2 b^2}{\epsilon^2}} \log(2W)  \label{led_comp:theorem:single_layer_simple:eq_norms}
       \enspace.
     \end{equation}
\end{theorem}

There is a clear similarity between \cref{eq:ledent_singlelayer} and \cref{led_comp:theorem:single_layer_simple:eq_norms}. Both depend quadratically on weight and data norms divided by the covering radius $\epsilon$, as well as on a logarithmic term. 
Consequently, differences between both bounds are nuanced and, ignoring the constant in \cref{eq:ledent_singlelayer}, it is a priori not clear which bound is preferable. We will discuss these nuances theoretically and provide a empirical comparison in \cref{fig:comparison_ledent_norms}.

\begin{enumerate}[label=(Diff-\arabic*), leftmargin=*]
    \item \label{enum:datanorms}\textbf{Data norms.} Our work assumes a bound on the $l_2$ norm of the \emph{whole input} $x$ (a $c_\textit{in}d^2$-tuple), whereas \cite{Ledent21a} only assumes a bound on the $l_2$ norm of every \emph{single patch} $p$ ($c_\textit{in}k^2$-tuples).
    This potentially improves \cref{eq:ledent_singlelayer} over
    \cref{led_comp:theorem:single_layer_simple:eq_norms} by a factor of $(k/d)^2$, as
    \begin{equation}
        \max_{p \in \mathrm{patches}} \norm{p}
        \le \norm{x}    
        \lesssim d/k \max_{p \in \mathrm{patches}} \norm{p}
        \enspace.
    \end{equation}
    The left inequality is obvious. The right inequality follows from considering the sum of all patch norms. 
    As every pixel $x_{ijk}$ appears in at least $\floor{k/s}^2$ patches and there are at most $\ceil{d/s}^2$ patches in total, it holds that
    \begin{align*}
        \floor{k/s}^2 \norm{x}^2
        &=
        \sum_{i=1}^c \sum_{j,l=1}^d |x_{ijl}|^2 \floor{k/s} \\
        & \le 
        \sum_{i=1}^c \sum_{j,l=1}^d |x_{ijl}|^2  \operatorname{card}(\set{p \in \mathrm{patches}~|~ x_{ijl} \in p}) \\
        & =
        \sum_{p \in \mathrm{patches}} \norm{p}^2
        \le 
        \ceil{d/s}^2 \max_{p \in \mathrm{patches}} \norm{p}^2 
        \enspace,
    \end{align*}
    and $\nicefrac{\floor{{k/s}}}{\ceil{d/s}} \approx k/d$.
    Notably, the maximum in \cref{eq:ledent_singlelayer} is over the patches on \emph{all} of the input data, which is quite \emph{sensitive to outliers}. Hence, the improvement over \cref{led_comp:theorem:single_layer_simple:eq_norms} is typically smaller than $d/k$, especially at hidden layers, see top row of \cref{fig:comparison_ledent_norms}.
    \item \label{enum:weightnorms} \textbf{Weight norms.} The ($2,1$) group norms on the weights are applied differently, \ie, we compute a $(2,1)$ norm via \cref{eqn:def_21_norm}, whereas \cite{Ledent21a} computes the ($2,1$) group norm of the matrix corresponding to the local linear map, which is applied to each patch, \ie, $\sum_i \norm{K_{i\cdot\cdot\cdot}}_2$. As 
    \begin{equation}
        \sum_i \norm{K_{i\cdot\cdot\cdot}}_2 \le 
        \norm{K}_{2,1} \le k \sum_i \norm{K_{i\cdot\cdot\cdot}}_2
        \enspace,
    \end{equation}
    this potentially improves \cref{eq:ledent_singlelayer} over
    \cref{led_comp:theorem:single_layer_simple:eq_norms} by a factor of $1/k^2$. Empirically, we observe that $\norm{K}_{2,1} \approx k \sum_i \norm{K_{i\cdot\cdot\cdot}}_2$, see bottom row of \cref{fig:comparison_ledent_norms}.
\end{enumerate}
Thus, considering norm constraints only, \ie, \ref{enum:datanorms} and \ref{enum:weightnorms}, we find that \cref{eq:ledent_singlelayer} is potentially better by a factor $(k/d)^2 \cdot(1/k)^2 = 1/d^2$, \ie, the reciprocal of the squared height/width of the input images.
However, the comparison is more intricate, as the coverings are with respect to different ($l_2$ vs.\ $l_\infty$) norms and, more importantly, the considered function classes differ. As, ultimately, we want to get Rademacher complexity bounds for whole networks, we need to consider effects that appear when transitioning to whole-network bounds.

\begin{enumerate}[label=(Diff-\arabic*), start=3, leftmargin=*]
    \item \label{enum:lipschitz} \textbf{Lipschitz constants.} In whole-network bounds, contributions of all layers are summed. These contributions are the (logarithmic) single-layer bounds, scaled by a factor corresponding to the Lipschitz constant of the remainder of a network after the layer. Typically,
    the Lipschitz constant of the part before a layer additionally enters as an estimate of the norm of the layer's input. Notably, in \cite{Ledent21a}, the Lipschitz constant of the network's remaining layers incurs an additional factor $\ceil{d/t}^2$, i.e., the spatial dimensionality of the output (denoted by $w_l$ in the reference).
    This counterbalances the improvements by \ref{enum:datanorms} and \ref{enum:weightnorms}.
    
    Specifically, in \cite{Ledent21a}, the Lipschitz constants are with respect to the norms $\norm{\cdot}_{\infty, r}$ on the domain and $|\cdot|_{s}$ on the codomain, see, \eg, the definition of \smash{$\rho^{\mathcal A}_{l_1 \to l_2}$} in the statement of \cite[Proposition 10]{Ledent21a}. There, 
    \begin{equation*}\norm{x}_{\infty, r} = \max_{j\le d} \max_{k\le d}\sqrt{ \sum_{i=1}^c x_{ijk}^2}\end{equation*}
    is the maximum $l_2$ norm of a slice of the image $x$ along the channels, \ie, at fixed spatial position. The norm $|y|_{s} = \max_{p \in \textrm{patches}} \norm{p}$ is the maximal norm of a convolutional patch on $y$. Transitioning to Lipschitz constants with respect to $l_2$ norms, \ie, spectral norms, as done for the main result in \cite[Theorem 3]{Ledent21a}, incurs an additional factor $d$ (the spatial dimension of $x$), since 
    \begin{equation*}
        \frac{|f(x)|_{s}}{\norm{x}_{\infty, r}}
        =
        \underbrace{\frac{|f(x)|_{s}}{\norm{f(x)}_2}}_{\le 1}\,
        \frac{\norm{f(x)}_2}{\norm{x}_2}\,
        \underbrace{\frac{\norm{x}_2}{\norm{x}_{\infty, r}}}_{\le d}
        \le 
        d \frac{\norm{f(x)}_2}{\norm{x}_2} \enspace.
    \end{equation*}
    In this inequality, $x$ denotes the output of the considered layer and so $d$ is its spatial width. Notably, in \cite{Ledent21a}, $d$ can actually be reduced to the output's spatial width after a subsequent max-pooling operation.
    In our whole-network bound, the Lipschitz constant is already with respect to $l_2$ norms and thus no additional factors appear.

    \item \label{item:class_dependency}\textbf{Dependency on number of classes.} The use of $l_\infty$ covering numbers in \cite{Ledent21a} improves the dependency on the number of classes for whole-network bounds. This is because the weights of the classification layer do not enter via a ($2,1$) group norm constraint, but a Frobenius norm constraint. This implicitly improves the log covering number of this layer by a factor of the number of classes. Since, for whole-network bounds, the contribution of all layers are summed, we expect this effect to be significant if the contribution of the classification is a substantial fraction of the whole-network bound. This would be the case, \eg, for shallow networks or in extreme multiclass settings.
\end{enumerate}
Finally, we discuss the logarithmic terms and constants.
\begin{enumerate}[label=(Diff-\arabic*), start=5, leftmargin=*]
    \item \textbf{Logarithmic terms.} Our bound in \cref{led_comp:theorem:single_layer_simple:eq_norms} depends logarithmically on the number of parameters, denoted by~$W$. By contrast, \cref{eq:ledent_singlelayer} depends on $\log_2\left[ 
        \left(\frac{8ab}{\epsilon}+7 \right) mnU
    \right]$. When transitioning to Rademacher complexity bounds via Dudley's entropy integral (cf. \cite[Eq. (29)]{Ledent21a}), the covering radius $\epsilon$ in the $\log_2$ term is replaced by $\frac 1 n$. So, considering the definitions of $U$ and $m$, we need to compare $8abmUn^2 = 8 ab c_\textit{out} d^2 n^2 / t^2$ (Ledent et al.) with $2W = 2c_\textit{in} c_\textit{out} k^2$ (Ours).
    As, typically, $k\le d/t$ and $c_{\textit{in}}<n^2$, we improve over \cite{Ledent21a} in the logarithmic term (recall that $a$ and $b$ denote the weight and data norm constraints, respectively).
    \clearpage
    \item \textbf{Multiplicative constants.}
    The single-layer bound by Ledent et al.\ \cref{eq:ledent_singlelayer} has a rather large multiplicative constant $64$ (compared to $1$ in \cref{led_comp:theorem:single_layer_simple:eq_norms}). 
    This constant enters mainly via a previous theorem by Zhang \cite[Theorem 4]{Zhang02a}. Notably, a remark in \cite{Zhang02a} highlights that the constants in this theorem are not optimized. Thus, improving \cref{eq:ledent_singlelayer} in this regard might be possible, and the difference in numerical constants might be less pronounced than it appears at first sight. Yet, from our understanding, some constants are unavoidable, \eg, the factor $2$ which enters the proof sketch of \cite[Proposition 6]{Ledent21a}.
\end{enumerate}
Overall, \ref{enum:datanorms} -- \ref{enum:lipschitz} lead to several effects, which can potentially compensate each other, especially if, for each layer, the coordinates of its input and of the weights have roughly equal norm. Notably, in this situation, our single-layer bound can be improved by a factor of $1/t^2$, with $t$ denoting the stride of the convolution, see \cref{rem:strides}. Thus, in the absence of pooling (\eg, when downsampling is handled directly by the stride of the convolutional layer), the scaling is precisely the same.
As \cref{eq:ledent_singlelayer} depends on the maximum norm of a patch over \emph{all of the input data} (\ie, a quantity which is sensitive to outliers), we do not expect \ref{enum:datanorms} -- \ref{enum:lipschitz} to fully compensate each other, but rather expect an advantage of our bound from \cref{led_comp:theorem:single_layer}.
A detailed investigation is shown in \cref{fig:comparison_ledent_norms}, which highlights weight and data norms for layers of a trained network (an 11-layer convolutional network as used for \cref{fig:appendix:simple_comparison}).
To incorporate the effects \ref{enum:datanorms} -- \ref{enum:lipschitz} and to allow for a cleaner comparison, we multiply the $l_2$ norms of the patches by the square root of spatial dimensionality of the output and shift a factor of the kernel size from our weight norms to our data norms.
We see that the (rescaled) weight norms across all layers are essentially the same, whereas, due to the maximum being sensitive to outliers, our data norms are substantially smaller at the hidden layers.

\begin{wrapfigure}{r}{0.46\textwidth}
    \vspace{-0.42cm}
    \includegraphics[width=0.46\textwidth]{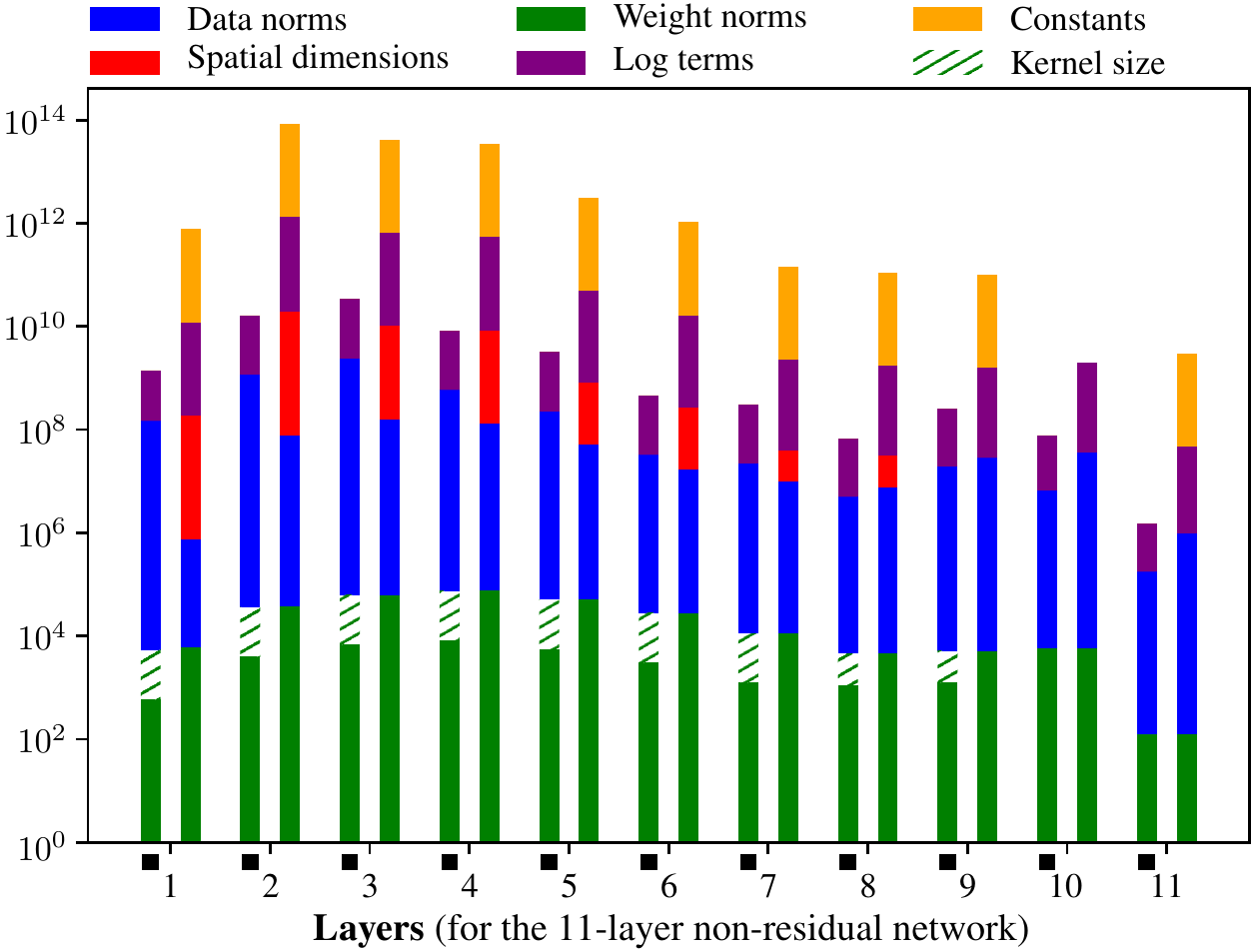}
    \vspace{-0.8cm}
\end{wrapfigure}
As a last comparison, we illustrate the magnitudes of all factors appearing in the bounds of Eqs.~\eqref{eq:ledent_singlelayer} and \eqref{led_comp:theorem:single_layer_simple:eq_norms} and of the spatial dimensionality of the output, see \ref{enum:lipschitz} in the enumeration above. In both bounds, we discard the denominator $\epsilon^2$ and, in \cref{eq:ledent_singlelayer}, we replace the factor $1/\epsilon$ in the logarithm by $n$, just as it enters the Rademacher complexity bounds. 

As can be seen from the figure to the right (with our single-layer bound marked by {\scriptsize $\blacksquare$}), our improvement in the quadratic terms is due to data norms. To be specific, one needs to compare the data norms in our case, to the combination of data norms and spatial dimensions in the bound of Ledent et al. \cite{Ledent21a}. We also improve in the logarithmic terms and constants.

\begin{figure*}
    \begin{center}
    \includegraphics[width=1.\textwidth]{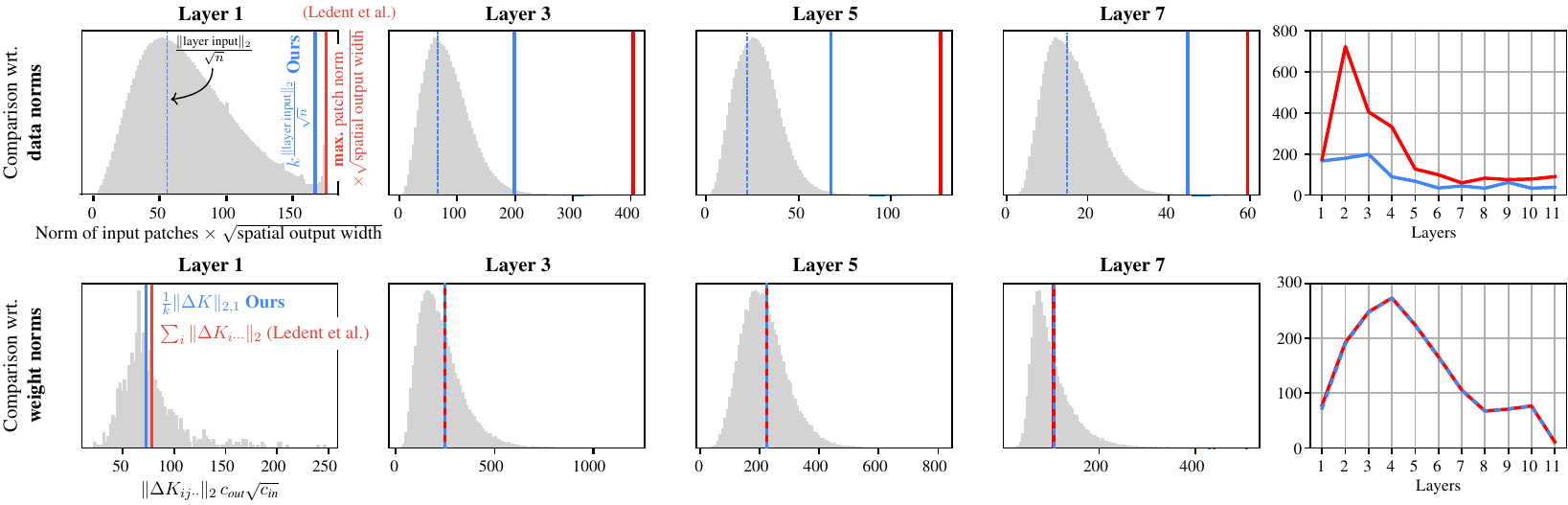}
    \end{center}
    \caption{Comparison of factors (\wrt \emph{data norms} and \emph{weight norms}) in the single-layer covering number bound of Ledent et al. \cite[Proposition 6]{Ledent21a} and our result from \cref{led_comp:theorem:single_layer_simple:eq_norms}. 
    Shown are detailed results for four exemplary layers (from the 11-layer convolutional network described in \cref{appendix:subsection:numerical_comparison_details}), as well as a summary plot across all layers (rightmost).
    The first row presents histograms of patch norms; the second row presents histograms of norms of $k\times k$ slices of tensors $\Delta K$, i.e. the difference $\Delta K$ between a weight tensor and its initialization.
    \label{fig:comparison_ledent_norms}}
\end{figure*}


\subsubsection*{Comparison for two-layer networks}
As discussed in \ref{item:class_dependency}, the $l_\infty$ covering approach in \cite{Ledent21a} allows for a favorable treatment of the last (classification) layer. If the contribution of this last layer to the respective Rademacher complexity bounds is substantial, then the bound in \cite{Ledent21a} is superior. We evaluate this effect on two-layer networks, where it is most pronounced.

The comparison considers networks of the following architecture. The first layer is convolutional and parametrized by a tensor $K\in \mathbb R^{c\times 3 \times k \times k}$. Here $c$ denotes the number of filters (channels of the output) and $k$ the kernel size, which is chosen equal to the stride and the spatial dimensionality of the input, \eg, 32 for images from CIFAR100. Consequently, the spatial dimensionality of the output is 1. This convolutional layer is followed by an activation function with Lipschitz constant 1 (\eg, ReLU) and a linear map $W\in \mathbb R^{\kappa\times c}$, with $\kappa$ denoting the number of classes.

Since the quantities and norms appearing in the respective bounds differ, we make the following simplifications, which are motivated by corresponding inequalities and verified empirically.

\begin{enumerate}[label=(\arabic*)]
    \item $\norm{K}_{2,1} \approx k \sum_{i=1}^c \norm{K_{i\cdot\cdot\cdot}}_2$\enspace,
    \item $\norm{W^\top}_{2,1} \approx \sqrt{c} \norm{W}_2$\enspace,
    \item $\max_{j\le \kappa} \norm{W_{j\cdot}}_2 \approx \Lip(W)\enspace,$
    \item $\frac{\norm{X}_2}{\sqrt n} \approx\max_{i\le n} \norm{x_i}_2 = \text{maximal norm of convolutional patches from the data}$ \enspace,
    \item $H_{n-1} \approx \log n$ \enspace.
\end{enumerate}

Furthermore, just as the single-layer bound in \cite{Ledent21a} depends on the maximal norm of a patch of the data, ours actually depends only on the maximal norm of particular slices of the data, which we here denote as $|X|_s$. In the special case of the stride being equal to the spatial dimensionality, these slices are over the channels at fixed local position (see \cref{rem:strides} and the last chain of inequalities in the proof of \cref{theorem:single_layer}). Thus, for the contribution of the first layer, we can use 

\begin{enumerate}[label=(\arabic*), start=6]
    \item $|X|_s \approx \frac{\norm{X}}{k}$ \enspace.
\end{enumerate}

Last, we empirically evaluate a maximum operator which appears in the quantity $\mathcal R$ in the two layer bound (Theorem 2) in \cite{Ledent21a}, i.e.,
\begin{enumerate}[label=(\arabic*), start=7]
    \item $\frac{1}{\max_{i\le n} \norm{\phi_K(x_i)}_2} \le \frac{\max_{j\le \kappa}\norm{W_{j\cdot}}_2}{\gamma}$ \enspace.
\end{enumerate}

With these simplifications, \textbf{our} bound (\ref{theorem:main_residual_network_covering_rademacher:eq_norms}) becomes
\begin{equation}
    \frac{48}{\gamma} \frac{\log n}{\sqrt n} \sqrt{\log(6 ck^2)}
    \left[
        \left(
            \frac{\Lip(W) \norm{K{-}K^{(0)}}_{2,1}}{k}
        \right)^{\frac{2}{3}}
        \hspace{-1ex}{+}
        \left(
            \Lip(\phi_K) \norm{(W{-}W^{(0)})^\top}_{2,1}
        \right)^{\frac{2}{3}}
    \right]^{\frac{3}{2}}
    \label{eq:twolayer:ours}
\end{equation}
and the bound from Theorem 2 in \cite{Ledent21a} becomes 
\begin{align}
    &\frac{768}{\gamma} \frac{\log n}{\sqrt n} \sqrt{n^2 \mathcal D}
    \left[
        \left(
            \frac{\Lip(W) \norm{K{-}K^{(0)}}_{2,1}}{k}
        \right)^{\frac{2}{3}}
        \hspace{-2ex}+
        \left(
            \frac
                {\Lip(\phi_K) \norm{(W{-}W^{(0)})^\top}_{2,1}}
                {\sqrt{\kappa}}
        \right)^{\frac{2}{3}}
    \right]^{\frac{3}{2}}
    \label{eq:twolayer:ledent}
    \\
    \nonumber
    &\text{with } 
    \mathcal D = 
    \max\left(
        \frac{\norm{K-K^{(0)}}_{2,1}}{k} \frac{\Lip(W)}{\Lip(\phi_K)}c\enspace,
        \frac{\norm{X}_2}{\sqrt n} \norm{W-W^{(0)}}_2 \frac{\Lip(K)}{\gamma} \kappa
    \right) \enspace.
\end{align}
As expected, ignoring constants and log terms, the bound from \cite{Ledent21a} is better by a factor of $\sqrt{\kappa}$ (square root of number of classes) in the summand corresponding to the last layer.

\textbf{Empirically}, we evaluate the bounds for networks of varying width $c \in \set{32, 1024, 8192}$ trained on CIFAR100. Here, we compute exact values and do not use the simplifications (1) - (7). As the models do not fit the training data, we use a margin parameter of $\gamma = 1$ for simplicity. Overall, the models performed rather poorly (as expected) with testing accuracies of 21\%, 29\%, 31\% and training accuracies 32\%, 99.8\%, 99.9\%.

\cref{table:twolayer_standard} lists the computed values of the bounds and how they distribute over the respective factors (weight \& data norms, logarithmic term, numerical constant, sample size dependency). Results are presented on a logarithmic scale with base 10.  Overall, we observe the following effects:

\begin{enumerate}
    \item Relatively, the contribution of the second layer in \cite{Ledent21a} is improved by a factor of 10. This is expected, as 10 is the square root of the number of classes in the CIFAR100 dataset.
    \item The wider the network, the more dominant the term corresponding to the first layer becomes. At width 32, the factor from weight and data norms of \cite{Ledent21a} is clearly superior. This is due to the improved class dependency. However, for a width of 1024, this effect is already negligible.
    \item For the wider networks ($c\in \set{1024, 8192}$), we improve over \cite{Ledent21a} by a factor of approximately $10^{1.5}\approx 30$. Ignoring numerical constants, we improve by a factor $10^{0.3} \approx 2$ , which is due to an improvement in the logarithmic term.
    \item For fixed width, the bounds and factors do not vary over the random initializations. For the models with widths 1024 and 8192, the standard deviation of the base 10 logarithms are $<0.005$, which corresponds to a geometric standard deviation of less than a factor $1.01$, \ie, 1\%.
\end{enumerate}

Last, we consider a network whose first layer has kernel size 3 and stride 1. Reduction of the spatial dimensionality is achieved by a subsequent max pooling layer of window size $32\times 32$. This setting favors \cite{Ledent21a} as this work can better account for the pooling layer. In our simplified bound of \cref{eq:twolayer:ours}, the factor $1/k$ in the first summand disappears because of the unit stride; their result improves due to the now smaller convolutional patches.
In this setting the contribution of the first layer's weight and data norms is clearly larger in our bound. Yet, our bound is still smaller, but only due to numerical constants.

\afterpage{%
    \clearpage
    \thispagestyle{empty}
    \begin{landscape}
        \centering 
    \savebox\Tbox{
    \begin{tabular}{l|cc|cc|cc}
        & \multicolumn{2}{c|}{Width  32} & \multicolumn{2}{c|}{Width  1024} & \multicolumn{2}{c}{Width  8192} \\
        & \hyperref[eq:twolayer:ours]{\textcolor{black}{\textbf{Ours}}}    & \hyperref[eq:twolayer:ledent]{\textcolor{black}{\textbf{Ledent et al.}}} &  \hyperref[eq:twolayer:ours]{\textcolor{black}{\textbf{Ours}}}  &  \hyperref[eq:twolayer:ledent]{\textcolor{black}{\textbf{Ledent et al.}}} & \hyperref[eq:twolayer:ours]{\textcolor{black}{\textbf{Ours}}}  & \hyperref[eq:twolayer:ledent]{\textcolor{black}{\textbf{Ledent et al.}}}    \\ \hline
        Bound& $\phantom{-} 5.669 \pm 0.009$& $\phantom{-} 6.704 \pm 0.009$& $\phantom{-} 6.538 \pm 0.001$& $\phantom{-} 7.973 \pm 0.001$& $\phantom{-} 7.111 \pm 0.001$& $\phantom{-} 8.657 \pm 0.001$\\ 
        Weight \ data norms& $\phantom{-} 4.760 \pm 0.009$& $\phantom{-} 4.294 \pm 0.009$& $\phantom{-} 5.575 \pm 0.001$& $\phantom{-} 5.544 \pm 0.001$& $\phantom{-} 6.121 \pm 0.001$& $\phantom{-} 6.208 \pm 0.001$\\ 
        $\rightarrow$ 1st layer& $\phantom{-} 4.255 \pm 0.012$& $\phantom{-} 4.202 \pm 0.009$& $\phantom{-} 5.356 \pm 0.001$& $\phantom{-} 5.507 \pm 0.001$& $\phantom{-} 6.036 \pm 0.001$& $\phantom{-} 6.195 \pm 0.001$\\ 
        $\rightarrow$ 2nd layer& $\phantom{-} 4.358 \pm 0.008$& $\phantom{-} 2.976 \pm 0.011$& $\phantom{-} 4.757 \pm 0.001$& $\phantom{-} 3.661 \pm 0.002$& $\phantom{-} 4.749 \pm 0.001$& $\phantom{-} 3.674 \pm 0.002$\\ 
        Logarithmic term& $\phantom{-} 0.543 \pm 0.000$& $\phantom{-} 0.839 \pm 0.000$& $\phantom{-} 0.597 \pm 0.000$& $\phantom{-} 0.858 \pm 0.000$& $\phantom{-} 0.624 \pm 0.000$& $\phantom{-} 0.879 \pm 0.000$\\ 
        Numerical constants& $\phantom{-} 1.681 \pm 0.000$& $\phantom{-} 2.885 \pm 0.000$& $\phantom{-} 1.681 \pm 0.000$& $\phantom{-} 2.885 \pm 0.000$& $\phantom{-} 1.681 \pm 0.000$& $\phantom{-} 2.885 \pm 0.000$\\ 
        Sample size dependency& $ -1.315 \pm 0.000$& $ -1.315 \pm 0.000$& $ -1.315 \pm 0.000$& $ -1.315 \pm 0.000$& $ -1.315 \pm 0.000$& $ -1.315 \pm 0.000$
    \end{tabular}}
    \parbox{\wd\Tbox}{
        \captionof{table}{Numerical comparison of bounds (\hyperref[eq:twolayer:ours]{Ours} and \hyperref[eq:twolayer:ledent]{Ledent et al.}) and relevant factors computed for two-layer networks trained on CIFAR100 data (5 networks, randomly initialized, per width). 
        Reported  are the mean $\pm$ standard deviation of the base 10 logarithms, \ie, the logarithms of the geometric mean and geometric standard deviation. The top table corresponds to an architecture without pooling layers and stride $32$;  the bottom table corresponds to an architecture with a pooling layer of window size $32\times 32$.
        \label{table:twolayer_standard}}
        }      
    \usebox\Tbox
    \vspace{0.5cm}
    \parbox{\wd\Tbox}{
    \begin{tabular}{ccl|cc|}
    &&& \multicolumn{2}{c|}{Width  1024 (pooling)} \\
    &&& \hyperref[eq:twolayer:ours]{\textcolor{black}{\textbf{Ours}}}    & \hyperref[eq:twolayer:ledent]{\textcolor{black}{\textbf{Ledent et al.}}} \\ \cline{3-5}
    \phantom{$\phantom{-} 7.850 \pm 0.001$}&\phantom{$\phantom{-} 7.850 \pm 0.001$}&Bound& $\phantom{-} 7.009 \pm 0.002$& $\phantom{-} 7.235 \pm 0.001$\\ 
   &&Weight \ data norms& $\phantom{-} 6.010 \pm 0.002$& $\phantom{-} 4.803 \pm 0.001$\\ 
   &&$\rightarrow$ 1st layer& $\phantom{-} 5.823 \pm 0.002$& $\phantom{-} 4.118 \pm 0.003$\\ 
   &&$\rightarrow$ 2nd layer& $\phantom{-} 5.408 \pm 0.002$& $\phantom{-} 4.523 \pm 0.001$\\ 
   &&Logarithmic term& $\phantom{-} 0.544 \pm 0.000$& $\phantom{-} 0.862 \pm 0.000$\\ 
   &&Numerical constants& $\phantom{-} 1.681 \pm 0.000$& $\phantom{-} 2.885 \pm 0.000$\\ 
   &&Sample size dependency& $ -1.315 \pm 0.000$& $ -1.315 \pm 0.000$\\ 
    \end{tabular}
    }
    \end{landscape}
    \clearpage
}
\clearpage

\label{appendix:section:comparison}

\section{Additional Experiments}
\label{appendix:section:additionalexperiments}
\subsection{Excess capacity in non-residual networks}

In addition to the experiments presented in \cref{section:empirical_evaluation} of the main manuscript, we performed the same excess capacity experiments on a non-residual convolutional network.

\textbf{Architecture.}
Essentially, we rely on the same 11-layer convolutional network with ReLU activations as described in \cref{appendix:subsection:numerical_comparison_details}, only that we substitute each convolutional layer with stride 2 by a convolutional layer with stride 1 followed by a max-pooling layer with kernel size 3 and stride 2. This is done so that we can enforce the constraints on the capacity-driving quantities via the approach described in \cref{appendix:subsection:projection_method}.
Consistent with our ResNet18 experiments, the linear classifier is fixed with weights set to the vertices of a $(\text{\#classes}-1)$ unit simplex in the output space of the network and kernel sizes of the convolutional layers are not larger than the width of their input.

\textbf{Datasets \& Training.}
Experiments are performed on the CIFAR10 and CIFAR100 benchmark datasets \cite{Krizhevsky09a}. We minimize the cross-entropy loss using SGD with momentum (0.9) for 200 epochs with batch size 256 and decay the initial learning rate (of 3e-2) with a cosine annealing scheduler after each epoch. \emph{No data augmentation is used}. For projecting onto the constraint sets, we perform one alternating projection step every 10th SGD update. After the final SGD update, we additionally do 15 alternating projection steps to ensure that the trained model is within the capacity-constrained class.

\textbf{Results.}
We observe similar phenomena as for the residual (ResNet18) network studied in
\cref{section:empirical_evaluation}. When comparing models trained with and without constraint, we
see a substantial amount of excess capacity, and this excess capacity increases with task difficulty. In fact, compared to our results with the residual network architecture,
this effect is even more pronounced as the capacity-driving quantities in the unconstrained setting
are surprisingly large. For instance, the median Lipschitz constant of the model trained on CIFAR100
is 11.53 (cf. \cref{table:app_numericalcapacity}), compared to 2.17 for the ResNet18 results in
\cref{table:numericalcapacity}. 
Notably, the capacity-driving quantities can be drastically reduced without a loss of testing accuracy and the constraints can be chosen equally across datasets. This is similar to \cref{section:empirical_evaluation} where constraints are not precisely equal, but within a small range. We also observe another manifestation of task difficulty: tightening both constraints beyond the identified operating point leads to a more rapid deterioration of the testing error as task difficulty increases (\cref{fig:app_detailedresults}, middle).

Different to \cref{section:empirical_evaluation}, the constrained models  (almost) fit the training data. However, under slightly stronger
constraints, we can still find models with testing accuracy comparable (but slightly worse) to the
unconstrained setting, but with noticeably less generalization gap (\cref{fig:app_detailedresults},
bottom). Again, this is primarily due to leaving the zero-training-error regime. We suspect that the
constraints could be much stronger, but enforcing the constraints appears to more heavily influence
optimization for networks without skip connections.
In this context, it is also worth pointing out that the constraints are quite strong for the non-residual network (proportionally much stronger than for the ResNet18 model in \cref{section:empirical_evaluation}). 
During training, we projected after every 10th SGD step, which was actually not enough to enforce the constraints throughout the whole training procedure. Only towards the end of training, when the learning rate is already small, do the constraints become satisfied. Increasing the projection frequency might thus allow for even stronger constraints.

\begin{table}[h!]
    \caption{Assessment of the capacity-driving quantities for the \textbf{non-residual} 11-layer convolutional network of this section. We list the \emph{median} over the Lipschitz constants (\textbf{Lip.}) and the ($2,1$) group norm distances (\textbf{Dist.}) across all layers. \textbf{Err.} denotes the training/testing error, \textbf{Capacity} denotes the measures (\ref{theorem:main_residual_network_covering_rademacher:eq_norms}, \ref{theorem:main_residual_network_covering_rademacher:eq_params})
    from \cref{paper:theorem:rademacher_complexity} (adapted to the non-residual setting) and \textbf{Gap} the empirical generalization gap. The top part lists results in the unconstrained regime (see \raisebox{-.48ex}{\FilledDiamondshape}  in \cref{fig:detailedresults}), the bottom part lists results at the operating point of the \underline{most restrictive} constraint combination where the testing error is on a par with the unconstrained case. 
    \textbf{Mar.} denotes the margin parameter $\gamma$ used for computing the capacity measures, which we choose such that the unconstrained and constrained models have the same ramp loss value.
    \label{table:app_numericalcapacity}}
    \begin{small}
    \begin{center}
        \setlength{\tabcolsep}{8pt}
    \begin{tabular}{rccccccc}
        \toprule
                & \textbf{Lip.} 
                & \textbf{Dist.} 
                & \textbf{Mar.}
                & \textbf{Err. (Tst)} 
                & \textbf{Err. (Trn)}
                & \textbf{Capacity} (\ref{theorem:main_residual_network_covering_rademacher:eq_norms}, \ref{theorem:main_residual_network_covering_rademacher:eq_params})
                & \textbf{Gap} \\
        \midrule
        CIFAR10         
            & $4.66$ 
            & $370.0$  
            & $16.4$
            & $0.17$
            & $0.00$
            & \textcolor{uniblue}{1.2$\cdot$10$^{12}$} / \textcolor{unired}{8.1$\cdot$10$^2$}  & $0.17$ 
            \\ 
        CIFAR100 
            & $11.53$ 
            & $854.0$ 
            & $52.1$
            & $0.47$
            & $0.00$
            & \textcolor{uniblue}{1.0$\cdot$10$^{16}$} / \textcolor{unired}{9.2$\cdot$10$^2$} & $0.47$
            \\
        \midrule
        \midrule
        CIFAR10 
            & $\underline{1.80}$ 
            & $\underline{200.0}$  
            & $10.0$
            & $0.17$ 
            & $0.00$
            & \textcolor{uniblue}{6.6$\cdot$10$^8$} / \textcolor{unired}{6.9$\cdot$10$^2$} & $0.17$
            \\
        CIFAR100 
            & $\underline{1.80}$ 
            & $\underline{200.0}$  
            & $10.0$
            & $0.47$
            & $0.03$
            & \textcolor{uniblue}{6.6$\cdot$10$^8$} / \textcolor{unired}{6.9$\cdot$10$^2$}  
            & $0.43$ 
            \\
        \bottomrule
    \end{tabular}
    \end{center}
\end{small}
\vspace{-0.1cm}
\end{table}

\begin{figure*}[t!]
    \begin{center}
        \includegraphics[scale=0.48]{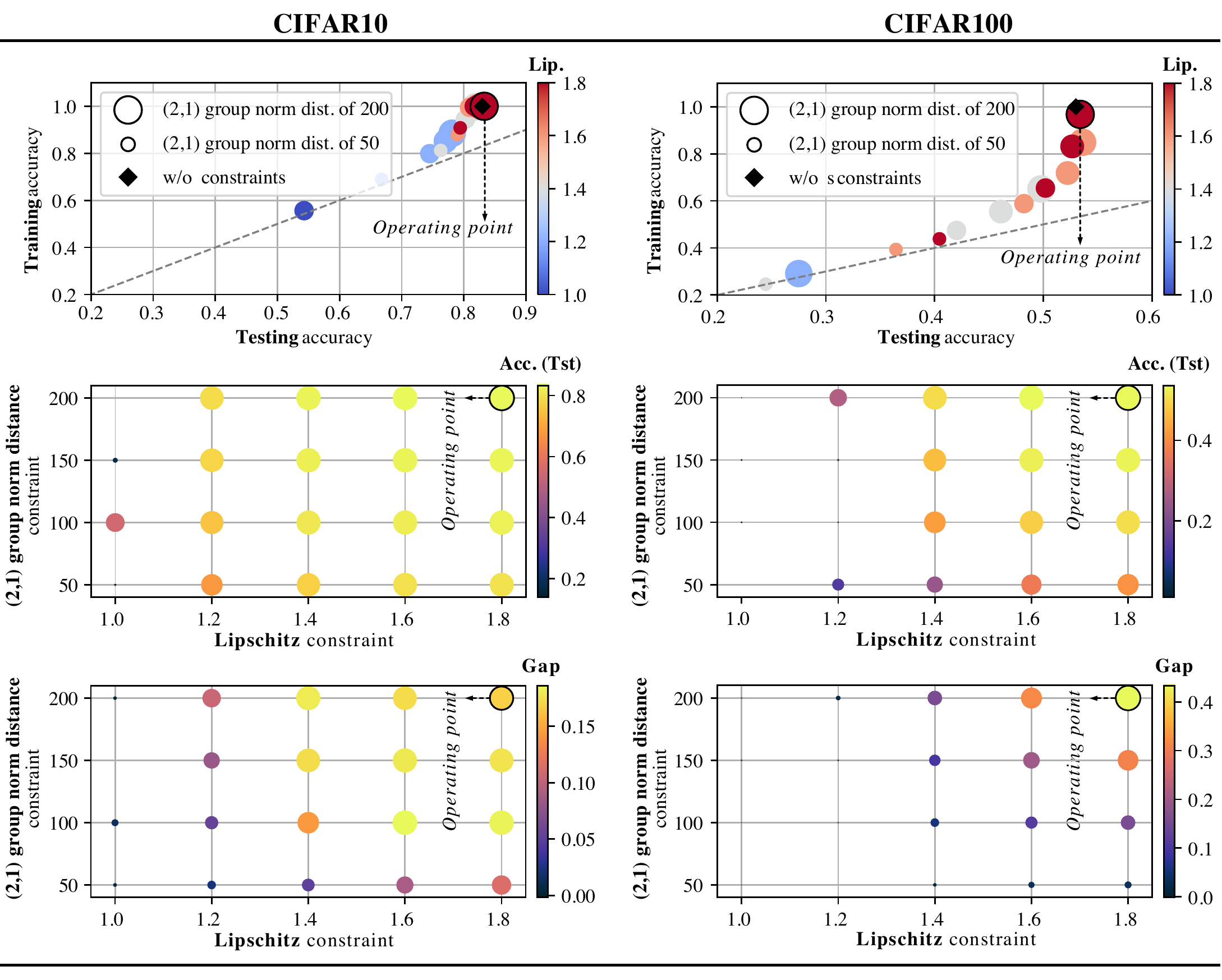}
    \end{center}
    \caption{Fine-grained analysis of training/testing accuracy in relation to the Lipschitz constraint and the ($2,1$) group norm distance to initialization constraint for the 11-layer \textbf{non-residual} convolutional network of this section. We see that testing accuracy can be retained (relative to \raisebox{-.48ex}{\FilledDiamondshape}) for a range of fairly restrictive constraints (\emph{top row}), compared to the unconstrained regime (cf. \textbf{Lip./Dist.} columns in the top part of \cref{table:app_numericalcapacity}). However, this range noticeably narrows with increasing task difficulty (\emph{middle row}). \emph{Best-viewed in color.}
    \label{fig:app_detailedresults}}
\end{figure*}

\subsection{Projection method}
\label{appendix:subsection:projection_method}
A key aspect of the experiments in \cref{section:empirical_evaluation} is to obtain, for each pair of Lipschitz constant and ($2,1$) group norm constraints, a model with testing accuracy as high as possible. The quality of such a model depends, to a large extent, on the way the constraints are enforced. This section specifies the projection method used for the experiments and provides additional background information. 

As mentioned in \cref{section:empirical_evaluation}, we utilize \emph{orthogonal projections}.
Given $x_0 \in \mathbb R^d$ and a nonempty closed convex set $A\subset \mathbb R^d$, the orthogonal projection of $x_0$ onto $A$ is defined as the unique
\begin{equation}
    x_{\text{orth}} = \argmin_{x \in A} \norm{x-x_0}_{2}
    \enspace.
    \label{eqn:appendix:orthoprojection}
\end{equation}

Orthogonal projections have several beneficial properties. First, if $f\colon\mathbb R^d\to \mathbb
R$ is a strictly convex function, then, for appropriately chosen step sizes, projected gradient
descent (\ie, gradient descent with a subsequent orthogonal projection onto $A$ after each step)
converges to the minimizer $\argmin_{x \in A} f(x)$~\cite{Bertsekas99a}. Second, for a tuple
$(\mathcal C_1, \dots \mathcal C_N)$ of closed convex sets $\mathcal C_j\subset \mathbb R^d$ with
orthogonal projections $P_{\mathcal C_j}$, alternating orthogonal projections, i.e., the sequence
$x_{i+1} = P_{\mathcal C_N} \circ \dots \circ P_{\mathcal C_1} (x_i)$ converges \cite{Bauschke96a}
to a point in the intersection $\mathcal C_1 \cap \dots \cap \mathcal C_N$ (if it
is non-empty).
Notably, there are variants of alternating orthogonal projections, e.g., Dykstra's algorithm \cite{Dykstra83a}, which converge to the orthogonal projection $P_{\bigcap_i C_i}$ onto the intersection $\bigcap_i C_i$.
However, there is a key disadvantage of orthogonal projections. Being defined by the optimization problem $x_{\text{orth}} = \argmin_{x \in A} \norm{x-x_0}_{2}$ of \cref{eqn:appendix:orthoprojection}, they often can only be computed numerically and might require a large compute budget.

In \cref{section:empirical_evaluation}, the convex set is $\mathcal C = \set{K \in \mathbb{R}^{c_\text{\it in} \times c_\text{\it out} \times k_h \times k_w}:~ \lVert K-K^0 \rVert_{2,1}\le b,~\Lip(\phi_K) \leq s }$. The orthogonal projection onto $\mathcal C$ is unknown, but the alternating orthogonal projection onto the sets 
\begin{equation}
    \begin{split}
        \mathcal{C}_1 & = \{K \in \mathbb R^{c_\text{\it in} \times c_\text{\it out} \times h \times w}: \lVert K-K^0 \rVert_{2,1} \leq b\}\,,\\
        \mathcal{C}_2 & = \{K \in \mathbb R^{c_\text{\it in} \times c_\text{\it out} \times h \times w}: \Lip(\phi_K) \leq s\}\,,\\
        \mathcal{C}_3 & = \{K \in \mathbb R^{c_\text{\it in} \times c_\text{\it out} \times h \times w}: K_{ijkl} = 0 \text{ for } k>k_h, j>k_w\}
        \enspace, 
    \end{split}
\end{equation}
still defines a projection onto $\mathcal C\subset \mathbb R^{c_\text{\it in} \times c_\text{\it out} \times k_h \times k_w} $ considered as subset of $\mathbb R^{c_\text{\it in} \times c_\text{\it out} \times h \times w}$. Importantly, all three orthogonal projections are known. The projection onto $\mathcal C_1$ is due to \cite{Liu09a}.
The projection onto $\mathcal C_2$ requires a singular value decomposition of $M_K$, \ie, the ${hw c_\text{\it in} \times hw c_\text{\it out}}$ matrix corresponding to the linear map $\phi_K$. As this matrix can be quite large, this is infeasible in practice. However, \cite{Sedghi19a} show that for strides 1\footnote{Extensions to strides >1 are not straightforward but seem possible.}, due to the particular structure of convolutions, it suffices to compute the singular value decomposition of $hw$ matrices of size  $c_\text{\it in} \times c_\text{\it out}$. Still, the computation of the projection onto $\mathcal C_2$ is the bottleneck of the training procedure in \cref{section:empirical_evaluation}. The orthogonal projection onto $\mathcal C_3$, which is a plane, is realized by setting the corresponding coordinates to zero.

Another approach is to use \emph{radial projections}.
Given $x_0 \in \mathbb R^d$ and a norm $\norm{\cdot}$, the {radial projection} of $x_0$ onto the $\norm{\cdot}$-ball $B(r,y,\norm{\cdot})$ of radius $r$ centered at $y$ , is defined as
\begin{equation}
    x_{\text{rad}} = x_0 - \left(1 - \frac{r}{\norm{x_0-y}}\right) (x_0-y)\,  \mathbbm 1_{\norm{x_0-y}>r}
    \enspace.
\end{equation}
Such a projection is called radial, as it translates the point $x_0$ in radial direction \wrt the ball $B(r,y,\norm{\cdot})$ such that it lands on the boundary (if it is not already in $B(r,y,\norm{\cdot})$).

Notably, $\mathcal C_1$ and $\mathcal C_2$ are both balls, one with respect to the ($2,1$) group norm, the other with respect to the spectral norm of the matrix $M_K$ associated to $K$. Importantly, the spectral norm can be easily estimated by the power method for convolutional layers \cite{Gouk21b,Li19a}, so radial projections have far less computational overhead. However, alternating radial projections are not guaranteed\footnote{empirically, we still observed convergence} to converge to a point in the intersection $\mathcal C$. Furthermore, by definition, we have $\norm{x-x_{\text{rad}}}_2 \ge \norm{x-x_{\text{orth}}}_2$, so we expect radial projections to yield inferior results (\wrt the constraint strengths that can be enforced).

We evaluated three different approaches to obtain models with constrained capacity: (1) training with a variant of projected gradient descent, where we perform one alternating orthogonal projection step after every 15th SGD update;  (2) performing one alternating radial projection step after every SGD update\footnote{The increased projection frequency is possible because of the reduced computational overhead of radial projections compared to orthogonal projections.}; (3) orthogonal projection onto $\mathcal C$ of an already trained unconstrained model, using 100 iterations of Dykstra's algorithm.
Our findings are summarized in \cref{fig:projections}. We see that, alternating orthogonal projections during training allow for the strongest constraints, without a drop in the testing accuracy. This is expected, because 
they divert the weights less from the training trajectory than radial projections. By the same logic it is obvious that projecting only at the \emph{end} of training is not feasible, as the weights of the trained network are already too far away from the constraint set.
We conclude that alternating orthogonal projections allow for the best estimate of excess capacity.

\begin{figure*}[h!]
    \begin{center}
        \includegraphics[width=\textwidth]{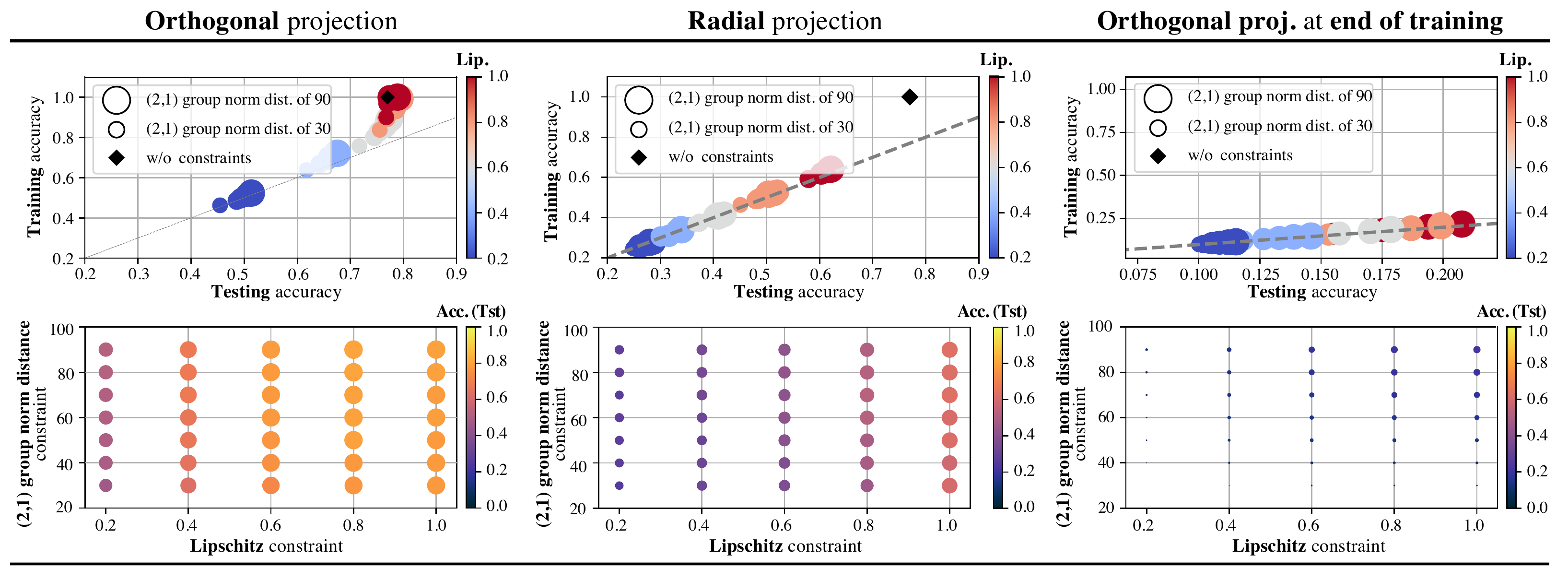}
    \end{center}
    \vspace{-0.2cm}
    \caption{Comparison of different \emph{projection techniques} for ResNet18 models trained on CIFAR10. \emph{Best-viewed in color.}
    \label{fig:projections}}
\end{figure*}

\vspace{-0.2cm}

\subsection{Comparison between constrained and unconstrained models beyond testing error}

So far, we have analyzed to which extent the weights of neural networks can be constrained without a loss of testing accuracy. In particular, we have identified the maximal constraint strength (\ie, the operating point) such that the testing error of the constrained models is on a par with the one of unconstrained models. However, this does not imply that constrained models and unconstrained ones can be used interchangeably, as they might differ in other aspects. In this section, we will study how pronounced such differences are with respect to (i) biases to particular classes, (ii) susceptibility to adversarial attacks, and (iii) compressibility in terms of the number of weights (via weight pruning).

For the evaluation, we use \textbf{25 unconstrained} and \textbf{25 constrained} models trained on CIFAR100, with the same architecture and optimization hyperparameters as listed in the main text, i.e., \cref{section:empirical_evaluation}. As constraint strength, we choose a layer-wise Lipschitz constant of 0.8 and a distance constraint of 70.

\subsubsection{Biases to particular classes}
For each of the 50 models, we counted how often each class is predicted on the testing data, which consists of 100 images per class. If there are no biases to particular classes, the counts should be distributed around this value with preferably small spread. \cref{fig:class_freq} visualizes the results. We immediately see, that for the unconstrained models, the predictions per class are more uniformly distributed with the average class close to 100 and small standard deviations. In contrast, for the constrained models, the standard deviations are much larger. Most striking is the peak at class index 21, indicating that the constrained models are indeed biased to this particular class (chimpanzee). In fact, only one of the 25 models predicted this class less than 100 times. Notably, the unconstrained models are also biased towards this class, as the error region of (mean $\pm$ standard deviation) does not contain 100.
Overall, there are more favored/disfavored classes for the unconstrained models (32 \emph{vs.} only 5 for the constrained models), but for the constrained models, the biases are more pronounced.

\begin{figure}
    \includegraphics[width=\textwidth]{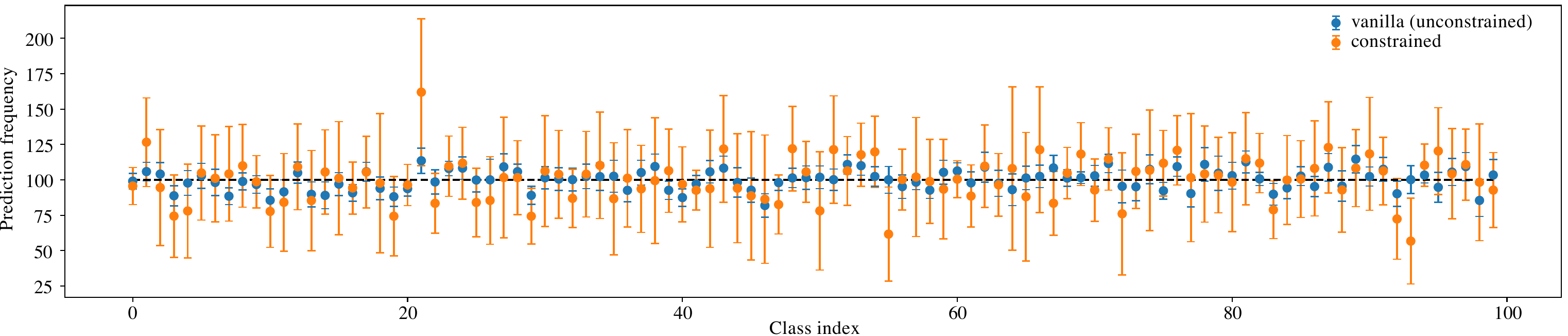}
    \caption{Prediction frequencies per class index on the testing portion of CIFAR100, averaged over the predictions of 25 constrained and 25 unconstrained (vanilla) ResNet18 models. \emph{Best-viewed in color.} \label{fig:class_freq}}
\end{figure}

\subsubsection{Susceptibility to adversarial attacks}
We tested several adversarial attacks (FGSM, FGM, L2PGD, LinfPGD, L2DeepFool,
L2AdditiveGaussianNoise, L2AdditiveUniformNoise, L2ContrastReduction, GaussianBlur) using the
\texttt{foolbox} \cite{rauber2017foolboxnative} Python package. To compare constrained \emph{vs.}
unconstrained models, we extract 1024 images from the testing data, which are correctly classified
by all 50 models (25 constrained models, 25 unconstrained models) on which the attacks are
evaluated. The fraction of correctly classified images for increasing attack strengths is visualized
in \cref{fig:adv_attacks}. As can be seen from the figure, constrained models are less susceptible
to the gradient-based attacks FGSM, FGM, L2PGD, and LinfPGD. For L2DeepFool, contrast reduction (L2ContrastReductionAttack)
and
Gaussian blur (GaussianBlurAttack), constrained and unconstrained models are equally affected. To our surprise, the
constrained models more vulnerable to additive Gaussian and uniform noise.

\begin{figure}
    \includegraphics[]{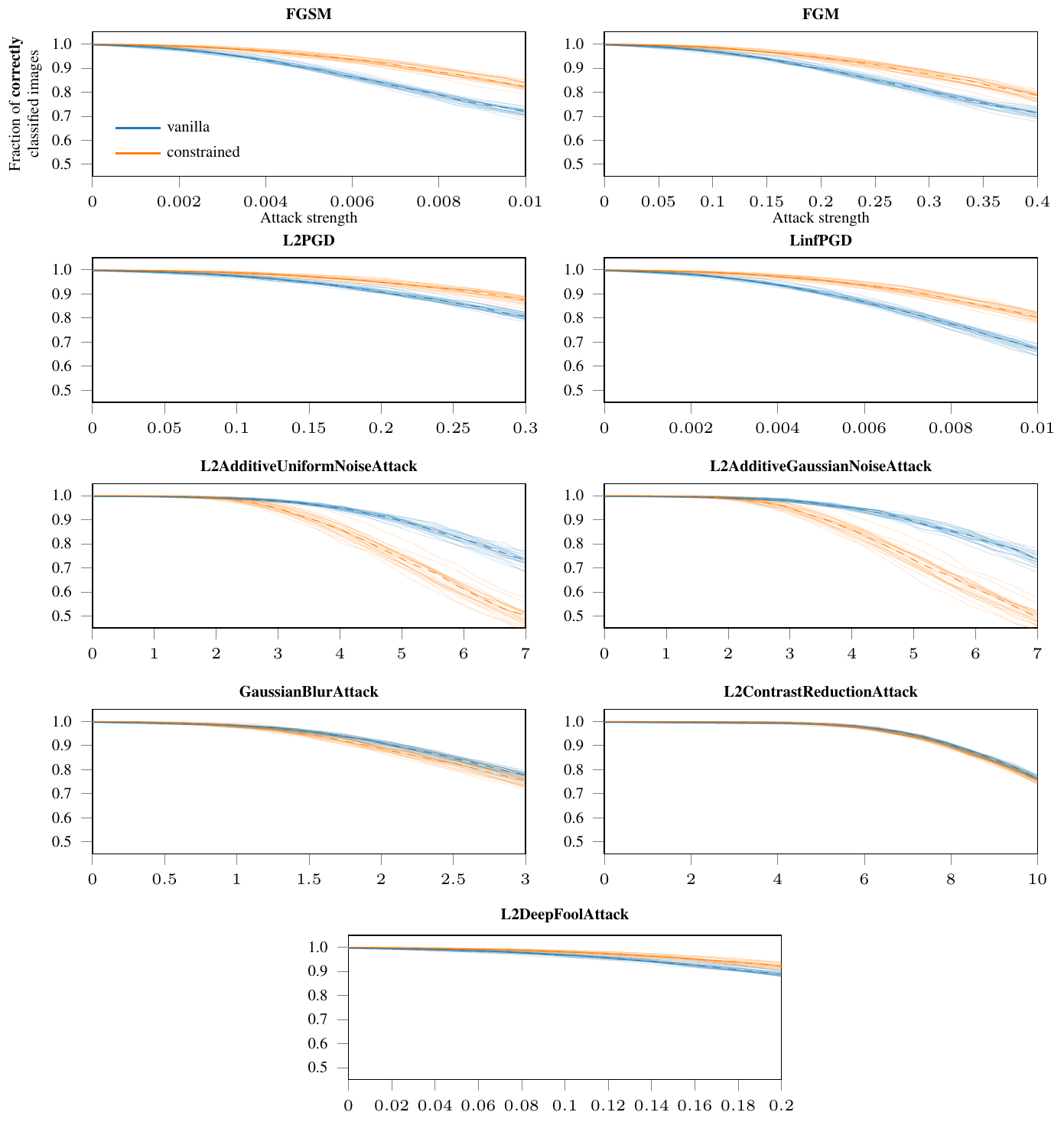}
    \caption{Results of running several adversarial attacks against constrained and unconstrained ResNet18 models on CIFAR100 across varying attack strengths ($x$-axis). Each solid line represents one trained model, the dashed lines represent the medians. The $y$-axis shows the fraction of correctly classified images that remain correct as attack strength increases. \emph{Best-viewed in color.}
    \label{fig:adv_attacks}}
\end{figure}

\subsubsection{Compressibility via weight pruning}
We measure compressibility for global (unstructured) $l_1$ weight pruning. This simple pruning technique identifies a predefined fraction of a model's parameters (\ie, elements of the weight tensors/matrices) and sets them to zero. \emph{No subsequent fine-tuning steps were performed.}
\begin{wrapfigure}{r}{0.45\textwidth}
    \vskip-0.2cm
    \includegraphics[width=0.45\textwidth]{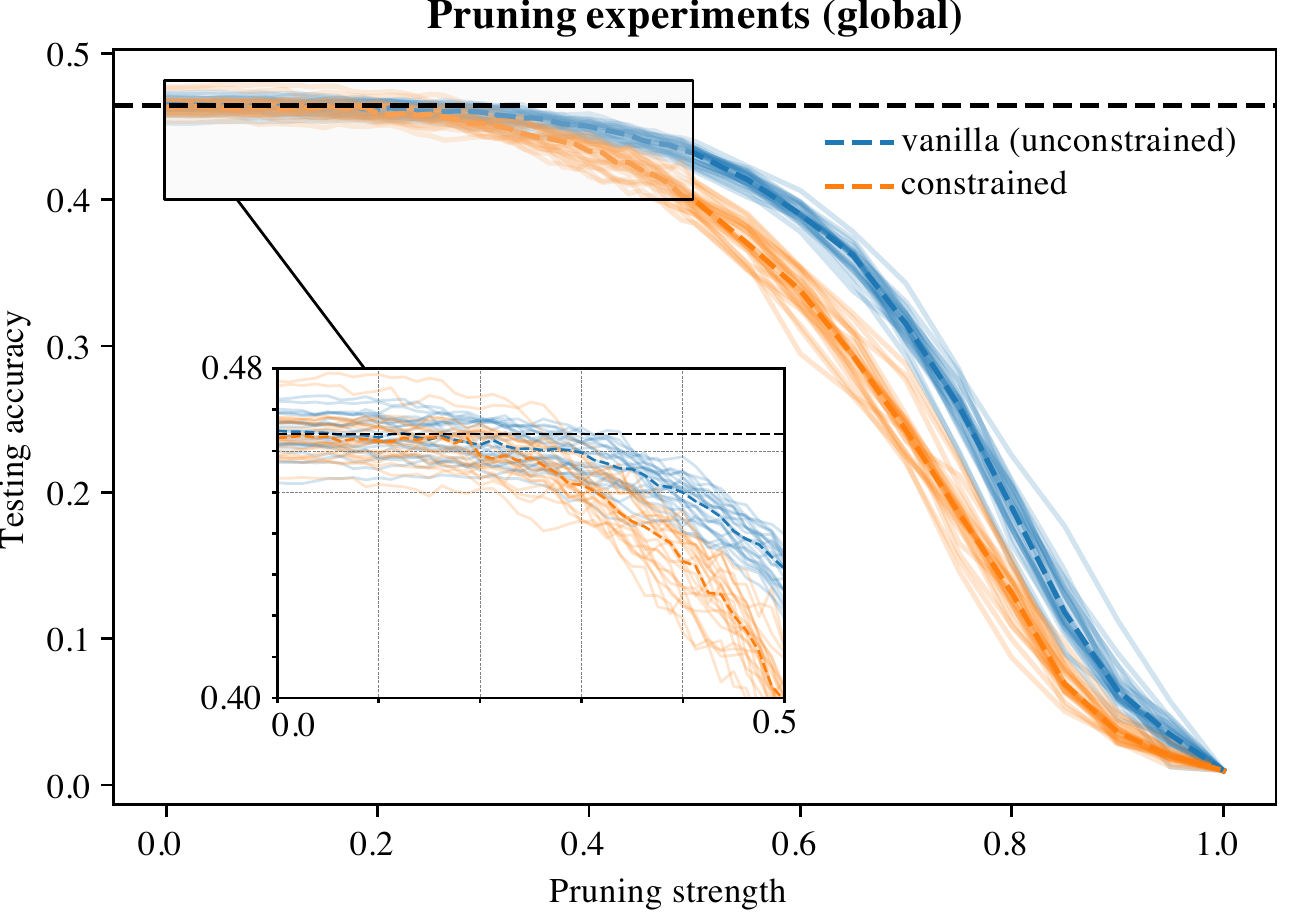}
    \caption{Results of global (unstructured) $l_1$ weight pruning on 25 constrained and 25 unconstrained ResNet18 models. Each solid line represents one trained model. \emph{Best-viewed in color.}\label{fig:pruning}}
    \vskip-0.15cm
  \end{wrapfigure}
As illustrated in \cref{fig:pruning}, unconstrained models are less affected by this pruning technique than constrained models. Yet, constrained models can still be pruned, and, in the range of pruning strengths 0 to 0.15, the median testing accuracy drops only marginally for both model types. Only at larger pruning strengths do differences between the model types become visible. If we consider a median testing accuracy of 46\%, resp. 45\%, to be acceptable, then this threshold allows for pruning 20\%, resp. 30\%, of the weights of constrained models and 30\%, resp 40\%, of the unconstrained models. Of course, the pruned models might not satisfy the constraints anymore. In particular, the distance constraint gets violated, as with increasing constraint strength, the distance to initialization converges to the norm of the initialization. Typically, the latter is already larger than the distance constraint (of 70).
Thus, when combining norm constraints with weight pruning, constraining the distance with respect to the zero weight appears to be more sensible. This can be done, as the distance constraint in \cref{paper:theorem:rademacher_complexity} is not required to be with respect to the initialization; it can be chosen relative to any reference weight as long as the reference weight does not depend on the training data.

\subsection{Hardware resources}
\label{appendix:subsection:hardwareresources}

All experiments were run on an Ubuntu Linux 20.04.4 LTS system with 128 GB of main memory, an Intel\textregistered Core™ i9-10980XE processor and two 
NVIDIA GeForce RTX 3090 graphics cards (24 GB memory, CUDA 11.4, driver version 470.129.06). All models are implemented in Pytorch (v1.10).

\clearpage
\section{Proofs}
In the following sections, we present proofs for the theoretical results listed in the manuscript as well as additional supplementary results.

\label{appendix:section:proofs}
\subsection{Preliminaries}
In terms of notation, we consider spaces $\mathcal F$ of functions $f\colon (\mathcal X, \norm{\cdot}_{\mathcal X}) \to (\mathcal Y, \norm{\cdot}_{\mathcal Y})$ between normed spaces. We write
 $$
 \Lip(f) = \sup_{x_1, x_2 \in \mathcal X} \frac{\norm{f(x_1)- f(x_2)}_{\mathcal Y}}{\norm{x_1-x_2}_{\mathcal X}}$$ 
 and 
 $$\Lip(\mathcal F) = \sup_{f\in \mathcal F} \Lip(f)$$
for the Lipschitz constant of $f$ and the supremal Lipschitz constant of $\mathcal{F}$, respectively. For the remainder of the section, all function spaces $\mathcal F$ will have bounded Lipschitz constants $\Lip(\mathcal F) < \infty$. Such function spaces are vector spaces, where addition and scalar multiplication are defined pointwise via the vector space structure on $\mathcal Y$, i.e.,
$(f+g)\colon x\mapsto f(x) + g(x)$ and $(\alpha f)\colon x \mapsto \alpha f(x)$.

We equip $\mathcal F$ with a \emph{data-dependent norm}, defined below.

\begin{definition}
    Let $\mathcal F$ be a space of functions $f: (\mathcal X, \norm{\cdot}_{\mathcal X}) \to (\mathcal Y, \norm{\cdot}_{\mathcal Y})$ between normed spaces.
    The \emph{data-dependent norm}, denoted as $\norm{\cdot}_X$, on 
    $\mathcal F$ restricted to $X = \tuple{x_1\dots,x_n}\in \mathcal X^n$, i.e., $\mathcal F|_X$, is defined as
    \begin{equation}
        \norm{f}_X
        =
        \sqrt{ \sum_{i=1}^n \norm{f(x_i)}_{\mathcal Y}^2} 
    \enspace.
    \end{equation}
    \label{appendix:def:data_dependent_norm}
\end{definition}

\begin{remark}
    This norm is a seminorm on $\mathcal F$ and a norm on $\mathcal F|_X = \{f|_{\set{x_1, \dots, x_n}}:~ f\in \mathcal F\}$.
    If $f \in \mathcal F$ has norm $\norm{f}_X = 0$, then it holds that $\forall i: f(x_i) = 0$.
    Thus, $f$ is the zero element in $\mathcal F|_X$, but not necessarily the zero element in $\mathcal F$, as there might exist $v \in \mathcal X\setminus \set{x_1,\dots, x_n}$ with $f(v) \neq 0$.
\end{remark}
Two fundamental properties concerning compositions of functions are worth pointing out:
\begin{align}
    \label{eq:empirical_norm_properties}
    \norm{f\circ g}_X = \norm{f}_{g(X)}
    \qquad \text{and} \qquad
    \norm{f\circ g - f\circ h}_X \le \Lip(f) \norm{g-h}_{X}
    \enspace.
\end{align}

Further, we recall the definition of \emph{covering numbers}.
\begin{definition}
    Let $(\mathcal H, \norm{\cdot})$ be a normed space, $S\subset \mathcal H$ and $\epsilon>0$.
    We call any subset $U\subset S$ an internal $\epsilon$-cover of $S$ if for every $s \in S$ there exists $u\in U$ such that $\norm{s-u}\le \epsilon$. The internal covering number $\mathcal N^{\text{int}}(\mathcal H, \epsilon, \norm{\cdot})$ is the cardinality of the smallest internal $\epsilon$-cover of $S$, i.e., 
    \begin{equation}
        \mathcal N^{\text{int}}(S, \epsilon, \norm{\cdot})
        = \min\left(\set{|U|:~ \text{U is an internal $\epsilon$-cover of $S$}}\right)
        \enspace
    \end{equation}

    Dropping the requirement $U \subset S$, we analogously define external $\epsilon$-covers $U\subset \mathcal H$ and external covering numbers $\mathcal N^{\text{ext}}(S, \epsilon, \norm{\cdot})$. 
\end{definition}

In the manuscript, if not stated otherwise, covers will \emph{always be internal} and covering numbers will be denoted as $\mathcal N = \mathcal N^{\text{int}}$.

Internal and external covering numbers are related via the following chain of inequalities:
\begin{equation}
    \mathcal N^{\text{ext}}(S, \epsilon, \norm{\cdot}) 
    \le
    \mathcal N^{\text{int}}(S, \epsilon, \norm{\cdot})
    \le
    \mathcal N^{\text{ext}}(S, \epsilon/2, \norm{\cdot})
    \enspace.
    \label{eq:covers_int_ext}
\end{equation}
The first inequality follows directly from the definition of covering numbers; the second one follows from the triangle inequality.
Furthermore, for any subset $T\subset S$, it holds that 
\begin{equation}
    \mathcal N^{\text{ext}}(T, \epsilon, \norm{\cdot}) 
    \le
    \mathcal N^{\text{ext}}(S, \epsilon, \norm{\cdot})
    \enspace.
    \label{eq:covers:subset}
\end{equation}
Notably, this is not true for internal coverings. For example, the unit ball in $\mathbb R^d$ defines an internal cover of itself, whereas an annulus cannot be covered internally with only one ball of radius $1$.

\subsection{Single-layer covering number bounds}
\label{appendix:subsection:singlelayerbounds}

This section contains the covering number bounds for single convolutional layers.
For simplicity, we will first present the special case of the single-layer covering number bound for \textbf{1D convolutions} with one channel, stride 1, odd kernel size and input size preserving (zero) padding.
The proof of the general case then follows along the same line of arguments, but is more tedious, due to the additional notation and subindices.

Let $X=(x_1,\dots, x_n) \in \mathcal X^n$ with $x_i\in \mathcal X = \mathbb R^h$. Further, let $K \in \mathbb R^k$, with $k$ odd, be a convolutional kernel and $\phi_K: \mathbb R^h \to \mathbb R^h$ the corresponding convolutional map, which is defined coordinate-wise as 
\begin{equation}
    \label{definition:1dconv}
    [\phi_K (x)]_{i}
    = 
    \sum_{\alpha=-\frac{k-1}{2}}^{\frac{k-1}{2}}
    K_{\alpha} \, x_{i+\alpha} \mathbbm{1}_{i+\alpha \in [1,h]}
    \enspace.
\end{equation}
For the norm of the data $X$, we write
\begin{equation}
\norm{X} = \sqrt{\sum_{i=1}^n \norm{x_i}^2}\enspace.
\end{equation}
Our covering number bounds hinge on the seminal Maurey sparsification lemma. We state one variant, see \cite[Lemma A.6]{Bartlett17a}.

\begin{lemma}[Maurey sparsification lemma]
    Fix a Hilbert space $\mathcal H$ with norm $\norm{\cdot}$. Let $U \in \mathcal H$ be given with representation $U = \sum_{i=1}^d \alpha_i V_i$ where $V_i \in \mathcal H$, $\alpha \in \mathbb R^d_{\ge 0} \setminus \set{0}$ and $\sum_i |\alpha_i| \le 1$. Then, for any positive integer $m$, there exists a choice of non-negative integers $(m_1, \dots, m_d)$ with $\sum_i m_i = m$, such that 
    \begin{equation}
        \norm{
            U - \frac{1}{m} \sum_{i=1}^d m_i V_i
        }^2
        \le \frac{1}{m}  \max_{i=1,\dots,d} \norm{V_i}^2
        \enspace.
    \end{equation}
    \label{appendix:maurey_lemma}
\end{lemma}

\begin{theorem}[Single-layer covering number bound -- {Simple 1D variant}]
    \label{theorem:single_layer_simple}
    Let $b>0$, $k$ odd and let $\mathcal{F} = \{\phi_K|~ K\in\mathbb{R}^k, \norm{K}_1\le b\}$ be the set of 1D convolutions determined by kernels $K \in \mathbb R^k$ with $\norm{K}_1\le b$.
    For any $X=(x_1,\dots, x_n) \in \mathcal X^n = \mathbb R^{n \times h}$ and $\epsilon>0$, the covering number $\mathcal{N}(\mathcal{F}, \epsilon, \norm{\cdot}_X)$ satisfies
    \begin{equation}
        \label{theorem:single_layer_simple:eq_norms}
        \log \mathcal{N}(\mathcal{F}, \epsilon, \norm{\cdot}_X) 
        \le
        \ceil{\frac{\norm{X}^2 b^2}{\epsilon^2}} \log(2k)
    \end{equation}
    and
    \begin{equation}
            \label{theorem:single_layer_simple:eq_params}
            \log\mathcal N(\mathcal F, \epsilon, \norm{\cdot}_X )
            \le
            (2 k-1) \log\left(1+
            \ceil{
                \frac{\norm{X}^2 b^2}{\epsilon^2}
            }
            \right)
            \enspace.
    \end{equation}
\end{theorem}

\begin{proof}
    \label{proof:single_simple}
    We rewrite the coordinate-wise definition of the 1D convolutional 
    operation from \cref{definition:1dconv} with unit stride as 
    \begin{align*}
        [\phi_K (x)]_{i} = 
        \sum_{j=1}^h 
        \underbrace{
        \left(
            \sum_{\alpha=-\frac{k-1}{2}}^{\frac{k-1}{2}}
            \mathbbm{1}_{j = \alpha + i}  \,K_\alpha 
        \right)}_{= M_{ij}} x_j
        =
        \sum_{j=1}^h 
        M_{ij} x_j
        \enspace.
    \end{align*}
    Thus, convolution is a linear map parametrized by a matrix 
    $M\in \mathbb R^{h\times h}$ with entries 
    $$M_{ij} = \sum_{\alpha=-\frac{k-1}{2}}^{\frac{k-1}{2}}
    \mathbbm{1}_{j = \alpha + i}  \,K_\alpha\enspace.$$
    In particular, we can write $M= \sum_{\alpha=-(k-1)/2}^{(k-1)/2} K_\alpha M^{(\alpha)}$ with $M^{(\alpha)}_{ij} = \mathbbm{1}_{j-i=\alpha}$ and 
    note that for every $x\in \mathbb R^h$ and every $\alpha$, we have
    $$\norm{M^{(\alpha)}x}\le \norm{x}\enspace.$$

    For example, if $h=5$ and $k=3$, we have 
    \begin{gather*}
        M = \begin{pmatrix}
            K_0 & K_1 &0 & 0 & 0 \\
            K_{-1} & K_0 & K_1 &  0  & 0\\
            0& K_{-1} & K_{0} & K_{1} &0 \\
            0 & 0 &  K_{-1} & K_{0} & K_{1} \\
            0& 0 & 0 & K_{-1} & K_{0}
            \end{pmatrix}%
            ,
        \\
        M^{(-1)} = 
        \begin{pmatrix}
            0 & 0 & 0 & 0 & 0 \\
            1 & 0 & 0 & 0 & 0\\
            0 & 1 & 0 & 0 & 0 \\
            0 & 0 & 1 & 0 & 0 \\
            0 & 0 & 0 & 1 & 0
        \end{pmatrix}%
        ,~~
        M^{(0)} = 
        \begin{pmatrix}
            1 & 0 & 0 & 0 & 0 \\
            0 & 1 & 0 & 0 & 0\\
            0 & 0 & 1 & 0 & 0 \\
            0 & 0 & 0 & 1 & 0 \\
            0 & 0 & 0 & 0 & 1
        \end{pmatrix}%
        ,~~
        M^{(1)} = 
        \begin{pmatrix}
            0 & 1 & 0 & 0 & 0 \\
            0 & 0 & 1 & 0 & 0\\
            0 & 0 & 0 & 1 & 0 \\
            0 & 0 & 0 & 0 & 1 \\
            0 & 0 & 0 & 0 & 0
        \end{pmatrix} 
        \enspace.
    \end{gather*}
    In summary, we have  
    $$\phi_K(x) 
    =\sum_{\alpha=-(k-1)/2}^{(k-1)/2} {K_\alpha} M^{(\alpha)} x
    = \sum_{\alpha=-(k-1)/2}^{(k-1)/2} \frac{K_\alpha}{b} b M^{(\alpha)} x
    \enspace.$$
    By assumption, we have $\sum_{\alpha} \left|\frac{K_\alpha}{b}\right| = \frac{\norm{K}_1}{b}\le1$ and so we can instantiate Maurey's sparsification lemma (\cref{appendix:maurey_lemma}) on the Hilbert space $(\mathcal{F}|_X, \norm{\cdot}_X)$ for $\set{V_1, \dots V_{2d} } = \set{x\mapsto \pm b M^{(\alpha)}x |~\alpha=1,\dots, d} \subset \mathcal F$. 
    As a consequence, for any convolutional kernel $K \in \mathbb R^k$ and any $m\in \mathbb N$, there exist $(m_1, \dots, m_{2d})$ with $\sum_{i=1}^{2d} m_i = m$ such that 
    \begin{equation*}
        \norm{
            \phi_K - \frac{1}{m} \sum_{i=1}^{2d} m_i V_i
        }_X^2
        \le \frac{1}{m}   \max_{i=1\dots,2d} \norm{V_i}_X^2
        \enspace.
    \end{equation*}
    Thus, for fixed $\epsilon>0$, if we choose $m\in \mathbb N$ such that 
    $\frac{1}{m}  \max_{i} \norm{V_i}_X^2 \le \epsilon^2$, then 
    the solutions (in $m_i$), of $\sum_{i=1}^{2d} m_i = m$, 
    define an $\epsilon$-cover 
    \begin{equation*}
        \set{\frac{1}{m} \sum_{i=1}^{2d} m_i V_i|~m_i \in \mathbb N_{\ge 0},~ \sum_{i=1}^{2d} m_i = m} \subset \mathcal F
        \enspace.
    \end{equation*}
    As the number of non-negative $2d$-tuples that add up to $m$ is equal to\footnote{This number equals the number of possibilities to separate $m$ objects by $2d-1$ delimiters. This corresponds to choosing $2d-1$ elements (position of delimiters) from a set of $m+2d-1$ elements (objects + delimiters).}
    $$N(m,d) = \binom{m + 2d -1}{2d-1}\enspace,$$
    this means that
    $\mathcal F$ has an $\epsilon$-cover of cardinality at most 
    $N(m,d)$, and so
    $\mathcal N(\mathcal F, \epsilon, \norm{\cdot}_X ) \le N(m,d)$.

    Since for all $i \in \{1,\dots, 2d\}$, the norms $\norm{V_i}_X^2$ satisfy
    \begin{equation*}
        \norm{V_i}_X^2 
        = \norm{b M^{(\alpha_i)}}_X^2
        = b^2 \sum_{j=1}^n \norm{M^{(\alpha_i)} x_j}^2
        \le  b^2 \sum_{j=1}^n \norm{x_j}^2
        = b^2 \norm{X}^2
        \enspace,
    \end{equation*}
    we can choose $\mathbb N \ni m = \ceil{\frac{b^2 \norm{X}^2}{\epsilon^2} }$.

    The theorem then follows from two particular bounds on $N(m,d)$, 
    see \cref{theorem:binomial_coefficient}.
    These are 
    $$\binom{m + 2d -1}{2d-1} \le (2d)^m\enspace,$$ 
    which implies  \cref{theorem:single_layer_simple:eq_norms},
    and 
    $$\binom{m + 2d -1}{2d-1} \le (1+m)^{2d-1}
    \enspace,$$
    which implies  \cref{theorem:single_layer_simple:eq_params}.
\end{proof}

\begin{remark}
    The definition of the convolution operation in \cref{definition:1dconv} corresponds to convolutional layers with zero-padding, such that the dimensionality $h$ of the data remains unchanged (\ie, input-size preserving). The covering bound equally holds for other types of padding, corresponding to other matrices $M^{(\alpha)}$, as long as 
    $$\norm{M^{(\alpha)}x} \le x$$ 
    for all $x$. In particular, it holds for convolutional layers with circular padding, where the matrices $M^{(\alpha)}$ become \emph{permutation matrices}.
\end{remark}

Next, we study the general case of 2D multi-channel convolutions with strides. 
To that end, let $\phi_K$ be the map determined by a weight tensor $K \in \mathbb R^{c_\textit{out}\times c_\textit{in} \times k_h \times k_w}$, where
$c_\textit{out}, c_\textit{in}$ denote the number of output and input channels, resp., and $(k_h, k_w)$ is the spatial extension of the kernel.
For input images $x \in \mathbb R^{c_\textit{in} \times h \times  w}$, 
convolution $\phi_{K,(s_h,s_w)}: \mathbb R^{c_\textit{in} \times h \times w} \to \mathbb R^{c_\textit{in}\times \ceil{h/s_h}\times \ceil{w/s_w}}$ with strides $(s_h,s_w)$ is defined coordinate-wise as
\begin{align}
    [\phi_{K,(s_h,s_w)} (x)]_{\sigma \mu \nu} = 
    \sum_{r=1}^{c_\textit{in}}
    \sum_{i=\lfloor-\frac{k_h-1}{2}\rfloor}^{\lfloor\frac{k_h-1}{2}\rfloor}
    \sum_{j=\lfloor-\frac{k_w-1}{2}\rfloor}^{\lfloor\frac{k_w-1}{2}\rfloor} 
    & K_{\sigma r i j} \, x_{r, 1+s_h(\mu-1)+i,1+s_w(\nu-1)+j} \label{eqn:appendix:cond1ddef}\\[-5pt]
    & \qquad\qquad \cdot \mathbbm{1}_{1+s_h(\mu-1)+i \in [1,h]}\mathbbm{1}_{1+s_w(\nu-1)+j \in [1,w]} \nonumber
\end{align}
\vskip2ex
\begin{theorem}[Single-layer covering number bound -- General case]
    \label{theorem:single_layer}
    Let $b>0$. Define the class of $(s_h,s_w)$-strided 2D convolutions parametrized by tensors $K \in \mathbb R^{c_\textit{out}\times c_\textit{in} \times k_h \times k_w}$ with $W = {c_\textit{out} c_\textit{in} k_h k_w}$ parameters and ($2,1$) group norm $\norm{K}_{2,1}\le b$ as
    $$\mathcal{F} = \{\phi_{K,(s_h,s_w)}~|~ K \in \mathbb R^{c_\textit{out} \times c_\textit{in} \times k_h \times k_w}, \norm{K}_{2,1}\le b\}\enspace.$$
    Then, for any 
    $X=(x_1,\dots, x_n) \in \mathbb R^{n \times c_\textit{in}\times h \times w}$ and $\epsilon>0$, the covering number $\mathcal{N}(\mathcal{F}, \epsilon, \norm{\cdot}_X)$ satisfies
    \begin{equation}
        \label{theorem:single_layer:eq_norms}
        \log \mathcal{N}(\mathcal{F}, \epsilon, \norm{\cdot}_X) 
        \le
        \ceil{\frac{\norm{X}^2 b^2}{\epsilon^2}} \log(2W)
    \end{equation}
    and
    \begin{equation}
        \label{theorem:single_layer:eq_params}
        \log\mathcal N(\mathcal F, \epsilon, \norm{\cdot}_X )
        \le
        (2 W-1) \log\left(1+
        \ceil{
            \frac{\norm{X}^2 b^2}{\epsilon^2}
        }
        \right)
        \enspace.
    \end{equation}
\end{theorem}

\vspace{0.2cm}
\begin{remark}
    Recall that convolution with kernel size 1 and input size 1 is a linear map on the input channels, determined by the matrix $M_K = K_{\cdot\cdot11}$. In this situation, the convolutional layer reduces to a fully-connected layer and our first bound reduces to \cite[Lemma 3.2]{Bartlett17a}.
\end{remark}
\begin{proof}
    The proof is quite similar to the one of the special case in \cref{theorem:single_layer_simple}. Recall, that the convolution operation is defined coordinate-wise in \cref{eqn:appendix:cond1ddef}.    
    Using identities of the form $\mathbbm{1}_{t \in \set{1,\dots, n}} = \sum_{i=1}^n \mathbbm{1}_{i = t}$, we write
    \vskip0.5ex
    \begin{align*}
        & [\phi_{K,(s_h,s_w)} (x)]_{\sigma \mu \nu} = \\ 
        & \sum_{\alpha=1}^{c_\textit{in}}
        \sum_{\beta=1}^{h}
        \sum_{\gamma=1}^{w}
        x_{\alpha \beta \gamma}\!
            \left(\!
            \sum_{p=1}^{c_\textit{out}}\!
            \sum_{r=1}^{c_\textit{in}}\!
            \sum_{i=\lfloor-\frac{k_h-1}{2}\rfloor}^{\lfloor\frac{k_h-1}{2}\rfloor}
            \sum_{j=\lfloor-\frac{k_w-1}{2}\rfloor}^{\lfloor\frac{k_w-1}{2}\rfloor}\!\!\!\!\!
            K_{p r i j} 
            \underbrace{
                \mathbbm{1}_{p=\sigma}
                \mathbbm{1}_{\alpha=r}
                \mathbbm{1}_{\beta = 1+s_h(\mu-1)+i}
                \mathbbm{1}_{\gamma= 1+s_w(\nu-1)+j}
            }_{=[M^{(p,r,i,j)}]^{\alpha \beta \gamma}_{\sigma \mu \nu}}\!
            \right)
    \end{align*}
    and condense this into 
    \begin{align*}
        [\phi_{K,(s_h,s_w)} (x)]_{\sigma \mu \nu} & =  
        \sum_{\alpha=1}^{c_\textit{in}}
        \sum_{\beta=1}^{h}
        \sum_{\gamma=1}^{w}
        x_{\alpha \beta \gamma}
        \underbrace{
        \left(
            \sum_{p=1}^{c_\textit{out}}
            \sum_{r=1}^{c_\textit{in}}
            \sum_{i=\lfloor-\frac{k_h-1}{2}\rfloor}^{\lfloor\frac{k_h-1}{2}\rfloor}
            \sum_{j=\lfloor-\frac{k_w-1}{2}\rfloor}^{\lfloor\frac{k_w-1}{2}\rfloor}
            K_{p r i j}
            [M^{(p,r,i,j)}]^{\alpha \beta \gamma}_{\sigma \mu \nu}
        \right)
        }_{=M^{\alpha \beta \gamma}_{\sigma \mu \nu}}
        \enspace
        \\
        &= 
        \sum_{\alpha=1}^{c_\textit{in}}
        \sum_{\beta=1}^{h}
        \sum_{\gamma=1}^{w}
        x_{\alpha \beta \gamma}
        M^{\alpha \beta \gamma}_{\sigma \mu \nu}
        \enspace.
    \end{align*}

    Thus, the convolution is a multilinear map $\mathbb R^{c_\textit{in} \times h \times w} \to \mathbb R^{c_\textit{in}\times \ceil{h/s_h}\times \ceil{w/s_w}}$ parametrized by 
    \begin{align*}
    M & = \sum_{p=1}^{c_\textit{out}}
            \sum_{r=1}^{c_\textit{in}}
            \sum_{i=\lfloor-\frac{k_h-1}{2}\rfloor}^{\lfloor\frac{k_h-1}{2}\rfloor}
            \sum_{j=\lfloor-\frac{k_w-1}{2}\rfloor}^{\lfloor\frac{k_w-1}{2}\rfloor}
            K_{p r i j}
            M^{(p,r,i,j)}\\
            & =
            \sum_{p=1}^{c_\textit{out}}
            \sum_{r=1}^{c_\textit{in}}
            \sum_{i=\lfloor-\frac{k_h-1}{2}\rfloor}^{\lfloor\frac{k_h-1}{2}\rfloor}
            \sum_{j=\lfloor-\frac{k_w-1}{2}\rfloor}^{\lfloor\frac{k_w-1}{2}\rfloor}
            \left(
                K_{prij}
                \frac{\norm{X_{\cdot r\cdot\cdot}}}{\norm{X} b}
            \right)
            \left(
                \frac{\norm{X} b}{\norm{X_{\cdot r\cdot\cdot}}}
                M^{(p,r,i,j)}
            \right)
    \enspace.
    \end{align*}
    Since 
    \begin{align*}
        \norm{K_{p r i j}
                \frac{\norm{X_{\cdot r\cdot\cdot}}}{\norm{X} b}}_1
        &=
        \frac{1}{\norm{X} b}
        \sum_{p=1}^{c_\textit{out}}
        \sum_{i=\lfloor-\frac{k_h-1}{2}\rfloor}^{\lfloor\frac{k_h-1}{2}\rfloor}
        \sum_{j=\lfloor-\frac{k_w-1}{2}\rfloor}^{\lfloor\frac{k_w-1}{2}\rfloor}
        \left(
            \sum_{r=1}^{c_\textit{in}} |K_{p r i j}| \norm{X_{\cdot r\cdot\cdot}}
        \right)
        \\
        &\le 
        \frac{1}{\norm{X} b}
        \sum_{p=1}^{c_\textit{out}}
        \sum_{i=\lfloor-\frac{k_h-1}{2}\rfloor}^{\lfloor\frac{k_h-1}{2}\rfloor}
        \sum_{j=\lfloor-\frac{k_w-1}{2}\rfloor}^{\lfloor\frac{k_w-1}{2}\rfloor}
        \left(
            \sum_{r=1}^{c_\textit{in}} |K_{p r i j}|^2
        \right)^{1/2}
        \underbrace{
            \left(
                \sum_{r=1}^{c_\textit{in}} \norm{X_{\cdot r\cdot\cdot}}^2
            \right)^{1/2}
        }_{= \norm{X}}
        \\
        &=
        \frac{1}{b}
        \sum_{p=1}^{c_\textit{out}}
        \sum_{i=\lfloor-\frac{k_h-1}{2}\rfloor}^{\lfloor\frac{k_h-1}{2}\rfloor}
        \sum_{j=\lfloor-\frac{k_w-1}{2}\rfloor}^{\lfloor\frac{k_w-1}{2}\rfloor}
        \left(
            \sum_{r=1}^{c_\textit{in}} |K_{p r i j}|^2
        \right)^{1/2}
        \\
        &=
        \frac{ \norm{K}_{2,1}} b \stackrel{\text{(by assumption)}}{\le} 1
        \enspace,
    \end{align*}
    we can instantiate Maurey's sparsification lemma (\cref{appendix:maurey_lemma}) on the Hilbert space
    $(\mathcal{F}|_X, \norm{\cdot}_X)$ for 
    \begin{align*}
        \set{V_1, \dots, V_{2W}} =
        \Bigg\{ \pm\frac{\norm{X} b}{\norm{X_{\cdot r\cdot\cdot}}}
        M^{(p,r,i,j)} ~\Bigg|~ 
        & p \in \set{1,\dots,c_\textit{out}},
        r \in \set{1\dots,c_\textit{in}},
        \\
        & i \in \left\{ \Bigl\lfloor-\frac{k_h-1}{2}\Bigr\rfloor,\dots, \Bigl\lfloor\frac{k_h-1}{2}\Bigr\rfloor\right\},\\
        & j \in \left\{ \Bigl\lfloor- \frac{k_w-1}{2}\Bigr\rfloor \dots,\Bigl\lfloor\frac{k_w-1}{2}\Bigr\rfloor\right\}
        \Bigg\}\enspace.
    \end{align*}
    
    As a consequence, for any convolutional kernel $K \in \mathbb R^{c_\textit{out} \times c_\textit{in} \times k_h \times k_w}$ and any $m\in \mathbb N$, there exist $(m_1, \dots, m_{2W})$ with $\sum_{i=1}^{2W} m_i = m$ such that 
    \begin{equation*}
        \norm{
            \phi_{K,(s_h,s_w))} - \frac{1}{m} \sum_{i=1}^{2W} m_i V_i
        }_X^2
        \le \frac{1}{m}   \max_{i} \norm{V_i}_X^2
        \enspace.
    \end{equation*}

    Thus, for fixed $\epsilon>0$, if we choose $m\in \mathbb N$ such that 
    $\frac{1}{m}  \max_{i} \norm{V_i}_X^2 \le \epsilon^2$, then 
    the solutions in $m_i$ of $\sum_{i=1}^{2W} m_i = m$, 
    define an $\epsilon$-cover 
    \begin{equation}
        \label{theorem:single_layer:eq_cover}
        \set{\frac{1}{m} \sum_{i=1}^{2W} m_i V_i|~m_i \in \mathbb N_{\ge 0},~ \sum_{i=1}^{2W} m_i = m} \subset \mathcal F
        \enspace.
    \end{equation}
    As the number of non-negative $2W$-tuples which add up to $m$, denoted as $N(m,W)$, is equal to
    $$N(m,W)= \binom{m + 2W -1}{2W-1}\enspace,$$
    this means that $\mathcal F$ has an $\epsilon$-cover of cardinality at most 
    $N(m,W)$; thus,   
    $\mathcal N(\mathcal F, \epsilon, \norm{\cdot}_X ) \le N(m,W)$.

    In order to compute the norms $\norm{V_i}_X^2$, we use that
    for all $(p,r,i,j) \in [c_\textit{out}]\times[c_\textit{in}]\times[k_h] \times [k_w]$ and for all $x\in \mathbb R^{c_\textit{in}\times h \times w}$, it holds that
    \begin{align*}
        \norm{M^{(p,r,i,j)} x}^2
        &=
        \sum_{\sigma,\mu,\nu}
        [M^{(p,r,i,j)}x]_{\sigma \mu \nu}^2
        \\
        &=
        \sum_{\sigma,\mu,\nu}
        \left(
            \sum_{\alpha=1}^{c_\textit{in}}
            \sum_{\beta=1}^{h}
            \sum_{\gamma=1}^{w}
            [M^{(p,r,i,j)}]^{\alpha \beta \gamma}_{\sigma \mu \nu}
            x_{\alpha \beta \gamma}
        \right)^2     
        \\
        &=
        \sum_{\sigma,\mu,\nu}
        \left(
            \sum_{\alpha=1}^{c_\textit{in}}
            \sum_{\beta=1}^{h}
            \sum_{\gamma=1}^{w}
            \mathbbm{1}_{p=\sigma}
            \mathbbm{1}_{\alpha=r}
            \mathbbm{1}_{\beta = 1+s_h(\mu-1)+i}
            \mathbbm{1}_{\gamma= 1+s_w(\nu-1)+j}
            x_{\alpha \beta \gamma}
        \right)^2
        \\
        &=
        \sum_{\sigma,\mu,\nu}
        \left(
            \mathbbm{1}_{p=\sigma}
            x_{r, 1+s_h(\mu-1)+i, 1+s_w(\nu-1)+j}
        \right)^2  
        \\
        &=
        \sum_{\mu=1}^{\ceil{h/s_h}}
        \sum_{\nu=1}^{\ceil{w/s_w}}
        \left(
            x_{r, 1+s_h(\mu-1)+i, 1+s_w(\nu-1)+j}
        \right)^2  
        \\
        &=
        \sum_{\beta=1}^{h}
        \sum_{\gamma=1}^{w}
        \left(
            x_{r\beta\gamma}
        \right)^2  
        \mathbbm{1}_{\beta  \equiv (1+i)\!\!\!\! \mod s_h}
        \mathbbm{1}_{\gamma \equiv (1+j)\!\!\!\! \mod s_w}
        \\
        &\le
        \sum_{\beta=1}^{h}
        \sum_{\gamma=1}^{w}
        \left(
            x_{r\beta\gamma}
        \right)^2
        = \norm{x_{r\cdot\cdot}}^2
        \enspace.
    \end{align*}
    Thus, for any $t\in \set{1,\dots,2W}$,
    \begin{align*}
        \norm{V_t}_X^2 
        = \norm{\pm\frac{\norm{X} b}{\norm{X_{\cdot r\cdot\cdot}}}
        M^{(p_t,r_t,i_t,j_t)}}_X^2 
        & = 
        \frac{\norm{X}^2 b^2}{\norm{X_{\cdot r\cdot\cdot}}^2}
        \sum_{k=1}^n \norm{M^{(p_t,r_t,i_t,j_t)} x_k}^2 \\
        & \le
        \frac{\norm{X}^2 b^2}{\norm{X_{\cdot r\cdot\cdot}}^2}
        \sum_{k=1}^n \norm{X_{kr\cdot\cdot}}^2
        =
        \norm{X}^2 b^2
        \enspace,
    \end{align*}
    and we can choose $\mathbb N \ni m = \ceil{\frac{b^2 \norm{X}^2}{\epsilon^2} }$ to get an $\epsilon$-cover of $\mathcal F$ via \cref{theorem:single_layer:eq_cover}.

    The theorem then follows from two particular bounds on $N(m,W) = \binom{m + 2W -1}{2W-1}$.
    These are
    $$\binom{m + 2W -1}{2d-1} \le (2W)^m\enspace,$$
    which implies  \cref{theorem:single_layer:eq_norms}, and 
    $$\binom{m + 2W -1}{2W-1} \le (1+m)^{2W-1}\enspace,$$ 
    which implies  \cref{theorem:single_layer:eq_params}; see~\cref{theorem:binomial_coefficient} for details.
\end{proof}
\begin{remark}
    \label{rem:strides}
    In the proof, we bound 
    $$
    \sum_{\beta=1}^{h}
    \sum_{\gamma=1}^{w}
    \left(
        x_{r\beta\gamma}
    \right)^2  
    \mathbbm{1}_{\beta  \equiv (1+i)\!\!\!\! \mod s_h}
    \mathbbm{1}_{\gamma \equiv (1+j)\!\!\!\! \mod s_w}
    \le
    \sum_{\beta=1}^{h}
    \sum_{\gamma=1}^{w}
    \left(
        x_{r\beta\gamma}
    \right)^2\enspace.
    $$
    Under additional assumptions on the data $X$ this result might be improved as, on average, one expects 
    $$
    \sum_{\beta=1}^{h}
    \sum_{\gamma=1}^{w}
    \left(
        x_{r\beta\gamma}
    \right)^2  
    \mathbbm{1}_{\beta  \equiv (1+i)\!\!\!\! \mod s_h}
    \mathbbm{1}_{\gamma \equiv (1+j)\!\!\!\! \mod s_w}
    \le
    \frac{1}{s_h s_w}
    \sum_{\beta=1}^{h}
    \sum_{\gamma=1}^{w}
    \left(
        x_{r\beta\gamma}
    \right)^2\enspace,
    $$ 
    which would reduce the $\frac{\norm{X}^2 b^2}{\epsilon^2}$ terms in the bounds by the factor $1/(s_h s_w)$.
\end{remark}

\subsection{Whole-network covering number bounds (general form)}
\label{appendix:subsection:whole_network_general}
In order to prove covering number bounds for residual networks, we utilize the following basic observation: a residual network is a composition of residual blocks and each residual block corresponds to addition of two (compositions of) functions on the same input (one of them is typically the identity function). Thus, if we know the covering numbers of compositions and additions, we can derive whole-network covering number bounds in an inductive way. 

\emph{Importantly, the derived covering bounds hold for a broad class of network architectures, including the special cases of non-residual and residual networks.}

\subsubsection{Covering number bounds for compositions and summations}
Given normed spaces  $(\mathcal X_i, \norm{\cdot}_{\mathcal X_i})$
and function spaces $\mathcal F_i$ and $\mathcal G_i$ of functions $\mathcal X_i \to \mathcal X_{i+1}$,
we present covering number bounds for the following derived function spaces:
\begin{align}
    \Comp(\mathcal F_1, \dots, \mathcal F_L) 
    &=
    \set{ f_L\circ \dots \circ f_1|~ f_i\in \mathcal F_i}
    \\
    \Sum(\mathcal F_i, \mathcal G_i)
    &=
    \set{ f_i + g_i|~ f_i\in \mathcal F_i, g_i \in \mathcal G_i} 
\end{align}

\vskip1ex
\begin{lemma}[Compositions]
    \label{theorem:compositions}
    For $i \in \{1,2,3\}$, let $(\mathcal X_i, \norm{\cdot}_{\mathcal X_i})$ be normed spaces and let $\mathcal F_i$ be classes of functions $\mathcal X_i\to \mathcal X_{i+1}$ with $\Lip(\mathcal F_i) < \infty$.
    Then, for any $\epsilon_1, \epsilon_2 >0$ and any
    $X = \tuple{x_1, \dots, x_n}\in \mathcal X_1^n$, the covering number of the class $\Comp(\mathcal F_1, \mathcal F_2)$ is bounded by
    \begin{equation}
        \label{theorem:compositions:eq1}
        \mathcal N\left(\Comp(\mathcal F_1, \mathcal F_2), \Lip(\mathcal F_2) \epsilon_1 + \epsilon_2, \norm{\cdot}_X\right)
        \le
        \mathcal N(\mathcal F_1, \epsilon_1, \norm{\cdot}_{X})\!
        \left(
            \sup_{f\in \mathcal F_1} \mathcal N\left(\mathcal F_2, \epsilon_2, \norm{\cdot}_{f(X)}\right)\!\!
        \right)
    \end{equation}
    If $\mathcal F_2 = \set{f_2}$ is a singleton, then
    \begin{equation}
        \mathcal N(\Comp(\mathcal F_1, \mathcal F_2), \Lip(\mathcal F_2) \epsilon_1, \norm{\cdot}_X)
        \le
        \mathcal N(\mathcal F_1, \epsilon_1, \norm{\cdot}_{X})
        \enspace.
    \end{equation}
\end{lemma}
\begin{remark}
    There is an analogous result which holds for external covering numbers, i.e.,
    \begin{equation}
        \begin{split}
        \mathcal N^{\text{ext}}\left(\Comp(\mathcal F_1, \mathcal F_2), \Lip(\mathcal F_2) \epsilon_1 + \epsilon_2, \norm{\cdot}_X\right)
        & \\
        &  \hspace{-2cm}\le \mathcal N^{\text{ext}}(\mathcal F_1, \epsilon_1, \norm{\cdot}_{X})
        \left(
            \sup_{f:\mathcal X_1 \to \mathcal X_2} \mathcal N^{\text{ext}}\left(\mathcal F_2, \epsilon_2, \norm{\cdot}_{f(X)}\right)
        \right)
        \enspace.
        \end{split}
    \end{equation}
    Notably, in this case, the supremum is taken over all $f\colon\mathcal X_1 \to \mathcal X_2$. However, this form is unusable for deriving the whole-network covering number bounds in \cref{appendix:subsection:multi_layer_specific} as we want to handle the supremum via an assumption on the Lipschitz constant of the layer.
\end{remark}
\begin{proof}
    Fix $\epsilon_1, \epsilon_2>0$. 
    Let~$\mathcal U_{\mathcal F_1}\subset \mathcal F_1$ be a minimal $\epsilon_1$-cover of $(\mathcal F_1, \norm{\cdot}_X)$, i.e.,
    $\card(\mathcal U_{\mathcal F_1}) = \mathcal N(\mathcal F_1, \epsilon_1, \norm{\cdot}_X)$. 
    For any covering element $v\in \mathcal U_{\mathcal F_1}$, let $\mathcal U_{\mathcal F_2}(v) \subset \mathcal F_2$ be a minimal $\epsilon_2$-cover of $(\mathcal F_2, \norm{\cdot}_{v(X)})$, i.e., 
    $\card(\mathcal U_{\mathcal F_2}(v)) = \mathcal N(\mathcal F_2, \epsilon_2, \norm{\cdot}_{v(X)})\enspace.$

    Denote $c_2 = \Lip(\mathcal F_2)$.
    We will show that
    \begin{equation*}
        \mathcal U_{\Comp(\mathcal F_1,\mathcal F_2)}
        =
        \set{ w^{v} \circ v | ~ v\in \mathcal U_{\mathcal F_1},~ w^{v} \in \mathcal U_{\mathcal F_2}(v)} \subset \Comp(\mathcal F_1, \mathcal F_2)
    \end{equation*}
    defines an $(\epsilon_1 c_2 + \epsilon_2)$-cover of $(\Comp(\mathcal F_1, \mathcal F_2), \norm{\cdot}_X)$,
    i.e., for any $f_1 \in \mathcal F_1$ and any $f_2 \in \mathcal F_2$, there exist $v\in \mathcal U_{\mathcal F_1}$ and $w^{v}\in \mathcal U_{\mathcal F_2}(v)$ such that
    \begin{equation*}
        \norm{
            f_2 \circ f_1 - w^{v} \circ v
        }_X
        \le
        c_2 \epsilon_1 + \epsilon_2
        \enspace.
    \end{equation*}

    Indeed, since $ \mathcal U_{\mathcal F_1}$ is an $\epsilon_1$-cover of $(\mathcal F_1, \norm{\cdot}_X)$, we can choose $v\in \mathcal U_{\mathcal F_1}$ such that $\norm{f_1 - v}_X \le \epsilon_1$,
    and, since
    $\mathcal U_{\mathcal F_2}(v)$ is an $\epsilon_2$-cover of $(\mathcal F_2, \norm{\cdot}_{v(X)})$, we can choose $w^{v} \in \mathcal U_{\mathcal F_2}(v)$ such that $\norm{f_2 - w^{v}}_{v(X)} \le \epsilon_2$.
    Thus, 
    \begin{align*}
        \norm{
            f_2 \circ f_1 - w^{v} \circ v
        }_X
        &=
        \norm{
            (f_2 \circ f_1 - f_2 \circ v)
            +
            (f_2 \circ v - w^{v} \circ v)
        }_X
        \\
        &\le
        \norm{f_2 \circ f_1 - f_2 \circ v}_X
            +
        \|f_2 \circ v - w^{v} \circ v\|_X
        \\
        & 
        \leq
        \Lip(f_2) \norm{f_1 - v}_X
            +
        \|f_2- w^{v}\|_{v(X)}
        \\
        &\le
        c_2 \epsilon_1 + \epsilon_2
        \enspace,
    \end{align*}
    where the second inequality follows from Eq.~\labelcref{eq:empirical_norm_properties}. Therefore,
    \begin{align*}
        \mathcal N(\Comp(\mathcal F_1, \mathcal F_2), c_2 \epsilon_1 + \epsilon_2, \norm{\cdot}_X)
        &\le
        \card\left(\mathcal U_{\Comp(\mathcal F_1,\mathcal F_2)}\right) \\
        & =
        \card\left(\set{ w^{v} \circ v | ~ v\in \mathcal U_{\mathcal F_1},~ w^{v} \in \mathcal U_{\mathcal F_2}(v)}\right)
        \\
        &=
        \sum_{v \in \mathcal U_{\mathcal F_1}} \card\left(\mathcal U_{\mathcal F_2}(v)\right)
        \\
        &\le
        \left(\sup_{v \in \mathcal U_{\mathcal F_1}} \card\left(\mathcal U_{\mathcal F_2}(v)\right) \right)
        \left(\sum_{v \in \mathcal U_{\mathcal F_1}} 1 \right)
        \\
        &=
        \left(\sup_{v \in \mathcal U_{\mathcal F_1}} \card\left(\mathcal U_{\mathcal F_2}(v)\right) \right)
        \card(\mathcal U_{\mathcal F_1})
        \\
        &\stackrel{(\star)}{\le}
        \left(
            \sup_{f\in \mathcal F_1}\card\left(\mathcal U_{\mathcal F_2}(f)\right)
        \right)
        \card\left(\mathcal U_{\mathcal F_1}\right)
        \\
        &\le
        \left(
            \sup_{f\in \mathcal F_1} \mathcal N\left(\mathcal F_2, \epsilon_2, \norm{\cdot}_{f(X)}\right)
        \right)
        \mathcal N\left(\mathcal F_1, \epsilon_1, \norm{\cdot}_{X}\right)
        \enspace.
    \end{align*} 
    For $(\star)$, we used that $\mathcal U_{\mathcal F_1} \subset \mathcal F_1$ is an internal cover.
    
    The special case of $\mathcal F_2=\set{f_2}$ being a singleton is obvious, as we can choose $\mathcal U_{\mathcal F_2}(v) = \set{f_2}$ for every $v \in \mathcal U_{\mathcal F_1}$.
    Then, for every $f_2 \in \mathcal F_2$ and every $w^v \in \mathcal U_{\mathcal F_2}(v)$, it holds that $\card (\mathcal U_{\mathcal F_2}(v)) = 1$  and $\norm{f_2-w_j^v}_{v(X)}= 0$.
\end{proof}

\vskip1ex
\begin{lemma}[Summations]
    \label{theorem:summations}
    Let $(\mathcal X, \norm{\cdot}_{\mathcal X})$ and $(\mathcal Y, \norm{\cdot}_{\mathcal Y})$ be normed spaces and let $\mathcal F, \mathcal G$ be classes of functions $\mathcal{X}\to\mathcal{Y}$.
    Then, for each $\epsilon_{\mathcal F}, \epsilon_{\mathcal G} >0$ and each $X = \tuple{x_i, \dots, x_n}\in\mathcal X^n$, the covering number of the class $\Sum(\mathcal F, \mathcal G)$ is bounded by
    \begin{equation}
        \mathcal N(\Sum(\mathcal F, \mathcal G), \epsilon_{\mathcal F} + \epsilon_{\mathcal G}, \norm{\cdot}_X)
        \le
        \mathcal N(\mathcal F, \epsilon_{\mathcal F}, \norm{\cdot}_{X})
        \mathcal N(\mathcal G, \epsilon_{\mathcal G}, \norm{\cdot}_{X})
        \enspace.
    \end{equation}
    If $\mathcal G = \set{g}$ is a singleton, then
    \begin{equation}
        \mathcal N(\Sum(\mathcal F, \mathcal G), \epsilon_{\mathcal F}, \norm{\cdot}_X)
        =
        \mathcal N(\mathcal F, \epsilon_{\mathcal F}, \norm{\cdot}_{X})
        \enspace.
    \end{equation}
\end{lemma}

\begin{proof}
    Fix $\epsilon_{\mathcal F}, \epsilon_{\mathcal G}>0$. Let $\mathcal U_{\mathcal F} \subset \mathcal F$ be a minimal $\epsilon_{\mathcal F}$-cover of $(\mathcal F, \norm{\cdot}_X)$ and let 
    $\mathcal U_{\mathcal G}\subset \mathcal G$ be a minimal $\epsilon_{\mathcal G}$-cover of $(\mathcal G, \norm{\cdot}_X)$, i.e.,
    $\mathcal N(\mathcal F, \epsilon_{\mathcal F}, \norm{\cdot}_X) = \card(\mathcal U_{\mathcal F})$ and $\mathcal N(\mathcal G, \epsilon_{\mathcal G}, \norm{\cdot}_X) = \card(\mathcal U_{\mathcal G})$.

    We will show that
    \begin{equation*}
        \mathcal U_{\Sum(\mathcal F, \mathcal G)}
        =
        \set{ v + w | ~ v\in \mathcal U_{\mathcal F},~ w \in U_{\mathcal G}} \subset \Sum(\mathcal F, \mathcal G)
    \end{equation*}
    defines an $(\epsilon_{\mathcal F} + \epsilon_{\mathcal G})$-cover of $(\Sum(\mathcal F, \mathcal G), \norm{\cdot}_X)$,
    i.e., for every $f \in \mathcal F$ and every $g \in \mathcal G$, there exist $v \in \mathcal U_{\mathcal F}$ and $w \in \mathcal U_{\mathcal G}$ such that
    \begin{equation*}
        \norm{
            (f + g) - (v +w)
        }_X
        \le
        \epsilon_{\mathcal F} + \epsilon_{\mathcal G}
        \enspace.
    \end{equation*}
    Indeed, since $\mathcal U_{\mathcal F}$ is an $\epsilon_{\mathcal F}$-cover of $(\mathcal F, \norm{\cdot}_X)$, we can choose $v \in \mathcal U_{\mathcal F}$ such that $\norm{f - v}_X \le \epsilon_{\mathcal F}$
    and since
    $\mathcal U_{\mathcal G}$ is an $\epsilon_{\mathcal G}$-cover of $(\mathcal G, \norm{\cdot}_{X})$, we can choose $w \in \mathcal U_{\mathcal G}$ such that $\norm{g - w}_{X} \le \epsilon_{\mathcal G}$.
    Then, 
    \begin{align*}
        \norm{
            (f + g) - (v +w)
        }_X
        &=
        \norm{
            (f - v) + (g - w)
        }_X
        \\
        &\le
        \norm{f - v}_X
        +
        \norm{g - w}_X
        \\
        &\le  
        \epsilon_{\mathcal F} + \epsilon_{\mathcal G}
        \enspace.
    \end{align*}
    Therefore, we have 
    \begin{align*}
        \mathcal N(\Sum(\mathcal F, \mathcal G), \epsilon_{\mathcal F} + \epsilon_{\mathcal G}, \norm{\cdot}_X)
        &\le
        \card\left(\mathcal U_{\Sum(\mathcal F, \mathcal G)}\right) \\
        & =
        \card\left(\set{ v + w | ~ v\in \mathcal U_{\mathcal F},~ w \in U_{\mathcal G}}\right)
        \\
        &\le 
        \card( \mathcal U_{\mathcal F})
        \card( \mathcal U_{\mathcal G})
        \\
        &=
        \mathcal N(\mathcal F, \epsilon_{\mathcal F}, \norm{\cdot}_{X})
        \mathcal N(\mathcal G, \epsilon_{\mathcal G}, \norm{\cdot}_{X})
        \enspace.
    \end{align*}    
    \vskip1ex
    The special case of $\mathcal G = \set{g}$ being a singleton is obvious, as we can choose $\mathcal U_{\mathcal G} = \set{g}$. Then $\mathcal U_{\Sum(\mathcal F, \mathcal G)} = \set{ v + g | ~ v\in \mathcal U_{\mathcal F}}$ is a cover of $\Sum(\mathcal F, \mathcal G)$ with cardinality $\card(\mathcal U_{\Sum(\mathcal F, \mathcal G)}) = \card(\mathcal U_{\mathcal F})$ and radius $\epsilon_{\mathcal F}$.
\end{proof}

Now that we know how to bound the covering numbers of compositions and summations, we can iteratively derive covering number bounds for all function classes obtained from these two operations.

\subsubsection{General strategy for bounding the covering numbers of complex classes}
Let $\mathcal F$ be a function class whose covering number is \emph{unknown} to us. If $\mathcal F$ can be built iteratively by compositions and summations of function classes with known covering number (bounds), then we can derive covering number bounds for $\mathcal F$ via the following strategy. In a first step, we identify the structure of $\mathcal F$, i.e., how it is built from compositions and summations. 
In a second step, starting with $\mathcal F$, we iteratively replace each function class by its simpler building blocks and the covering number of $\mathcal F$ by the respective bound.

To be more specific, we know by \cref{theorem:compositions} and \cref{theorem:summations} that for $\mathcal F = \Comp(\mathcal F_{a}, \mathcal F_{b})$, it holds that
\begin{equation*}
    \mathcal N(\mathcal F, \Lip({\mathcal F_{b}}) \epsilon_{\mathcal F_{a}} + \epsilon_{\mathcal F_{b}}, \norm{\cdot}_X)
    \le
    \mathcal N(\mathcal F_{a}, \epsilon_{\mathcal F_{a}}, \norm{\cdot}_{X})
    \left(
        \sup_{f\in \mathcal F_{a}} \mathcal N(\mathcal F_{b}, \epsilon_{\mathcal F_{b}}, \norm{\cdot}_{f(X)})
    \right)
    \enspace.
\end{equation*}
and for $\mathcal F = \Sum(\mathcal F_{a}, \mathcal F_{b})$, it holds that 
\begin{equation*}
    \mathcal N(\mathcal F, \epsilon_{\mathcal F_{a}} + \epsilon_{\mathcal F_{b}}, \norm{\cdot}_X)
    \le
    \mathcal N(\mathcal F_{a}, \epsilon_{\mathcal F_{a}}, \norm{\cdot}_{X})
    \mathcal N(\mathcal F_{b}, \epsilon_{\mathcal F_{b}}, \norm{\cdot}_{X})
    \enspace.
\end{equation*}

Now, if (for $x=a$ or $x=b$) some class $\mathcal F_{x}$ is of the form $\mathcal F_{x} = \Comp(\mathcal F_{xa}, \mathcal F_{xb})$ or 
$\mathcal F_{x} = \Sum(\mathcal F_{xa}, \mathcal F_{xb})$, we bound the right-hand side of the equations above by the same argument.
We repeat this procedure until the right-hand side contains only terms of known covering number bounds. For an illustration of this stepwise process, see 
\cref{appendix:figure_strategy}.

\begin{figure}[]
    \begin{center}
        \includegraphics[width=.8 \textwidth]{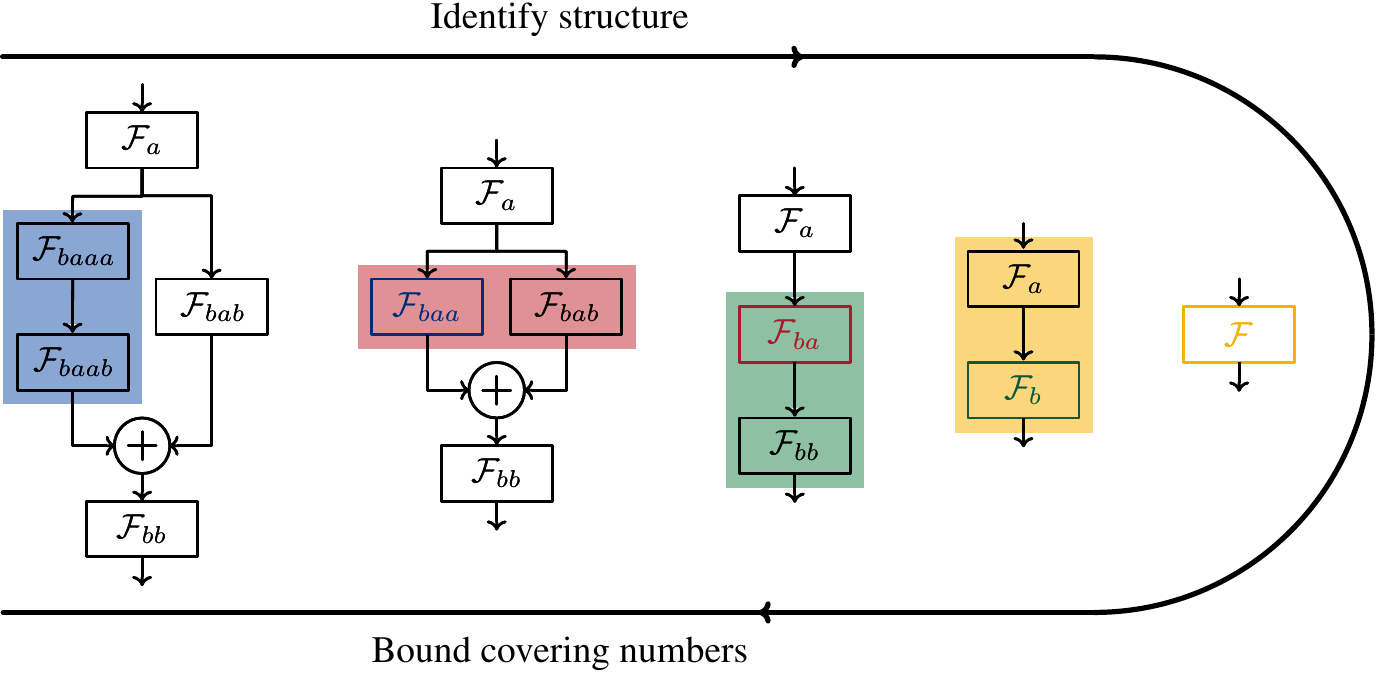}
    \end{center}
    \caption{Schematic illustration of how to obtain whole-network covering number bounds by first identifying a way to write the network via summations and compositions, and then iteratively applying the respective inequalities.
    The function classes are systematically denoted by words with characters $a$ and $b$. Starting with $\mathcal F$ at the very right, we always add a character when replacing a function class by its building blocks.
    \label{appendix:figure_strategy}}
\end{figure}

\subsubsection{Examples}
\begin{example}[Multi-composition]
    \label{theorem:composition_induction}
    Let $(\mathcal X_1, \norm{\cdot}_{\mathcal X_1}),\ldots,(\mathcal X_L, \norm{\cdot}_{\mathcal X_L})$ be normed spaces. Let $\mathcal F_i$ be classes of functions $\mathcal X_i\to \mathcal X_{i+1}$ with bounded Lipschitz constants, i.e., $\Lip(\mathcal F_i)< \infty$.
    Denote $\compto{\mathcal F_i} = \Comp(\mathcal F_{1}, \dots, \mathcal F_{i-1})$.
    Then, for any $\epsilon_i>0$ and any finite $X = \tuple{x_1, \dots, x_n}\in \mathcal X_1^n$, the covering number of the class $\mathcal F = \Comp(\mathcal F_1, \dots, \mathcal F_L)$ is bounded by
    \begin{equation}
        \label{theorem:composition_induction:eq1}
        \mathcal{N}\!\left(
            \mathcal F,
            \sum_{i=1}^L \left(\prod_{l=i+1}^{L}\! \Lip(\mathcal F_l)\right)
            \epsilon_i {\mathbbm{1}_{\card(\mathcal F_i)>1}},
            \norm{\cdot}_X
        \right)
        \le
        \prod_{i=1}^L
        \sup_{\substack{\psi_i \in \compto{\mathcal F_i}}}
        \mathcal N\left(\mathcal F_i, \epsilon_i, \norm{\cdot}_{\psi_i(X)}\right)^{\mathbbm{1}_{\card(\mathcal F_i)>1}}
        \enspace.
    \end{equation}
\end{example}

\begin{proof}
We have
    \begin{align*}
    & \mathcal N \left(
        \mathcal F,
        \sum_{i=1}^L \left(\prod_{l=i+1}^{L} \Lip(\mathcal F_l)\right)
        \epsilon_i {\mathbbm{1}_{\card(\mathcal F_i)>1}},
        \norm{\cdot}_X
    \right) \\
    & = \mathcal N
        \Biggl(
        \Comp(
            \compto{\mathcal F_{L}},
            \mathcal F_L                
        ),\\
        & \hspace{1.11cm} \Lip(\mathcal F_L)
        \left( 
            \sum_{i=1}^{L-1} \left(\prod_{l=i+1}^{L-1} \Lip(\mathcal F_l)\right)
            \epsilon_i {\mathbbm{1}_{\card(\mathcal F_i)>1}}
        \right)
        + 
        \epsilon_L {\mathbbm{1}_{\card(\mathcal F_L)>1}},
        \norm{\cdot}_X\Biggr)\\
        & \le
        \mathcal N
        \Biggl(
            \compto{\mathcal F_{L}}, \sum_{i=1}^{L-1} \left(\prod_{l=i+1}^{L-1} \Lip(\mathcal F_l)\right)
                \epsilon_i {\mathbbm{1}_{\card(\mathcal F_i)>1}},
            \norm{\cdot}_X
        \Biggr)
        \\
        & \hspace{1.11cm} 
        \cdot\left(
            \sup_{\psi_L \in \compto{\mathcal F_{L}} } 
            \mathcal N(
                \mathcal F_{L},
                \epsilon_L,
                \norm{\cdot}_{\psi_L(X)}
        )
        \right)^{\!\mathbbm{1}_{\card(\mathcal F_L)>1}}
        \\
        & \le 
        \dots 
        \\
        & \le
        \prod_{i=1}^L
        \sup_{\substack{\psi_i \in \compto{\mathcal F_i}}}
        \mathcal N\left(\mathcal F_i, \epsilon_i, \norm{\cdot}_{\psi_i(X)}\right)^{\mathbbm{1}_{\card(\mathcal F_i)>1}}
        \enspace.
\end{align*}
Here, we used \cref{theorem:compositions} with {$\epsilon_{\compto{\mathcal F_k}}= \sum_{i=1}^{k-1} \left(\prod_{l=i+1}^{k-1} \Lip(\mathcal F_l)\right)
\epsilon_i {\mathbbm{1}_{\card(\mathcal F_i)>1}}$}
\end{proof}

\begin{remark}
We want to point out that \cref{theorem:composition_induction} implies the whole-network covering bound as in \cite[Lemma A.7]{Bartlett17a}.
To see this, let $\mathcal X_i = \mathbb R^{d_i}$,
$X_i = (x_{i_1},\dots, x_{i_n}) \in \mathbb R^{d_i \times n}$ and let $\sigma_i$ be fixed $\rho_i$-Lipschitz functions.
Further, let $\mathcal A_i$ be sets of matrices $A\in \mathbb R^{d_{i+1}\times d_i}$.
Then, the maps
\begin{align*}
    \psi_i: (\mathcal F_i = \set{ \sigma_i \circ A |~ A \in \mathcal A_i}, \norm{\cdot}_{X_i})
    &\to
    (\mathbb R^{d_i\times n}, \norm{\cdot}_{l_2})
    \\
    \sigma_i \circ A &\mapsto \sigma_i(AX_i)
\end{align*}
define isometries, because
\begin{equation}
    \begin{split}
    \norm{\psi_i(\sigma_i \circ A)}_{l_2}^2 
    & = \norm{\sigma_i (AX_i)}_{l_2}^2 \\
    & = \sum_{k=1}^n \norm{\sigma_i(Ax_{i_k})}^2\\
    & = \sum_{k=1}^n \norm{(\sigma_i \circ A) (x_{i_k})}^2\\
    & = \norm{\sigma_i\circ A}_{X_i}^2
    \enspace.
    \end{split}
\end{equation}
Consequently,
\begin{align*}
    \mathcal N (\mathcal F_i, \rho_i \epsilon_i, \norm{\cdot}_{X_i}) 
    &=
    \mathcal N(\psi_i(\mathcal F_i), \rho_i \epsilon_i, \norm{\cdot}_{l_2}) \\
    & =
    \mathcal N(\set{\sigma_i(AX_i)|~A\in \mathcal A_i}, \rho_i \epsilon_i, \norm{\cdot}_{l_2})
    \\
    &\le
    \mathcal N(\set{AX_i|~A\in \mathcal A_i}, \epsilon_i, \norm{\cdot}_{l_2})
    \enspace,
\end{align*}

    which are the factors on the right hand side of \cite[Lemma A.7]{Bartlett17a}.
\end{remark}

\vskip1ex
\begin{example}[Addition block]
    \label{theorem:addition_block}
    Let $\mathcal F$ be the function class of addition blocks, i.e.,
    \begin{equation*}
        \mathcal F = \Sum (\mathcal G, \mathcal H)
        \enspace,
    \end{equation*}
    where $\mathcal G = \Comp(\mathcal G_1, \dots, \mathcal G_{L_{\mathcal G}})$ and $\mathcal H = \Comp(\mathcal H_1, \dots, \mathcal H_{L_{\mathcal H}})$.
    For brevity, we write 
    \begin{align*}
        \compto{\mathcal G_{i}} & = \Comp(\mathcal G_{1}, \dots, \mathcal G_{i-1})\enspace,\\
        \compto{\mathcal H_{i}} & = \Comp(\mathcal H_{1}, \dots, \mathcal H_{i-1})\enspace.
    \end{align*}

    The covering number of a block $\mathcal F = \Sum (\mathcal G, \mathcal H)$ is bounded by 
    \begin{equation}
        \begin{split}
        \label{theorem:addition_block:eq1}
        & \mathcal N
        \Biggl(\!
            \mathcal F,
            \sum_{i=1}^{L_{\mathcal G}}
            \left(\prod_{l=i+1}^{L_{\mathcal G}}\!\! \Lip(\mathcal G_l) \right)
            \!\epsilon_{\mathcal G_{i}} {\mathbbm{1}_{\card(\mathcal G_i)>1}}
            ~+ \\
            &\hspace{1.02cm}
            \sum_{i=1}^{L_{\mathcal H}} 
            \left(\prod_{l=i+1}^{L_{\mathcal H}}\!\! \Lip(\mathcal H_l) \right)
            \!\epsilon_{\mathcal H_{i}} {\mathbbm{1}_{\card(\mathcal H_i)>1}},
            \norm{\cdot}_{X}\!
        \Biggr) {\leq} \left(
            \prod_{i=1}^{{L_{\mathcal G}}}
            \sup_{\substack{\psi \in \compto{\mathcal G_{i}}}}
            \!\!\mathcal N\left(
                \mathcal G_{i},
                \epsilon_{\mathcal G_{i}},
                \norm{\cdot}_{\psi(X)}
            \right)
        \!\!\right)^{\!\!\mathbbm{1}_{\card(\mathcal G_i)>1}} \\
        & 
        \hspace{7.66cm}\left(
            \prod_{i=1}^{{L_{\mathcal H}}}
            \!\sup_{\substack{\psi \in \compto{\mathcal H_{i}}}}
            \!\!\!\mathcal N\!\left(
                \mathcal H_{i},
                \epsilon_{\mathcal H_{i}},
                \norm{\cdot}_{\psi(X)}
            \right)
        \!\!\right)^{\!\!\mathbbm{1}_{\card(\mathcal H_i)>1}}
        \end{split}
    \end{equation}
\end{example}
\begin{proof}
    From \cref{theorem:summations}, we know that
    \begin{equation*}
        \mathcal N(
            \mathcal F,
            \epsilon_{\mathcal G} + \epsilon_{\mathcal H},
            \norm{\cdot}_{X}
        )
        \le
        \mathcal N(
            \mathcal G,
            \epsilon_{\mathcal G},
            \norm{\cdot}_{X}
        )
        \,
        \mathcal N(
            \mathcal H,
            \epsilon_{\mathcal H},
            \norm{\cdot}_{X}
        )
    \end{equation*}
    holds for every $\epsilon_{\mathcal G}>0$ and $\epsilon_{\mathcal H}>0$.
    Choosing 
    $$\epsilon_{\mathcal G}= \sum_{i=1}^{L_{\mathcal G}}
    \left(\prod_{l=i+1}^{L_{\mathcal G}} \Lip(\mathcal G_l) \right)
    \epsilon_{\mathcal G_{i}} {\mathbbm{1}_{\card(\mathcal G_i)>1}} $$ 
    and 
    $$\epsilon_{\mathcal H} =  \sum_{i=1}^{L_{\mathcal H}}
    \left(\prod_{l=i+1}^{L_{\mathcal H}} \Lip(\mathcal H_l) \right)
    \epsilon_{\mathcal H_{i}} {\mathbbm{1}_{\card(\mathcal H_i)>1}}\enspace,$$ 
    and bounding each factor on the right-hand side via \cref{theorem:composition_induction} yields \cref{theorem:addition_block:eq1}.
\end{proof}

\begin{example}[Residual network]
    \label{theorem:residual_networks}
    In the setting of \cref{theorem:composition_induction}, let the function classes $\mathcal F_i$ be residual blocks
    \begin{equation*}
        \mathcal F_i = \Sum (\mathcal G_i, \mathcal H_i)
        \enspace,
    \end{equation*}
    where 
    \begin{align*}
        \mathcal G_{i} = \Comp(\mathcal G_{i1}, \dots, \mathcal G_{iL_{\mathcal G_i}}) \quad \text{and} \quad \mathcal H_{i} &= \Comp(\mathcal H_{i1}, \dots, \mathcal H_{iL_{\mathcal H_i}})
        \enspace.
    \end{align*}
    Assume, that $\Lip({\mathcal G}_{ij}), \Lip(\mathcal H_{ij}) < \infty$ and that input data $X$ is given.
    For brevity, we write 
    \begin{align*}
        \compto{\mathcal G_{ij}} & = \Comp(\mathcal G_{i1}, \dots, \mathcal G_{i,{j-1}}) \\
        \compto{\mathcal H_{ij}} & = \Comp(\mathcal H_{i1}, \dots, \mathcal H_{i,{j-1}}) \\
        \compto{\mathcal F_{i}}  & = \Comp(\mathcal F_{1}, \dots, \mathcal F_{i-1})\enspace.
    \end{align*}
    
    The covering number of the residual network, $\mathcal F = \Comp(\mathcal F_1,\dots, \mathcal F_L)$, is bounded by
    \begin{equation}
        \begin{split}
        \label{theorem:residual_networks:eq1}
        & \mathcal N\left(
            \mathcal F,
            \epsilon_{\mathcal F},
            \norm{\cdot}_X
        \right) \le\\
        &~~~~ 
        \left(
        \prod_{i=1}^L
        \prod_{j=1}^{L_i}
                \sup_{\substack{\psi_{ij} \in \\ \Comp( \compto{\mathcal F_{i}}, \compto{\mathcal G_{ij}})}}
                \mathcal N\left(
                    \mathcal G_{ij},
                    \epsilon_{\mathcal G_{ij}},
                    \norm{\cdot}_{\psi_{ij} (X)}
                \right)^{\mathbbm{1}_{\card(\mathcal G_{ij})>1}}
        \right)\\
        &~~~~
        \left(
        \prod_{i=1}^L
        \prod_{j=1}^{L_i}
                \sup_{\substack{\psi_{ij} \in \\ \Comp( \compto{\mathcal F_{i}}, \compto{\mathcal H_{ij}})}}
                \mathcal N\left(
                    \mathcal H_{ij},
                    \epsilon_{\mathcal H_{ij}},
                    \norm{\cdot}_{\psi_{ij} (X)}
                \right)^{\mathbbm{1}_{\card(\mathcal H_{ij})>1}}
        \right)\,,
        \end{split}
    \end{equation}
    where
    \begin{align*}
        \epsilon_{\mathcal F} = 
        \sum_{i=1}^L
           \left( \prod_{l=i+1}^L \Lip(\mathcal F_i)\right)\!
           \epsilon_{\mathcal F_i}
    \end{align*}
    with
    \begin{align*}
        \epsilon_{\mathcal F_i} =   
                \sum_{j=1}^{L_{\mathcal G_i}}
                \left(\! \prod_{k=j+1}^{L_{\mathcal G_i}} \Lip(\mathcal G_{ik})\right)
                \epsilon_{\mathcal G_{ij}} {\mathbbm{1}_{\card(\mathcal G_{ij})>1}}
            + 
                \sum_{j=1}^{L_{\mathcal H_i}}\!
                \left( \prod_{k=j+1}^{L_{\mathcal H_i}} \Lip(\mathcal H_{ik})\!\right)
                \epsilon_{\mathcal H_{ij}} {\mathbbm{1}_{\card(\mathcal H_{ij})>1}}
            \enspace.
    \end{align*}
\end{example}
\begin{proof}
    Assuming $\card(\mathcal F_{i})>1$, we apply \cref{theorem:composition_induction} to $\mathcal F = \Comp(\mathcal F_1,\dots, \mathcal F_L)$ to obtain
    \begin{equation*}
        \mathcal N(
            \mathcal F,            
            \sum_{i=1}^L
            \left( \prod_{l=i+1}^L \Lip(\mathcal F_i)\right)
            \epsilon_{\mathcal F_i} ,
            \norm{\cdot}_X
        )
        \le
        \prod_{i=1}^L
        \sup_{\substack{\psi_i \in \compto{\mathcal F_i}}}
        \mathcal N(\mathcal F_i, \epsilon_{\mathcal F_i}, \norm{\cdot}_{\psi_i(X)})
        \enspace.
    \end{equation*}
    Bounding the covering number of each block $\mathcal F_i$ via \cref{theorem:addition_block}
    yields \cref{theorem:residual_networks:eq1}.
\end{proof}

\subsubsection{Covering number bounds for concatenations}

The general approach to bounding covering numbers of function classes, obtained from linking simple function classes via summations and compositions, can be easily extended. As an example, we can incorporate \emph{concatenations}, as typically used in DenseNets \citep{Huang17a}, via the following lemma.

\vskip2ex
\begin{lemma}[Concatenations]
    \label{appendix:lemma:concatenation}
    Let $(\mathcal X, \norm{\cdot}_{\mathcal X})$ be a normed space and let  $(\mathcal Y, \norm{\cdot}_{\mathcal Y}) = (\mathbb R^{d_{\mathcal Y}}, \norm{\cdot}_{l_2})$ and $(\mathcal Z, \norm{\cdot}_{\mathcal Z}) = (\mathbb R^{d_{\mathcal Z}}, \norm{\cdot}_{l_2})$.
    Let $\mathcal F, \mathcal G$ be classes of functions $\mathcal X\to \mathcal Y$, resp.  $\mathcal X\to \mathcal Z$.
    Define the function class $\Cat(\mathcal F, \mathcal G)$ of concatenations
    $\mathcal X \to \mathcal Y \times \mathcal Z$
    as 
    \begin{equation}
        \Cat(\mathcal F, \mathcal G) 
        =
        \set{(f,g): x\mapsto (f(x),g(x))~|~ f\in \mathcal F,~ g\in \mathcal G}
        \enspace.
    \end{equation}
    If we equip $\mathcal Y \times \mathcal Z = \mathbb R^{d_{\mathcal Y} d_{\mathcal Z}}$ with the $l_2$ norm, then
    \begin{equation}
        \mathcal N(\Cat(\mathcal F, \mathcal G),
        \sqrt{\epsilon_{\mathcal F}^2 + \epsilon_{\mathcal G}^2},
        \norm{\cdot}_X)
        \le
        \mathcal N(\mathcal F, \epsilon_{\mathcal F}, \norm{\cdot}_{X})
        \mathcal N(\mathcal G, \epsilon_{\mathcal G}, \norm{\cdot}_{X})
        \enspace.
    \end{equation}
\end{lemma}
\begin{proof}
    Fix $\epsilon_{\mathcal F}, \epsilon_{\mathcal G}>0$. Let $\mathcal U_{\mathcal F} \subset \mathcal F$ be a minimal $\epsilon_{\mathcal F}$-cover of $(\mathcal F, \norm{\cdot}_X)$ and let 
    $\mathcal U_{\mathcal G}\subset \mathcal G$ be a minimal $\epsilon_{\mathcal G}$-cover of $(\mathcal G, \norm{\cdot}_X)$, i.e.,
    $\mathcal N(\mathcal F, \epsilon_{\mathcal F}, \norm{\cdot}_X) = \card(\mathcal U_{\mathcal F})$ and $\mathcal N(\mathcal G, \epsilon_{\mathcal G}, \norm{\cdot}_X) = \card(\mathcal U_{\mathcal G})$.

    We will show that
    \begin{equation*}
        \mathcal U_{\Cat(\mathcal F, \mathcal G)}
        :=
        \set{ (v, w) | ~ v\in \mathcal U_{\mathcal F},~ w \in U_{\mathcal G}} \subset \Cat(\mathcal F_1, \mathcal F_2)
    \end{equation*}
    defines an $\sqrt{\epsilon_{\mathcal F}^2 + \epsilon_{\mathcal G}^2}$-cover of $(\Cat(\mathcal F, \mathcal G), \norm{\cdot}_{l_2})$,
    i.e., for every $f \in \mathcal F$ and every $g \in \mathcal G$, there exist $v \in \mathcal U_{\mathcal F}$ and $w \in \mathcal U_{\mathcal G}$ such that
    \begin{equation*}
        \norm{
            (f,g) - (v,w)
        }_X
        \le
        \sqrt{\epsilon_{\mathcal F}^2 + \epsilon_{\mathcal G}^2}
        \enspace.
    \end{equation*}
    Indeed, since $\mathcal U_{\mathcal F}$ is an $\epsilon_{\mathcal F}$-cover of $(\mathcal F, \norm{\cdot}_X)$, we can choose $v$ such that $\norm{f - v}_X \le \epsilon_{\mathcal F}$
    and since
    $\mathcal U_{\mathcal G}$ is an $\epsilon_{\mathcal G}$-cover of $(\mathcal G, \norm{\cdot}_{X})$, we can choose $w$ such that $\norm{g - w}_{X} \le \epsilon_{\mathcal G}$.
    Then
    \begin{align*}
        \norm{
            (f,g) - (v,w)
        }_X^2
        &=
        \sum_{i=1}^n
        \norm{(f, g)(x_i) -(v, w)(x_i)}_{l_2}^2
        \\
        &=
        \sum_{i=1}^n
        \norm{(f-v, g-w)(x_i)}_{l_2}^2
        \\
        &=
        \sum_{i=1}^n
        \left(
        \norm{(f -v)(x_i)}_{l_2}^2 + \norm{(g - w)(x_i))}_{l_2}^2
        \right) \tag{Pythagorean thm.}
        \\
        &=
        \norm{f-v}_X^2 + \norm{g-w}_X^2
        \\
        &\le  
        \epsilon_{\mathcal F}^2 + \epsilon_{\mathcal G}^2
        \enspace,
    \end{align*}
    Therefore,
    \begin{align*}
        \mathcal N(\Cat(\mathcal F, \mathcal G),
        \sqrt{\epsilon_{\mathcal F}^2 + \epsilon_{\mathcal G}^2},
        \norm{\cdot}_X)
        &\le
        \card(\mathcal U_{\Cat(\mathcal F, \mathcal G)}
        )\\
        & =
        \card\left(\set{ (v, w)| ~ v\in \mathcal U_{\mathcal F},~ w \in U_{\mathcal G}}\right)
        \\
        &= 
        \card( \mathcal U_{\mathcal F})
        \card( \mathcal U_{\mathcal G})
        \\
        &=
        \mathcal N(\mathcal F, \epsilon_{\mathcal F}, \norm{\cdot}_{X})
        \mathcal N(\mathcal G, \epsilon_{\mathcal G}, \norm{\cdot}_{X})
        \enspace.
    \end{align*}
\end{proof}

\subsection{Whole-network covering number bounds (convolutional \& fully-connected)}
\label{appendix:subsection:multi_layer_specific}

In order to compute covering number bounds for specific residual network architectures, we need to specify the function classes $\mathcal G_i$ and $\mathcal H_i$. 
We will present exemplary proofs for a simple residual network with fixed shortcuts (\cref{theorem:residual_network_covering_ceiling}, which corresponds to \cref{paper:theorem:residual_network_covering} in the main text) and the ResNet18 architecture \citep{He16a} without batch normalization, see \cref{theorem:resnet18}.

\subsubsection{Bounds for residual networks}

\begin{theorem}[Covering numbers for residual networks]
    \label{theorem:residual_network_covering_ceiling}
    For $i=1,\dots,L$ let $j=1,\dots, L_i$, $s_{ij}>0$ and $b_{ij}>0$.
    Further, let $\mathcal F$ be the class of residual networks of the form 
    \begin{equation}
        f = \sigma_L \circ f_L \circ \dots \circ \sigma_1 \circ f_1 \enspace,
    \end{equation}
    with $\sigma_i$ fixed $\rho_i$-Lipschitz functions satisfying $\sigma_i(0)=0$, and $f_i$ residual blocks with fixed shortcuts $g_i$, i.e.,
    \begin{equation}
        f_i: g_i  + (\sigma_{iL_i} \circ f_{iL_i} \circ \dots \circ \sigma_{i1} \circ f_{i1}) \enspace,
    \end{equation}
    where $\sigma_{ij}$ are fixed $\rho_{ij}$-Lipschitz functions with $\sigma_{ij}(0)=0$ and $g_i$ is Lipschitz with $g_i(0)=0$.
    
    The fully-connected or convolutional layers $f_{ij} \in \layer_{ij}$ are parametrized by matrices $A_{ij}$ or weight tensors $K_{ij}$, respectively. They satisfy Lipschitz constant constraints $s_{ij}$ and ($2,1$) group norm distance constraints $b_{ij}$ \wrt reference weights $M_{ij}$.
    That is, for convolutions
    \begin{equation*}
        \layer_{ij}
        =
        \set{\phi_{K_{ij}}~|~ \Lip(\phi_{K_{ij}}) \le s_{ij},~ \norm{K_{ij} - M_{ij}}_{2,1}\le b_{ij}}
    \end{equation*}
    and for fully-connected layers
    \begin{equation*}
        \layer_{ij}
        =
        \set{\phi: x \mapsto A_{ij}x~|~ \Lip(\phi) \le s_{ij},~ \norm{A_{ij}^\top - M_{ij}^\top}_{2,1}\le b_{ij}}
        \enspace.
    \end{equation*}

    Upon letting $W_{ij}$ denote the number of parameters of each layer and defining
    \begin{align*}
        & C_{ij}
        =  C_{ij}(X) 
        =
        2\,\frac{\norm{X}}{\sqrt{n}} 
        \left(
            \prod_{\substack{l=1}}^{L}
                s_l \rho_l
        \right)
        \frac
        {\prod_{\substack{k=1}}^{L_i} \rho_{ik} s_{ik}}     
        {s_i} 
        \frac{b_{ij}}{s_{ij}}
        \enspace, \\
        & \bar L = \sum_{i=1}^L L_i\enspace, \quad 
        W = \max_{ij} W_{ij}\enspace, \quad
        s_i = \Lip(g_i)+\prod_{j=1}^{L_i} \rho_{ij} s_{ij}
        \enspace,
    \end{align*}
   it holds that 
    \begin{equation}
        \label{theorem:ceil_residual_net:eq_norms}
        \log \mathcal N(
            \mathcal F,
            \epsilon,
            \norm{\cdot}_{X}
        )
        \le
        \log(2W)
        \left(
            \sum_{i=1}^L
            \sum_{j=1}^{L_i}
            \ceil{
                C_{ij}^{2/3}
            }
        \right)^3  
        \ceil{\frac {{n}} {\epsilon^2}}
        \enspace,
    \end{equation}
    and
    \begin{equation}
        \label{theorem:ceil_residual_net:eq_params}
        \log \mathcal N(\mathcal F, \epsilon, \norm{\cdot}_X)
        \le
        \sum_{i=1}^L
        \sum_{j=1}^{L_i}
        2 W_{ij}
        \log\left(
            1 +
            \ceil{
                \bar L^2 C_{ij}^2
            }
            \ceil{\frac{n}{\epsilon^2}}
        \right)  \enspace.
    \end{equation}
\end{theorem}

\begin{proof}
    As we consider residual networks with fixed shortcuts, the covering number bound from \cref{theorem:residual_networks} simplifies to 
    \begin{align*}
        & \mathcal N\left(
            \mathcal F,
        \sum_{i=1}^L
            \left( \prod_{l=i+1}^L \Lip(\mathcal F_l)\right)\!
                \sum_{j=1}^{L_{\mathcal H_i}}
                \left(\! \prod_{k=j+1}^{L_{\mathcal H_i}} \Lip(\mathcal H_{ik})\right)
                \epsilon_{\mathcal H_{ij}},
            \norm{\cdot}_X
        \right) \\
        & ~~~~~\le
        \prod_{i=1}^L
        \prod_{j=1}^{L_i}
                \sup_{\substack{\psi_{ij} \in \\ \Comp( \compto{\mathcal F_{i}}, \compto{\mathcal H_{ij}})}}
                \mathcal N\left(
                    \mathcal H_{ij},
                    \epsilon_{\mathcal H_{ij}},
                    \norm{\cdot}_{\psi_{ij} (X)}
                \right)
        \enspace.
    \end{align*}
    In our setting,  $\mathcal F_i = \set{\sigma_i \circ f_i}$ with
    $\Lip(\mathcal F_i) \le \rho_i s_i$, $\mathcal H_{ij}= \Comp(\layer_{ij},\set{\sigma_{ij}})$ with $\Lip(\mathcal H_{ij}) \le \rho_{ij} s_{ij}$. Further, $\epsilon_{\mathcal H_{ij}} = \rho_{ij} \epsilon_{ij}$ and $L_{\mathcal H_{ij}}= L_{ij}$.
    As covering numbers decrease with the radius, it follows that 
    \begin{align*}
        & \mathcal N\left(
            \mathcal F,
        \sum_{i=1}^L
            \left( \prod_{l=i+1}^L s_l \rho_l \right)\!
                \sum_{j=1}^{L_{i}}
                \left(\! \prod_{k=j+1}^{L_{i}} s_{ik} \rho_{ik}\right)
                \rho_{ij}
                \epsilon_{ij},
            \norm{\cdot}_X
        \right) \\
        & ~~~~~\le
        \prod_{i=1}^L
        \prod_{j=1}^{L_i}
                \sup_{\substack{\psi_{ij} \in \\ \Comp( \compto{\mathcal F_{i}}, \compto{\mathcal \layer_{ij}})}}
                \mathcal N\left(
                    \mathcal \layer_{ij},
                    \epsilon_{ij},
                    \norm{\cdot}_{\psi_{ij} (X)}
                \right)
        \enspace.
    \end{align*}

    Now, for each $ij$ referring to convolutional layers, we have 
    \begin{align*}
        &\mathcal N\left(
                \layer_{ij},
                \epsilon_{ij},
                \norm{\cdot}_{\psi_{ij}(X)}
            \right)
            \\
        & ~~~~~= 
        \mathcal N\left(
            \set{\phi_{K_{ij}}~|~ \Lip(\phi_{K_{ij}}) \le s_{ij},~ \norm{K_{ij} - M_{ij}}_{2,1}\le b_{ij}},
                \epsilon_{ij},
                \norm{\cdot}_{\psi_{ij}(X)}
            \right)
        \\
        & ~~~~~=
        \mathcal N\left(
            \set{\phi_{K_{ij}} -\phi_{M_{ij}} ~|~ \Lip(\phi_{K_{ij}}) \le s_{ij},~ \norm{K_{ij} - M_{ij}}_{2,1}\le b_{ij}},
                \epsilon_{ij},
                \norm{\cdot}_{\psi_{ij}(X)}
            \right)
        \\
        & ~~~~~=
        \mathcal N\left(
            \set{\phi_{K_{ij} -M_{ij}} ~|~ \Lip(\phi_{K_{ij}}) \le s_{ij},~ \norm{K_{ij} - M_{ij}}_{2,1}\le b_{ij}},
                \epsilon_{ij},
                \norm{\cdot}_{\psi_{ij}(X)}
            \right)
            \enspace.
    \end{align*}

    In this chain of equalities, we used the translation invariance of covering numbers, i.e., \cref{theorem:summations} with one summand being the singleton $\set{-\phi_{M_{ij}}}$, and the linearity of $\phi$ in the weights to accommodate the distance to initialization. An analogous inequality holds for fully-connected layers.

    \cref{theorem:single_layer} provides bounds for the covering number of the superset 
    $$\set{\phi_{K_{ij} -M_{ij}} ~\big |~ \norm{K_{ij} - M_{ij}}_{2,1}\le b_{ij}}\enspace.$$ 
    Hence, to proceed, we need to transition to external covering numbers, which requires halving the radius $\epsilon$. This yields
    \begin{align}
        & \mathcal N\left(
                \layer_{ij},
                \epsilon_{ij},
                \norm{\cdot}_{\psi_{ij}(X)}
            \right)
            \nonumber
        \\
        & ~~~~~\stackrel{\phantom{\text{Eq.}~\eqref{eq:covers_int_ext}}}\le 
        \mathcal N\left(
            \set{\phi_{K_{ij} -M_{ij}} ~|~ \Lip(\phi_{K_{ij}}) \le s_{ij},~ \norm{K_{ij} - M_{ij}}_{2,1}\le b_{ij}},
                \epsilon_{ij},
                \norm{\cdot}_{\psi_{ij}(X)}
            \right)
            \nonumber
        \\
        & ~~~~~\stackrel{\text{Eq.}~\eqref{eq:covers_int_ext}}\le 
        \mathcal N^{\text{ext}}\left(
            \set{\phi_{K_{ij} -M_{ij}} ~|~ \Lip(\phi_{K_{ij}}) \le s_{ij},~ \norm{K_{ij} - M_{ij}}_{2,1}\le b_{ij}},
                \frac {\epsilon_{ij}}{2},
                \norm{\cdot}_{\psi_{ij}(X)}
            \right)
            \nonumber
        \\
        &~~~~~\stackrel{\text{Eq.}~\eqref{eq:covers:subset}}\le 
        \mathcal N^{\text{ext}}\left(
            \set{\phi_{K_{ij} -M_{ij}} ~|~ \norm{K_{ij} - M_{ij}}_{2,1}\le b_{ij}},
            \frac {\epsilon_{ij}}{2},
                \norm{\cdot}_{\psi_{ij}(X)}
            \right)
            \nonumber
        \\
        &~~~~~\stackrel{\text{Eq.}~\eqref{eq:covers_int_ext}}\le 
        \mathcal N\left(
            \set{\phi_{K_{ij} -M_{ij}} ~|~ \norm{K_{ij} - M_{ij}}_{2,1}\le b_{ij}},
            \frac {\epsilon_{ij}}{2},
                \norm{\cdot}_{\psi_{ij}(X)}
            \right)
        \enspace.
        \label{eq:to_external_coverings}
    \end{align} 

    Thus, by \cref{theorem:single_layer:eq_norms}, it holds that
    \begin{align*}
        &\log \mathcal N\left(
            \mathcal F,
        \sum_{i=1}^L
            \left( \prod_{l=i+1}^L s_l \rho_l\right)\!
                \sum_{j=1}^{L_i}
                \left(\! \prod_{k=j+1}^{L_i} s_{ik} \rho_{ik}\right)
                \rho_{ij}
                \epsilon_{{ij}},
            \norm{\cdot}_X
        \right) 
        \\
        &~~~~~\le
        \sum_{i=1}^L
        \sum_{j=1}^{L_i}
                \sup_{\substack{\psi_{ij} \in \\ \Comp( \compto{\mathcal F_{i}}, \compto{\mathcal \layer_{ij}})}}
                \log(2W_{ij}) 
                \ceil{\frac{4\norm{\psi_{ij}(X)}^2 b_{ij}^2}{\epsilon_{ij}^2}}
        \\
        &~~~~~\le
        \log(2W)
        \sum_{i=1}^L
        \sum_{j=1}^{L_i}
        \ceil{
            4\norm{X}^2
            \left(\prod_{l=1}^{i-1} s_l \rho_l\right)^2
            \left(\prod_{k=1}^{j-1} s_{lk}\rho_{lk}\right)^2
            \frac{b_{ij}^2} {\epsilon_{ij}^2}
        }
        \enspace.
    \end{align*}
    Notably, the second inequality requires the assumption that all $g_i, f_{ij}, \sigma_i$ and $\sigma_{ij}$ map zero to zero.
    The next step is to choose radii $\epsilon_{ij}$ so that the right-hand side becomes small under the condition that 
    \begin{equation}
        \label{theorem:residual_network_covering_ceiling:eq_eps}
        \sum_{i=1}^L
            \left( \prod_{l=i+1}^L s_l \rho_l\right)\!
                \sum_{j=1}^{L_i}
                \left(\! \prod_{k=j+1}^{L_i} s_{ik} \rho_{ik}\right)
                \rho_{ij}
                \epsilon_{{ij}}
        = \epsilon
        \enspace.
    \end{equation}
    We choose
    \begin{equation}
        \epsilon_{{ij}}
        = 
        \frac{\epsilon}
        {\left( \prod_{l=i+1}^L s_l \rho_l\right)\!
        \left(\! \prod_{k=j+1}^{L_i} s_{ik} \rho_{ik}\right)
        \rho_{ij}}
        \frac{\alpha_{ij}}
        {\sum_{lk} \alpha_{lk}}
        \enspace,
        \qquad
        \alpha_{ij} = \frac{b_{ij}^{2/3}}{s_{ij}^{2/3}}
        \enspace,
    \end{equation}
    which would be optimal for the analogous optimization problem without ceiling functions.
    Then, 
    \begin{align*}
        & \log \mathcal N\left(
            \mathcal F,
            \epsilon,
            \norm{\cdot}_X
        \right) \\
        &~~~~~\le
        \log(2W)
        \sum_{i=1}^L
        \sum_{j=1}^{L_i}
        \ceil{4
            \left(\prod_{l=1}^{L} s_l \rho_l \right)^2
            \frac
                {\left(\prod_{k=1}^{L_i} s_{ik}\rho_{ik}\right)^2}
                {s_i^2}
            \frac{b_{ij}^2} {s_{ij}^2 \alpha_{ij}^2}
            \frac{\norm{X}^2}{\epsilon^2}
        \left(
            \sum_{l=1}^L \sum_{k=1}^{L_l}
            \alpha_{lk}
        \right)^2
        }
        \\
        & ~~~~~\le
        \log(2W)
        \sum_{i=1}^L
        \sum_{j=1}^{L_i}
        \ceil{4
            \left(\prod_{l=1}^{L} s_l \rho_l\right)^2
            \frac
                {\left(\prod_{k=1}^{L_i} s_{ik}\rho_{ik}\right)^2}
                {s_i^2}
            \frac{b_{ij}^{2/3}} {s_{ij}^{2/3}}
            \frac{\norm{X}^2}{n}
            \frac{n}{\epsilon^2}
        \left(
            \sum_{l=1}^L \sum_{k=1}^{L_l}
            \frac{b_{lk}^{2/3}} {s_{lk}^{2/3}}
        \right)^2
        }
        \\
        & ~~~~~\le
        \log(2W)
        \left(
            \sum_{i=1}^L
            \sum_{j=1}^{L_i}
            \ceil{
                \left(2\,
                    \frac{\norm{X}}{\sqrt n}
                    \left(\prod_{l=1}^{L} s_l \rho_l\right)
                    \frac
                        {\left(\prod_{k=1}^{L_i} s_{ik}\rho_{ik}\right)}
                        {s_i}
                    \frac{b_{ij}} {s_{ij}}
                \right)^{2/3}
            }
        \right)^3  
        \ceil{\frac {{n}} {\epsilon^2}}
        \\
        & ~~~~~=
        \log(2W)
        \left(
            \sum_{i=1}^L
            \sum_{j=1}^{L_i}
            \ceil{
                C_{ij}^{2/3}
            }
        \right)^3  
        \ceil{\frac {{n}} {\epsilon^2}}
        \enspace,
    \end{align*}
    which establishes the \emph{first} covering number bound, i.e., \cref{theorem:ceil_residual_net:eq_norms}, from \cref{theorem:residual_network_covering_ceiling}.

    Similarly, \cref{theorem:single_layer:eq_params} implies
    \begin{align*}
        &\log \mathcal N\left(
            \mathcal F,
        \sum_{i=1}^L
            \left( \prod_{l=i+1}^L s_l \rho_l \right)\!
                \sum_{j=1}^{L_i}
                \left(\! \prod_{k=j+1}^{L_i} s_{ik} \rho_{ik}\right)
                \rho_{ij}
                \epsilon_{ij}, 
            \norm{\cdot}_X
        \right) 
        \\
        &~~~~~\le
        \sum_{i=1}^L
        \sum_{j=1}^{L_i}
        2 W_{ij}
        \log\left(
            1 + 
            \ceil{4
                \norm{X}^2
                \left(\prod_{l=1}^{i-1} s_l \rho_l\right)^2
                \left(\prod_{k=1}^{j-1} s_{ik}\rho_{ik}\right)^2
                \frac{b_{ij}^2} {\epsilon_{ij}^2}
            }
        \right)        
        \enspace.
    \end{align*}
    Again, we need to choose the $\epsilon_{ij}$ such that \cref{theorem:residual_network_covering_ceiling:eq_eps} holds.
    We choose
    \begin{equation}
        \epsilon_{ij}
        = 
        \frac{\epsilon}
        {\left( \prod_{l=i+1}^L s_l \rho_l\right)\!
        \left(\! \prod_{k=j+1}^{L_i} s_{ik} \rho_{ik}\right)
        \rho_{ij}}
        \frac{\alpha_{ij}}
        {\sum_{lk} \alpha_{lk}}
        \quad \text{with} \quad
        \alpha_{ij} = 1
        \enspace.
    \end{equation}
    This simple choice yields the optimal solution for the problem of minimizing
    $$\sum_{ij} \log\left( \frac
    {4 \norm{X}^2
    \left(\prod_{l=1}^{i-1} s_l \rho_l \right)^2
    \left(\prod_{k=1}^{j-1} \rho_{ik}\right)^2 b_{ij}^2}
    {\epsilon_{ij}^2}\right)\enspace.$$
    Hence, we expect it to be a good choice if the $W_{ij}$ are roughly equal and $\epsilon$ is small.
    Overall, we get
    \begin{align*}
        \log \mathcal N\left(
            \mathcal F,
            \epsilon,
            \norm{\cdot}_X
        \right) 
        &\le
        \sum_{i=1}^L
        \sum_{j=1}^{L_i}
        2 W_{ij}
        \log\left(
            1 + 
            \ceil{4
                \bar L^2
                \left(\prod_{l=1}^{L} s_l \rho_l\right)^2
            \frac
                {\left(\prod_{k=1}^{L_i} s_{ik}\rho_{ik}\right)^2}
                {s_i^2}
            \frac{b_{ij}^2} {s_{ij}^2}
            \frac{\norm{X}^2}{\epsilon^2}
            }
        \right)
        \\
        & \le
        \sum_{i=1}^L
        \sum_{j=1}^{L_i}
        2 W_{ij}
        \log\left(
            1 +
            \ceil{
                \bar L^2    C_{ij}^2
            }
            \ceil{\frac{n}{\epsilon^2}}
        \right)    
        \enspace,
    \end{align*}
    which establishes the \emph{second} covering number bound, i.e., \cref{theorem:ceil_residual_net:eq_params}, from \cref{theorem:residual_network_covering_ceiling}.
\end{proof}

\begin{corollary}[Covering numbers for non-residual networks]
    \label{theorem:covering_feedforwad}
    For $i \in \{1,\dots, L\}$, let $\layer_i$ be a function class 
    with Lipschitz constraint $s_i$ and ($2,1$) group norm distance constraint $b_i$ with respect to a reference weight $M_i$.
    In particular, if $\layer_i$ is convolutional, then
    \begin{equation*}
        \layer_{i}
        =
        \set{\phi_K~|~ \Lip(\phi_K) \le s_{i},~ \norm{K_{i} - M_{i}}_{2,1}\le b_{i}}
    \end{equation*}
    and if $\layer_i$ is fully-connected, then
    \begin{equation*}
        \layer_{i}
        =
        \set{\phi: x \mapsto A_{i}x~|~ \Lip(\phi) \le s_{i},~ \norm{A_{i} - M_{i}}_{2,1}\le b_{i}}
        \enspace.
    \end{equation*}
    We write $W_{i}$ for the number of parameters of each layer, i.e., the number of elements of each $K_{i}$, resp. $A_{i}$.
    Further, let $\mathcal F = \set{\sigma_L \circ f_L \circ \dots \circ \sigma_1 \circ f_1~|~ f_i \in \layer_i}$,
    where the maps $\sigma_i$ are $\rho_i$-Lipschitz with $\sigma_i(0)=0$, and define 
    \begin{equation}
        C_i = C_i(X)
        = 2\,\frac{\norm{X}}{\sqrt{n}} \left(\prod_{l=1}^{L} \rho_l s_l \right) \frac{b_i}{s_i}\enspace, 
        \qquad
        W = \max_i W_i
        \enspace.
    \end{equation}
    Then, for every input data $X = \tuple{x_1,\dots, x_n}$ and every $\epsilon>0$, it holds that
    \begin{equation}
        \label{theorem:fully_connected:norms}
        \log \mathcal N\left(
            \mathcal F,
            \epsilon,
            \norm{\cdot}_{X}
        \right)
        \le
        \log(2W)
        \left(
            \sum_{i=1}^{{L}}
            \ceil{C_i^{2/3}}
        \right)^3
        \ceil{\frac{n}{\epsilon^2}}
    \end{equation}
    and
    \begin{equation}
        \label{theorem:fully_connected:params}
        \log \mathcal N(\mathcal F, \epsilon, \norm{\cdot}_X)
        \le
        \sum_{i=1}^{L}
        2 W_i \log\left(1 +
            \ceil{C_i^2} \ceil{\frac{n}{\epsilon^2}}
        \right)
        \enspace.
    \end{equation}
\end{corollary}
\begin{proof}
    Follows directly from \cref{theorem:residual_network_covering_ceiling}, as the network can be considered as a single (long) residual block, whose shortcut $g:x \mapsto 0$ is the zero map.
\end{proof}

\begin{remark}
    \label{remark:complex_resnets}
    Similarly, we can derive covering number bounds for networks, where each block $\mathcal F_i$ is a sum of  $w_i$ parametrized maps, i.e., $\mathcal F_i = \Sum(\mathcal G_{i1}, \dots, \mathcal G_{iw_i})$, with $\mathcal G_{ij} = \Comp(\layer_{ij1}, \sigma_{ij1}, \dots, \layer_{ijL_{ij}}, \sigma_{ijL_{ij}})$. In this setting, 
    the whole-network covering number is bounded by
    \begin{equation}
        \log \mathcal N(
            \mathcal F,
            \epsilon,
            \norm{\cdot}_{X}
        )
        \le 
        \log(2W)
            \left(
            \sum_{i=1}^{{L}}
            \sum_{j=1}^{{w_i}}
            \sum_{k=1}^{L_{ij}}
            \ceil{C_{ijk}^{2/3}}
            \right)^3
            \left\lceil\frac{n}{\epsilon^2}\right\rceil
    \end{equation}
    and
    \begin{equation}
    \log \mathcal N(
        \mathcal F,
        \epsilon,
        \norm{\cdot}_{X}
    )
    \le 
    \sum_{i=1}^{{L}}
    \sum_{j=1}^{{w_i}}
    \sum_{k=1}^{L_{ij}}
        2 W_{ijk}
        \log\left(
            1 +
            \ceil{\bar L^2C_{ijk}^2}
            \left\lceil\frac{n}{\epsilon^2}\right\rceil
        \right)
    \end{equation}
    for 
    \begin{equation}
        C_{ijk} = 4 \frac{\norm{X}}{\sqrt{n}} 
        \left(
            \prod_{\substack{l=1}}^{L}
                s_l \rho_l
        \right)
        \frac
        {\prod_{\substack{m=1}}^{L_{ij}} \rho_{ijm} s_{ijm}}     
        {s_i} 
        \frac{b_{ijk}}{s_{ijk}}
        ,
        \qquad
        \bar L = \sum_{i=1}^L \sum_{j=1}^{w_i} L_{ij}
        \enspace,
    \end{equation}
    where $s_i=\sum_{j=1}^{w_i} 
    \left(\prod_{k=1}^{L_{ij}}s_{ijk}\rho_{ijk}\right)$, $\rho_{ijk} = \Lip(\sigma_{ijk})$ and $s_{ijk}, b_{ijk}$ are constraints on the layers $\layer_{ijk}$.
\end{remark}

\subsubsection{Application to specific architectures}

\begin{example}[ResNet18]
    \label{theorem:resnet18}
    We derive covering number bounds for the ResNet18 architecture \citep{He16a} \emph{without} batch normalization, illustrated in \cref{theorem:resnet18:fig1}. We can think of the ResNet18 as a composition of 10 residual blocks, the first and last one having the zero map as shortcut and five blocks having identity shortcuts. The remaining 3 blocks have downsampling shortcuts of the form $\sigma \circ \psi$, where $\psi$ is a 1x1 convolution and $\rho$ is the ReLU activation function. These blocks are handled by \cref{remark:complex_resnets}.
    Furthermore, all nonlinearities are $1$-Lipschitz and map zero to zero.

    \vskip1ex
    \begin{figure}[h!]
        \begin{center}
        \includegraphics[scale=0.59, angle=90]{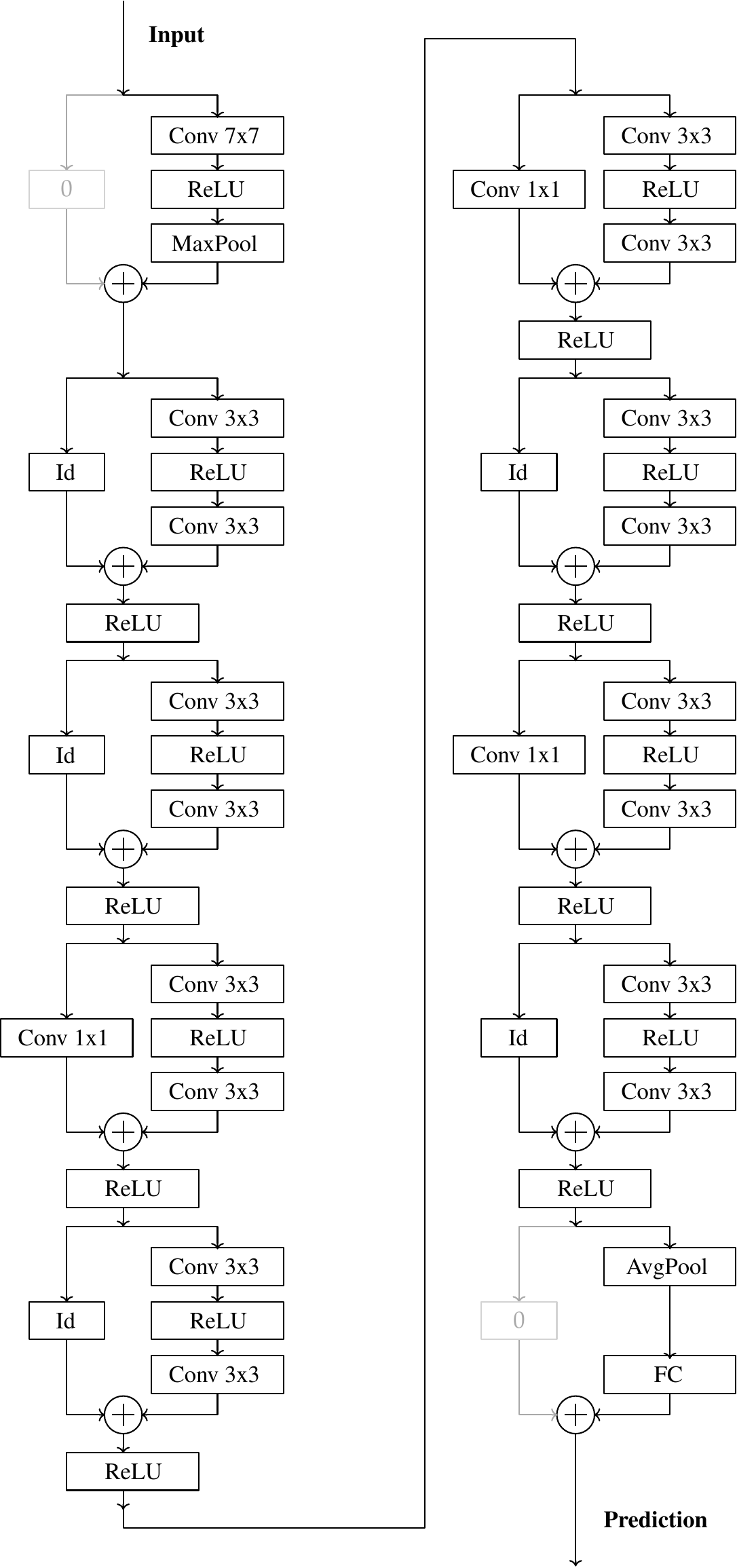}
        \end{center}
        \caption{ResNet18 \citep{He16a} architecture \emph{without} batch normalization.}
        \label{theorem:resnet18:fig1}
    \end{figure}
    
    For any data $X = \tuple{x_1, \dots x_n}$ and any $\epsilon>0$, the covering number of the function class  $\mathcal F$ corresponding to the ResNet18 architecture without batch normalization, with no bias parameters and  with distance and Lipschitz constrained layers, is approximately (ignoring ceiling functions) bounded by 
    \begin{alignat*}{2}
        \log \mathcal N ( \mathcal F, \epsilon, \norm{\cdot}_X) 
        \lesssim
        & & &
        4\,\frac{\norm{X}^2}{\epsilon}
        \log(2W) \\
        & & & 
        s_1^2 \left(\prod_{i\in \set{2,3,5,7,9}} (1 + s_{i1}s_{i2})\right)^{\!2} 
        \left(\prod_{i\in \set{4,6,8}} (s_{i,\text{down}} + s_{i1}s_{i2})\right)^{\!2}
        s_{10}^2 \\
        &\Biggl[&& \frac{b_{1}^{2/3}}{s_{1}^{2/3}}
        + \frac{b_{10}^{2/3}}{s_{10}^{2/3}}
        + \sum_{i\in \set{2,3,5,7,9}} 
                \frac{1}{(1 + s_{i1} s_{i2})^{2/3}}
                \left(
                \frac{b_{i1}^{2/3}}{s_{i1}^{2/3}} 
                +
                \frac{b_{i2}^{2/3}}{s_{i2}^{2/3}} 
                \right) \\
        & &&+ \sum_{i\in \set{4,6,8}} 
        {\frac{1}{(s_{i,\text{down}} + s_{i1} s_{i2})^{2/3}}}
            \left(
                    \frac{b_{i,\text{down}}^{2/3}}{s_{i,\text{down}}^{2/3}}
                    +
                    \frac{ b_{i1}^{2/3}}{s_{i1}^{2/3}} 
                    +
                    \frac{ b_{i2}^{2/3}}{s_{i2}^{2/3}} 
            \right)
        \Biggl]^3 
        \enspace.
    \end{alignat*}

    Here $s_{i1}$, resp $s_{i2}$, denotes the Lipschitz constraint on the first, resp. second, layer in the $i$-th residual block and $s_{i,{\text{down}}}$ the constraint on the downsampling layer (1x1 convolution). The (2,1)-distance constraints are denoted by $b_{i1}, b_{i2}$ and $b_{i,\text{down}}$.

\end{example}

\subsection{Rademacher complexity \& Generalization bounds}
\label{appendix:subsection:generalization_bounds} 
The empirical Rademacher complexity can be upper bounded via Dudley's entropy integral. In the following, we restate a variant of this standard result as it appears in Bartlett et al.~\cite{Bartlett17a}.
\vskip2ex

\begin{theorem}[Dudley entropy integral, cf.~{\cite[Lemma A.5]{Bartlett17a}}]
Let $\mathcal{F}$ be a class of functions mapping to $[0,1]$ containing the zero function. Then
$$
\rade{X}{\mathcal F} \leq \inf\limits_{0<t\le\sqrt{n}} 
\left( 
    \frac{4t}{\sqrt{n}} + \frac{12}{n}\int_{t}^{\sqrt{n}}
    \sqrt{
    \log \mathcal{N}(\mathcal{F},\epsilon,\|\cdot\|_X)
    }\,
    \mathrm{d}\epsilon
\right)\enspace.
$$
\end{theorem}

We will compute Dudley's entropy integral for the covering number bounds from \cref{appendix:subsection:multi_layer_specific}.

\vskip2ex
\begin{theorem}[Empirical Rademacher complexity for residual networks]
    \label{theorem:residual_network_covering_rademacher}
    For $i=1,\dots,L$ let $j=1,\dots, L_i$, $s_{ij}>0$ and $b_{ij}>0$.
    Further, let $\mathcal F$ be the class of residual networks of the form 
    \begin{equation}
        f = \sigma_L \circ f_L \circ \dots \circ \sigma_1 \circ f_1 \enspace,
    \end{equation}
    with $\sigma_i$ fixed $\rho_i$-Lipschitz functions satisfying $\sigma_i(0)=0$, and with $f_i$ residual blocks with identity shortcuts, i.e.,
    \begin{equation}
        f_i: \Id  + (\sigma_{iL_i} \circ f_{iL_i} \circ \dots \circ \sigma_{i1} \circ f_{i1}) \enspace,
    \end{equation}
    where $\sigma_{ij}$ are fixed $\rho_{ij}$-Lipschitz functions with $\sigma_{ij}(0)=0$ and $f_{ij}$ are convolutional or fully-connected layers. 
    The layers $f_{ij}$ satisfy Lipschitz constraints $\Lip(f_{ij}) \le s_{ij}$ and the corresponding weight tensors $K_{ij}$, respectively weight matrices $A_{ij}$, satisfy distance constraints 
    $$\norm{K_{ij} - K_{ij}^{(0)}}_{2,1}\le b_{ij}\enspace, \quad\text{respectively}\quad \norm{A_{ij} - A_{ij}^{(0)}}_{2,1} \le b_{ij}\enspace,$$
    with respect to reference weights $K_{ij}^{(0)}$, respectively $A_{ij}^{(0)}$.
    
    Upon letting $W_{ij}$ denote the number of parameters of each layer and defining
    \begin{align*}
        & \tilde C_{ij}(X) = 
        \tilde C_{ij}=
        \frac{4}{\gamma}
        \frac{\norm{X}}{\sqrt{n}} 
        \left(
            \prod_{\substack{l=1}}^{L}
                \Lip(\mathcal F_l) \rho_l
        \right)
        \frac
        {\prod_{\substack{k=1}}^{L_i} \rho_{ik} s_{ik}}     
        {\Lip(\mathcal F_i)} 
        \frac{b_{ij}}{s_{ij}}
        \enspace,\\
        & \bar L = \sum_{i=1}^L L_i
        \enspace,
        \qquad
        W = \max_{ij} W_{ij}
        \enspace,
    \end{align*}
    the empirical Rademacher complexity of the function class $\mathcal F_\gamma$, with margin parameter $\gamma>0$, satisfies
    \begin{equation}
        \label{theorem:residual_network_covering_rademacher:eq_norms}
        \rade{X}{\mathcal F_\gamma}
        \le
        \frac{4}{n}
        +
        \frac{12 H_{n-1}}{\sqrt{n}}
        \sqrt{\log(2W)}
        \left(
            \sum_{i=1}^L
            \sum_{i=j}^{L_i}
            \ceil{\tilde C_{ij}^{2/3}}
        \right)^{3/2}
    \end{equation}
    and
\begin{equation}
    \label{theorem:residual_network_covering_rademacher:eq_params}
    \rade{X}{\mathcal F}
        \leq
        \frac{12}{\sqrt n}
        \sqrt{
            \sum_{i=1}^L
            \sum_{j=1}^{L_i}
            2W_{ij}
            \left(
            \log\left(1+\ceil{\bar L ^2 \tilde C_{ij}^2}\right) + 
            \zeta\left(\frac 3 2,1\right)^{1/3}
            \zeta\left(
                \frac 3 2,
                1+ {1}/{\ceil{\bar L ^2 \tilde C_{ij}^2}}
            \right)^{2/3}
            \right)
        }
\end{equation}
Here, $H_{n-1}= \sum_{m=1}^{n-1} \frac{1}{m}$ denotes the $(n-1)$-th harmonic number and $\zeta(s,q) = \sum_{n=0}^\infty \frac{1}{(q+n)^s}$ the Hurwitz zeta function.
\end{theorem}

\vskip1ex
\begin{remark}
    The harmonic number satisfies
    $H_{n-1} < \log(n) + \gamma \approx \log(n) + 0.58$. 
    The function 
    $\psi: x \mapsto 
    \zeta\left(\frac 3 2,1\right)^{1/3}
            \zeta\left(
                \frac 3 2,
                1+ {1}/{x}
            \right)^{2/3}
    $
    is monotonically increasing with $\psi(0)=0$ and upper bounded by $\zeta(\frac{3}{2})\approx 2.62$. So, for large $\tilde C = \max_{ij} \tilde C_{ij}$, the second summand is negligible and \cref{theorem:residual_network_covering_rademacher:eq_params} scales as $\sqrt{\bar W \log (\bar L^2 \tilde C^2)}$. Here, $\bar W = \sum_{ij}
    W_{ij}$ denotes the number of network parameters.
\end{remark}

\begin{proof}
    Both inequalities follow from a combination of Dudley's entropy integral 
    with a covering number bound from \cref{theorem:residual_network_covering_ceiling}.

    Since $\ell_\gamma(-\mathcal M (\cdot,\cdot))$ is  a fixed $2/\gamma$-Lipschitz function, the covering number of $\mathcal F_\gamma$ can be bounded as in \cref{theorem:covering_feedforwad} with
    \begin{equation}
        C_{ij}(X) = \tilde C_{ij}(X)
        = 
        \frac{4}{\gamma} \frac{\norm{X}}{\sqrt{n}} \left(\prod_{l=1}^{L} \rho_l s_l \right) \frac{b_i}{s_i}
        \enspace.
    \end{equation}

    To prove \cref{theorem:residual_network_covering_rademacher:eq_norms}, we insert \cref{theorem:ceil_residual_net:eq_norms} into Dudley's entropy integral, which yields
    \begin{align*}
        \rade{X}{\mathcal F}
        & \le
        \inf\limits_{0\le t\le \sqrt{n}}
        \left( 
            \frac{4t}{\sqrt{n}} 
            + 
            \frac{12}{n}
                \sqrt{\log(2W)}
                \left(
                    \sum_{i=1}^L
                    \sum_{j=1}^{L_i}
                    \ceil{
                        \tilde C_{ij}^{2/3}
                    }
                \right)^{3/2}
            \int_{t}^{\sqrt{n}}
            \ceil{\frac{\sqrt{n}}{\epsilon}}
            \,
            \mathrm{d}\epsilon
        \right)
        \\
        &=
        \inf\limits_{0\le t\le \sqrt{n}}
        \left( 
            \frac{4t}{\sqrt{n}} 
            + 
            \frac{12}{\sqrt n}
            \sqrt{\log(2W)}
                \left(
                    \sum_{i=1}^L
                    \sum_{j=1}^{L_i}
                    \ceil{
                        \tilde C_{ij}^{2/3}
                    }
                \right)^{3/2}
            \int_{t/\sqrt{n}}^{1}
            \ceil{\frac{1}{s}}
            \,
            \mathrm{d}s
        \right)
        \\
        &=
        \inf\limits_{0\le t\le 1}
        \left( 
            4t
            + 
            \frac{12}{\sqrt n}
            \sqrt{\log(2W)}
                \left(
                    \sum_{i=1}^L
                    \sum_{j=1}^{L_i}
                    \ceil{
                        \tilde C_{ij}^{2/3}
                    }
                \right)^{3/2}
            \int_{t}^{1}
            \ceil{\frac{1}{s}}
            \,
            \mathrm{d}s
        \right)
        \enspace.
    \end{align*}
    The value of the integral is a harmonic number if $1/t \in \mathbb N$, as then
    \begin{align*}
        \int_t^1 \ceil{\frac{1}{s}} \, \mathrm{d}s
        & = 
        \sum_{m=1}^{1/t-1}
        \int_{1/(m+1)}^{1/m} 
        \ceil{\frac{1}{s}} \, \mathrm{d}s
        = \sum_{m=1}^{1/t-1} \left(\frac 1 m - \frac{1}{m+1}\right) (m+1)
        = \sum_{m=1}^{1/t-1} \frac 1 m
        = H_{1/t-1}
        \enspace.
    \end{align*}
    Choosing $t = 1/n$, establishes the inequality in \cref{theorem:residual_network_covering_rademacher:eq_norms}.
    \vspace{0.5cm}

    To prove \cref{theorem:residual_network_covering_rademacher:eq_params}, we first observe that, by Jensen's inequality, it holds that 
    \begin{align*}
        \frac{1}{\sqrt n}\int_{0}^{\sqrt{n}}
        \sqrt{
        \log \mathcal{N}(\mathcal{F},\epsilon, \norm{\cdot}_X)
        }\,
        \mathrm{d}\epsilon
        &\le
        \sqrt{
            \frac{1}{\sqrt n}\int_{0}^{\sqrt{n}}
            \log \mathcal{N}(\mathcal{F},\epsilon, \norm{\cdot}_X) \,
            \mathrm{d}\epsilon
        }
        \\
        &=
        \frac{1}{\sqrt[4]{n}}
        \sqrt{
            \int_{0}^{\sqrt{n}}
            \log \mathcal{N}(\mathcal{F},\epsilon, \norm{\cdot}_X) \,
            \mathrm{d}\epsilon
        }
    \end{align*}
    and thus 
    \begin{equation*}
        \rade{X}{\mathcal F}
        \leq
        \frac{12}{n}
        {n^{1/4}}
        \sqrt{
            \int_{0}^{\sqrt{n}}
            \log \mathcal{N}(\mathcal{F},\epsilon, \norm{\cdot}_X)
            \,
            \mathrm{d}\epsilon
        }
        \enspace.
    \end{equation*}

    Then, recalling \cref{theorem:ceil_residual_net:eq_params}, i.e.,  
    \begin{equation*}
        \log \mathcal N(\mathcal F_\gamma, \epsilon, \norm{\cdot}_X)
        \le
        \sum_{ij}
                2 W_{ij}
                \log\left(
                    1 +
                    \ceil{ \bar L^2 \tilde C_{ij}^2}
                    \ceil{\frac{n}{\epsilon^2}}
                \right)
    \end{equation*}
    yields
    \begin{align*}
        &\int_{0}^{\sqrt{n}}
        \log \mathcal{N}(\mathcal{F},\epsilon, \norm{\cdot}_X)
        \,
        \mathrm{d}\epsilon \\
        & ~~~~~\le
        \sum_{ij}
        2 W_{ij}
        \int_{0}^{\sqrt{n}}
        \log\left(
            1 +
            \ceil{ \bar L^2 \tilde C_{ij}^2}
            \ceil{\frac{n}{\epsilon^2}}
        \right)
        \,
        \mathrm{d}\epsilon
        \\
        & ~~~~~=
        \sqrt{n}
        \sum_{ij}
        2 W_{ij}
        \int_{0}^{1}
        \log\left(
            1 +
            \ceil{ \bar L^2 \tilde C_{ij}^2}
            \ceil{\frac{1}{s^2}}
        \right)
        \,
        \mathrm{d}s
        \\
        & 
        \stackrel{\text{\cref{theorem:integral_params}}}{\le}
        \sqrt{n}
        \sum_{ij}
        2 W_{ij}
        \left(
            \log\left(1+\ceil{\bar L ^2 \tilde C_{ij}^2}\right) + 
            \zeta\left(\frac 3 2,1\right)^{1/3}
            \zeta\left(
                \frac 3 2,
                1+ {1}/{\ceil{\bar L ^2 \tilde C_{ij}^2}}
            \right)^{2/3}
        \right)
        \enspace.
    \end{align*}
    The last inequality follows from \cref{theorem:integral_params} (proof deferred to \cref{appendix:sec:calculations}). Overall, this implies 
    \begin{equation*}
        \rade{X}{\mathcal F}
        \leq
        \frac{12}{\sqrt n}
        \sqrt{
            2 \sum_{i=1}^L
            \sum_{j=1}^{L_i}
            W_{ij}
            \left(
            \log\left(1+\ceil{\bar L ^2 \tilde C_{ij}^2}\right) + 
            \zeta\left(\frac 3 2,1\right)^{1/3}
            \zeta\left(
                \frac 3 2,
                1+ {1}/{\ceil{\bar L ^2 \tilde C_{ij}^2}}
            \right)^{2/3}
            \right)
        }
    \end{equation*}
    which establishes the inequality in \cref{theorem:residual_network_covering_rademacher:eq_params}.
\end{proof}

\vskip2ex
\begin{corollary}
    Let $\gamma>0$ and let $\tilde C_i = 2 C_i /\gamma$. For non-residual networks as specified in \cref{theorem:covering_feedforwad}, the empirical Rademacher complexity  of $\mathcal F_\gamma$ satisfies
    \begin{equation}
        \label{theorem:nonresidual_network_covering_rademacher:eq_norms}
        \rade{X}{\mathcal F_\gamma}
        \le
        \frac{4}{n}
        +
        \frac{12 H_{n-1}}{\sqrt{n}}
        \sqrt{\log(2W)}
        \left(
            \sum_{i=1}^L
            \ceil{\tilde C_{i}^{2/3}}
        \right)^{3/2}
    \end{equation}
    and
\begin{equation}
    \label{theorem:nonresidual_network_covering_rademacher:eq_params}
    \rade{X}{\mathcal F}
        \leq
        \frac{12}{\sqrt n}
        \sqrt{
            \sum_{i=1}^L
            2 W_{i}
            \left(
            \log\left(1+\ceil{L ^2 \tilde C_{i}^2}\right) + 
            \zeta\left(\frac 3 2,1\right)^{1/3}\!\!
            \zeta\left(
                \frac 3 2,
                1+ {1}/{\ceil{L ^2 \tilde C_{i}^2}}
            \right)^{2/3}
            \right)
        }
    \enspace.
\end{equation}
\end{corollary}

For the sake of completeness, we state the generalization bounds that result from the Rademacher complexity bounds for networks with a priori constrained weights.

\vskip1ex
\begin{theorem}
    \label{theorem:generalization_feedforwad}
    For $i=1,\dots,L$ let $j=1,\dots, L_i$, $s_{ij}>0$ and $b_{ij}>0$.
    Let $\mathcal F$ be the class of residual networks of the form 
    \begin{equation}
        f = \sigma_L \circ f_L \circ \dots \circ \sigma_1 \circ f_1 \enspace,
    \end{equation}
    with $\sigma_i$ fixed $\rho_i$-Lipschitz functions satisfying $\sigma_i(0)=0$, and with $f_i$ residual blocks with identity shortcuts, i.e.,
    \begin{equation}
        f_i: \Id  + (\sigma_{iL_i} \circ f_{iL_i} \circ \dots \circ \sigma_{i1} \circ f_{i1}) \enspace,
    \end{equation}
    where $\sigma_{ij}$ are fixed $\rho_{ij}$-Lipschitz functions with $\sigma_{ij}(0)=0$ and $f_{ij}$ are convolutional or fully-connected layers whose weight tensors $K_{ij}$, resp. weight matrices $A_{ij}$, satisfy the distance constraints 
    $$\norm{K_{ij} - K_{ij}^{(0)}}_{2,1}\le b_{ij}\quad\text{and}\quad \norm{A_{ij} - A_{ij}^{(0)}}_{2,1} \le b_{ij}\enspace,$$
    with respect to reference weights $K_{ij}^{(0)}$, resp. $A_{ij}^{(0)}$, and the Lipschitz constraints $\Lip(f_{ij}) \le s_{ij}$.
    
    Let $W_{ij}$ denote the number of parameters of each layer and define
    \begin{align*}
        & \tilde C_{ij}(X) = 
        \tilde C_{ij}=
        \frac{4}{\gamma}
        \frac{\norm{X}}{\sqrt{n}} 
        \left(
            \prod_{\substack{l=1}}^{L}
                \Lip(\mathcal F_l) \rho_l
        \right)
        \frac
        {\prod_{\substack{k=1}}^{L_i} \rho_{ik} s_{ik}}     
        {\Lip(\mathcal F_i)} 
        \frac{b_{ij}}{s_{ij}}
        \enspace,\\
        & \bar L = \sum_{i=1}^L L_i
        \enspace,
        \qquad
        W = \max_{ij} W_{ij}
        \enspace.
    \end{align*}
    Then, for fixed margin parameter $\gamma>0$, every network $f \in \mathcal F$ satisfies
    \begin{equation}
        \mathbb P[\argmax_{i=1,\dots, k} f(x)_i \neq y]
        \le
        \hat {\mathcal R}_\gamma (f)
        + \frac{8}{n}
        +
        24 
        \sqrt{\log(2W)}
        \left(
            \sum_{i=1}^L
            \sum_{i=j}^{L_i}
            \ceil{\tilde C_{ij}^{2/3}}
        \right)^{\!\!3/2}
        \!\frac{H_{n-1}}{\sqrt{n}} + 3 \sqrt{\frac{\log(\frac 2 \delta)}{2n}}
    \end{equation}
    and
    \begin{equation}
        \begin{split}
        & \mathbb P[\argmax_{i=1,\dots, k} f(x)_i \neq y] \\
    &~~~~~\le  
        \hat{\mathcal R}_\gamma (f) \\
        & ~~~~~~~~~~+ \frac{24}{\sqrt{n}}
        \sqrt{
                2 
                \sum_{i=1}^L
                \sum_{i=j}^{L_i}
                W_{ij} 
                \left( 
                    \log \left(1 + {\ceil{\bar L^2 \tilde C_{ij}^2}}\right)
                    +
                    \zeta\left(\frac 3 2,1\right)^{1/3}
                    \zeta\left(
                        \frac 3 2,
                        1+ {1}/{\ceil{\bar L ^2 \tilde C_{ij}^2}}
                    \right)^{2/3}
                \right) 
        }\\ 
        & ~~~~~~~~~~+ 3 \sqrt{\frac{\log\left(\frac 2 \delta\right)}{2n}}
        \end{split}
    \end{equation}
    with probability of at least $1-\delta$ over an i.i.d. draw $((x_1,y_1), \dots, (x_n, y_n))$.
\end{theorem}
\begin{proof}
    Recall \cref{thm:Bartlett17a_Lemma31}, i.e.,
    \begin{equation}
        \mathbb P[\argmax_{i=1,\dots, k} f(x)_i \neq y]
        \le
        \hat {\mathcal R}_\gamma (f)
        + 2 \rade{S}{\mathcal F_\gamma} + 3 \sqrt{\frac{\log(\frac 2 \delta)}{2n}}
        \enspace,
    \end{equation}
    where 
    \begin{equation}
        \mathcal F_\gamma = \set{(x,y) \mapsto \ell_\gamma(-\mathcal M (f(x),y)) : f\in \mathcal F}
        \enspace.
    \end{equation}
    Bounding the empirical Rademacher complexity $\rade{S}{\mathcal F}$ via \cref{theorem:residual_network_covering_rademacher} proves the theorem. 
\end{proof}

\begin{remark}
    By a union bound argument over the constraint sets and the margin parameter, the generalization bound above can be transformed to a generalization bound which depends directly on the norms of the network weights and the Lipschitz constants instead of a priori constraints, see for example \cite[Lemma A.9]{Bartlett17a}.
    Furthermore, Lipschitz augmentation \cite{Wei19a} allows to replace the product of Lipschitz constants by empirical equivalents, i.e., norms of activations and norms of Jacobians.
\end{remark}

\subsection{Calculations}
\label{appendix:sec:calculations}
This section contains postponed calculations.

\begin{lemma}[used in \cref{theorem:single_layer_simple} and \cref{theorem:single_layer}]
    \label{theorem:binomial_coefficient}
    For any $n\in \mathbb N$, it holds that
    \begin{equation}
        \binom{n+k}{k} 
        \le
        \min\left((k+1)^n,(n+1)^k\right)
    \end{equation} 
\end{lemma}
\begin{proof}
    To prove the first inequality, note that
    \begin{equation*}
        \binom{n+k}{k} 
        =
        \frac{(n+k)!}{k!\,n!}
        =
        \frac{(k+1) \cdots (n+k)}{1 \cdots n}
        = 
        \prod_{j=1}^{n} \frac{k+j}{j}
        \le 
        \prod_{j=n}^{k} (k+1)
        =
        (k+1)^n
        \enspace.
    \end{equation*}
    Similarly,
    \begin{equation*}
        \binom{n+k}{k} 
        =
        \frac{(n+k)!}{k!\,n!}
        =
        \frac{(n+1) \cdots (n+k)}{1 \cdots k}
        = 
        \prod_{i=1}^{k} \frac{n+i}{i}
        \le 
        \prod_{i=1}^{k} (n+1)
        =
        (n+1)^k
        \enspace.
    \end{equation*}    
\end{proof}

\begin{lemma}[used in \cref{theorem:residual_network_covering_rademacher}]
    \label{theorem:integral_params}
    For any $\alpha>0$, it holds that
    \begin{equation}
    \int_{0}^{1}
    \log\left(
        1 +
        \alpha
        \ceil{\frac{1}{s^2}}
    \right)
    \,
    \mathrm{d}s
    \le 
    \log(1+\alpha)
    +
    \zeta\left(\frac 3 2,1\right)^{1/3} \zeta\left(\frac 3 2,\frac{1+\alpha}{\alpha}\right)^{2/3}
    \end{equation}
\end{lemma}

\begin{proof}
    The function
    \begin{equation*}
        \mathbbm{1}_{[0,1]}(s)
        \log\left(
            1 +
            \alpha
            \ceil{\frac{1}{s^2}}
        \right)
        =
        \sum_{m=1}^{\infty}
        \mathbbm{1}_{[\frac {1}{\sqrt{m+1}}, \frac{1}{\sqrt{m}}]}(s)
        \log\left(
            1 +\alpha (m+1)
        \right)
    \end{equation*}
    is piecewise constant and so its integral is defined as
    \begin{equation*}
        \int_{0}^{1}
        \log\left(
            1 +
            \alpha
            \ceil{\frac{1}{s^2}}
        \right)
        \,
        \mathrm{d}s
        =
        \lim_{M\to \infty}
        \sum_{m=1}^{M}
        \log(1+\alpha(m+1)) 
        \left(\frac {1}{\sqrt{m}} - \frac{1}{\sqrt{m+1}}\right)
        \enspace.
    \end{equation*}
    For any $M\in \mathbb N$, the partial sums are
    \begin{align*}
        &
        \sum_{m=1}^{M}
        (\frac {1}{\sqrt{m}} - \frac{1}{\sqrt{m+1}})
        \log\left(
            1 +\alpha (m+1)
        \right)
        \\
        =&
        \sum_{m=1}^{M}
        \frac {1}{\sqrt{m}}
        \log\left(
            1 +\alpha (m+1)
        \right)
        -
        \sum_{m=1}^{M}
        \frac{1}{\sqrt{m+1}}
        \log\left(
            1 +\alpha (m+1)
        \right)
        \\
        =&
        \sum_{m=1}^{M}
        \frac {1}{\sqrt{m}}
        \log\left(
            1 +\alpha (m+1)
        \right)
        -
        \sum_{m=2}^{M+1}
        \frac{1}{\sqrt{m}}
        \log\left(
            1 +\alpha m
        \right)
        \\
        =&
        \sum_{m=1}^{M}
        \frac {1}{\sqrt{m}}
        \log\left(
            1 +\alpha (m+1)
        \right)
        -
        \sum_{m=1}^{M}
        \frac{1}{\sqrt{m}}
        \log\left(
            1 +\alpha m
        \right)
        +
        \log\left(
            1 + \alpha
        \right)
        -
        \frac{\log\left(
            1 +\alpha (M+1)
        \right)}
        {\sqrt{M+1}}
        \\
        =&
        \sum_{m=1}^{M}
        \frac {1}{\sqrt{m}}
        \log\left(
            \frac
            {1 +\alpha (m+1)}
            {1 + \alpha m}
        \right)
        +
        \log\left(
            1 + \alpha
        \right)
        -
        \frac{\log\left(
            1 +\alpha (M+1)
        \right)}
        {\sqrt{M+1}}
        \\
        =&
        \sum_{m=1}^{M}
        \frac {1}{\sqrt{m}}
        \log\left(
            1 + 
            \frac
            {1}
            {1/\alpha + m}
        \right)
        +
        \log\left(
            1 + \alpha
        \right)
        -
        \frac{\log\left(
            1 +\alpha (M+1)
        \right)}{\sqrt{M+1}}        
        \enspace.
    \end{align*}
    Since $\lim_{M\to \infty} \frac{\log(1+\alpha M)}{\sqrt{M}} = 0$ for every $\alpha>0$, we conclude
    \begin{equation}
        \int_{0}^{1}
        \log\left(
            1 +
            \alpha
            \ceil{\frac{1}{s^2}}
        \right)
        \,
        \mathrm{d}s
        =
        \sum_{m=1}^{\infty}
        \frac {1}{\sqrt{m}}
        \log\left(
            1 + 
            \frac
            {1}
            {1/\alpha + m}
        \right)
        +
        \log\left(
            1 + \alpha
        \right)
        \enspace.
    \end{equation}
    Since 
    $\frac
    {1}
    {1/\alpha + m}
    \in (0,1)
    $
    for any $\alpha>0$ and $m\in \mathbb N$, the logarithm is given by the Mercator series
    \begin{equation}
    \log\left(
        1 + 
        \frac
        {1}
        {1/\alpha + m}
    \right)
    = 
    \sum_{k=1}^\infty
    \frac{(-1)^{k+1}}{k}
    \left(
        \frac{1}{1/\alpha+m}
    \right)^k
    \enspace.
    \end{equation}
    Inserting this into the series from above and exchanging the order of summation, we get
    \begin{align*}
        \sum_{m=1}^{\infty}
        \frac {1}{\sqrt{m}}
        \log\left(
            1 + 
            \frac
            {1}
            {1/\alpha + m}
        \right)
        &=
        \sum_{k=1}^{\infty}
        \frac{(-1)^{k+1}}{k}
        \sum_{m=1}^{\infty}
        \frac {1}{\sqrt{m}}
        \frac{1}{(1/\alpha+m)^k}
        \enspace.
    \end{align*}
    This is an alternating convergent series in $k$, so its first summand
    \begin{equation}
        \begin{split}
        \sum_{m=1}^{\infty}
        \frac {1}{\sqrt{m}}
            \frac{1}{1/\alpha+m}
        & \le
        \left(
            \sum_{m=1}^{\infty}\frac{1}{m^{3/2}}
        \right)^{1/3}
        \left(
            \sum_{m=1}^{\infty}
            \frac{1}{(1/\alpha+m)^{3/2}}
        \right)^{2/3} \\
        & =
        \zeta(3/2,1)^{1/3}
        \zeta(3/2,1+1/\alpha)^{2/3}
        \end{split}
    \end{equation}
    (using H\"older inequality) already provides an upper bound, i.e.,
    \begin{equation*}
        \int_{0}^{1}
    \log\left(
        1 +
        \alpha
        \ceil{\frac{1}{s^2}}
    \right)
    \,
    \mathrm{d}s
    \le
    \log(1+\alpha) + 
    \zeta(3/2,1)^{1/3}
    \zeta(3/2,1+1/\alpha)^{2/3}
    \enspace.
    \end{equation*}
\end{proof}

\end{appendices}

\end{document}